\documentclass[twoside]{article}

\usepackage[accepted]{aistats2020}

\usepackage[round]{natbib}

\usepackage[utf8]{inputenc} 
\usepackage[T1]{fontenc}    
\usepackage[colorlinks=true,linkcolor=red,citecolor=blue]{hyperref}       
\usepackage{url}            

\usepackage{amsfonts}       
\usepackage{nicefrac}       
\usepackage{microtype}      

\usepackage{amsmath}
\usepackage{amsthm}
\usepackage{amssymb}
\usepackage{bm}
\usepackage{graphicx}
\usepackage{enumitem}

\usepackage{makecell}
\usepackage[figuresright]{rotating}

\usepackage{booktabs}
\usepackage{multirow}

\usepackage{graphicx}
\usepackage{xcolor}
\usepackage{pdflscape}

\usepackage{dsfont}

\makeatletter
\newcommand\opteq[1]{\mathrel{\mathpalette\opt@eq{#1}}}
\newcommand{\opt@eq}[2]{%
  \begingroup
  \sbox\z@{$#1#2$}%
  \sbox\tw@{\resizebox{!}{.5\ht\z@}{$\m@th#1($}}%
  \nonscript\hskip-\wd\tw@
  \mkern1mu
  \raisebox{-.35\ht\z@}[0pt][0pt]{\resizebox{!}{.5\ht\z@}{$\m@th#1($}}%
  \mkern-1mu
  {#2}%
  \mkern-1mu
  \raisebox{-.35\ht\z@}[0pt][0pt]{\resizebox{!}{.5\ht\z@}{$\m@th#1)$}}%
  \mkern1mu
  \nonscript\hskip-\wd\tw@
  \endgroup
}
\makeatother
\newcommand{\leoq}{\opteq{\leq}}

\newcounter{assumptions}
\renewcommand\theassumptions{\Roman{assumptions}}

\newtheorem{theorem}{Theorem}
\newtheorem{lemma}{Lemma}

\newcommand{\R}{\mathbb{R}}

\newcommand{\As}{A_{\text{c}}}
\newcommand{\AsTheorem}{A_{\text{\emph{c}}}}

\DeclareMathOperator{\Pro}{Pr}
\DeclareMathOperator{\Bias}{Bias}
\DeclareMathOperator{\Error}{Error}

\newcommand{\Psymb}{\Pro}

\newcommand{\Ycorr}{\widehat{Y}_{\text{corr}}}
\newcommand{\Ytrue}{\widehat{Y}_{\text{true}}}

\newcommand{\YcorrTheorem}{\widehat{Y}_{\text{\emph{corr}}}}
\newcommand{\YtrueTheorem}{\widehat{Y}_{\text{\emph{true}}}}

\setlist[itemize]{leftmargin=*}
\setlist[enumerate]{leftmargin=*}

\begin{document}

\twocolumn[

\aistatstitle{Equalized odds postprocessing under imperfect group information}

\aistatsauthor{ Pranjal Awasthi \And Matth\"{a}us Kleindessner \And  Jamie Morgenstern }

\aistatsaddress{ Rutgers University \& Google \And University of Washington \And University of Washington \& Google } ]

\begin{abstract}
  Most  approaches 
  aiming to 
  ensure 
  a model's fairness
  with respect to a protected attribute (such as gender or race) assume 
  to know 
  the true value of
  the 
  attribute  for every
  data point. 
  In this paper, we ask to what
  extent fairness interventions can be effective even 
  when only imperfect information about the protected attribute is available. 
  In particular, we study 
  the
  prominent equalized odds 
postprocessing 
  method of \citet{hardt2016equality}
  under a perturbation of the 
 attribute.
We identify conditions on the perturbation 
that guarantee that the bias of a classifier is reduced 
even 
by running equalized odds 
with 
the perturbed 
attribute. We also study the error of the resulting classifier. 
We empirically observe that under 
our identified 
conditions most often 
the error does not suffer from a perturbation of the protected attribute.
For a special case, we formally prove this observation to be true.
\end{abstract}

\section{INTRODUCTION}\label{section_intro}

As 
machine learning (ML)
algorithms become more 
and more 
embedded into our society,
evidence has surfaced questioning whether they produce equally high-quality
predictions for most members of diverse populations. The work on
fairness in 
ML 
aims to understand the extent to
which existing ML methods produce 
\emph{fair}  
predictions for
different individuals, and what new methods can 
remove the
discrepancies therein~\citep{barocas-hardt-narayanan}. The appropriate formalization of 
``fair'' 
necessarily varies based upon the domain, leading
to a variety of definitions, largely falling into either the category
of \emph{individual fairness}~\citep[e.g., ][]{fta2012, dwork2018} 
or \emph{group fairness} \citep[e.g., ][]{Kamishima2012, hardt2016equality, kleinberg2017, pleiss2017,zafar2017,zafar2017www}. 
The former 
tries to ensure
some property for 
every 
individual (and usually is agnostic to any group
membership), while the latter asks 
some statistic (e.g.,
accuracy or false positive rate) 
to 
be 
similar for different groups.  One
key drawback of individual fairness is the need for 
a task-specific 
similarity metric over the space of individuals. Group fairness, 
on the other hand, 
usually requires knowledge of group
membership (such as gender or race), 
encoded by a \emph{protected attribute}. 
While arguably a more reasonable requirement than asking for a task-specific similarity metric, in many practical applications perfect knowledge of the protected attribute is still an invalid assumption.  
In this work, we ask to what extent one can guarantee group fairness criteria with only
imperfect information about the protected attribute, hence generalizing the
applicability~of~such~methods.

More specifically, 
we explore 
the question of when perturbed protected attribute information can be
substituted for the true 
attribute in the training phase of an existing algorithmic
framework for fair classification with 
limited 
harm to the resulting model's
fairness 
and 
accuracy.  In particular, one would never want to end up
in a situation where the ``fair'' classifier obtained from perturbed protected attribute information has worse fairness guarantees than a classifier that ignores
fairness altogether, when tested on the true data distribution. In
this work,
we study 
this question 
in the context of 
the prominent postprocessing method of \citet{hardt2016equality} for ensuring equalized odds (EO).

Our main 
contribution is to identify (fairly natural) conditions on the perturbation of the protected attribute in the training data 
for the 
EO method 
that guarantee that the resulting classifier~$\widehat{Y}$ is still more fair than the original classifier $\widetilde{Y}$ that it is based on. 
To illustrate the application of our general result, consider 
a balanced case, where the probability of a data point having 
label $y\in\{-1,+1\}$ and protected attribute $a\in\{0,1\}$ equals~$1/4$ 
independent of the values of 
$y$ and $a$,
and assume that in the training phase every 
attribute is independently flipped to its complementary value with probability~$\gamma$. 
Our result implies that 
for 
$\gamma<0.5$, the bias  (as defined in Section~\ref{subsection_bias}) of $\widehat{Y}$ will be strictly smaller than the bias of $\widetilde{Y}$. 
While a similar phenomenon was empirically observed in the
recent work of \citet{gupta2018} (see Section~\ref{section_related_work} for related work), our work is 
among 
the first to
provide 
a formal guarantee on the effectiveness of a 
prominent 
method for fairness in ML under a perturbation
of~the~attribute. 
To complement our result, we show that our identified conditions are necessary for providing such a guarantee.

We also study the error of the classifier~$\widehat{Y}$. 
We observe 
that under our identified conditions, 
most often 
the error of $\widehat{Y}$ is not larger than the error of 
the classifier that we would obtain 
from running 
the EO 
method  
with the true protected attribute (if it is larger, the difference tends to be negligible---as long as the perturbation is 
moderate). 
In the balanced case outlined above, we
 formally 
 prove this observation to be true.

\section{EQUALIZED ODDS}\label{section_equalized_odds_MAIN}    

We begin by reviewing 
the equalized odds (EO) postprocessing method of \citet{hardt2016equality},
assuming the true protected attribute for every data point is
known. Like \citeauthor{hardt2016equality} and as is common in the
literature on fair machine learning
\citep[e.g.,][]{pleiss2017,Hashimoto2018}, 
we deal with the distributional setting 
and ignore the effect of estimating probabilities from finite 
training 
samples. 

Let  $X\in\mathcal{X}$, $Y\in\{-1,+1\}$ and $A\in\{0,1\}$  be random variables with 
some joint probability distribution. 
The variable $X$ represents a data point ($\mathcal{X}$ is some
suitable set), $Y$ is the data point's ground-truth label and $A$ its
protected attribute. Like \citeauthor{hardt2016equality}, we only consider the case of binary classification and a binary protected attribute. 
The goal in fair classification is to predict 
$Y$ from $X$, or from $(X,A)$, such that the prediction is ``fair'' with respect to 
the two groups 
defined by 
$A=0$ and $A=1$. 
Think of the standard example of hiring: 
in this case, $X$ 
would 
be a collection of features describing an applicant such as 
her 
GPA, 
$Y$ 
would 
encode whether 
the applicant 
is a good fit for the job, 
and $A$ could encode  
the applicant's gender. 
There are numerous formulations of what it means for a prediction to be fair in such an example (some of them contradicting each other; 
see Section~\ref{section_related_work}),  
of which the notion of equalized odds 
as introduced by \citeauthor{hardt2016equality}  
is one of the most prominent ones. 
Denoting the (possibly randomized) prediction by $\widehat{Y}\in\{-1,+1\}$, 
the prediction satisfies the 
EO 
criterion if, for $y\in\{-1,+1\}$, 
\begin{align}\label{eq_eq_odds_criterion}
\Pro\left[\widehat{Y}=1 \big| Y=y,A=0\right]&=\Pro\left[\widehat{Y}=1 \big| Y=y,A=1\right].
\end{align}
Throughout the paper we assume 
$\Pro[Y=y,A=a]>0$ for $y\in\{-1,+1\}$ and $a\in\{0,1\}$.
For $y=+1$, 
Equation~\eqref{eq_eq_odds_criterion} 
requires that $\widehat{Y}$ has equal true positive rates for the two groups  $A=0$ and $A=1$, and 
for $y=-1$ it 
requires $\widehat{Y}$ to have equal false positive rates. 
In their paper, \citeauthor{hardt2016equality} propose a simple postprocessing method to derive a 
predictor $\widehat{Y}$ that satisfies the 
EO 
criterion from a predictor $\widetilde{Y}$  that does not, which works as follows: given a data point 
with $\widetilde{Y}=y$ and $A=a$, the predictor~$\widehat{Y}$ predicts $+1$ with probability~$p_{y,a}$ 
(hence, 
$\widehat{Y}$ depends on $X$ and $Y$ only via $\widetilde{Y}$ and $A$). 
The four probabilities 
$p_{-1,0},p_{-1,1},p_{1,0},p_{1,1}$ are 
computed in such a 
way that (i) $\widehat{Y}$ satisfies the 
 EO 
 criterion, and (ii) the error of $\widehat{Y}$, that is the probability of $\widehat{Y}$ not equaling $Y$, is 
minimized.  The former requirement and the latter objective naturally give rise to the
following linear program:
\begin{align}\label{eq_odds_linear_program}
\begin{split}
&\min_{\substack{p_{-1,0},~p_{-1,1},\\p_{1,0},~p_{1,1}\in[0,1]}}~ \sum_{\substack{\tilde{y}\in\{-1,+1\}\\a\in\{0,1\}}}
\hspace{-1mm}\left\{G(-1,a,\tilde{y})-G(1,a,\tilde{y})\right\}\cdot p_{\tilde{y},a}\\
&~\text{s.t.}~~H(y,0)\cdot p_{1,0} + \left\{1-H(y,0)\right\}\cdot p_{-1,0} =\\
&~~~ H(y,1)\cdot  p_{1,1}+ \left\{1-H(y,1)\right\}\cdot p_{-1,1}, \quad y\in\{-1,1\},
\end{split}
\end{align}
where 
$G(y,a,\tilde{y})=\Pro\left[Y=y, A=a, \widetilde{Y}=\tilde{y}\right]$ and $H(y,a)=\Pro\left[\widetilde{Y}=1 \big| Y=y,A=a\right]$.
Note that 
this linear program 
is not guaranteed to have a unique solution: 
for example, in case of $\Pro[Y=1]=\Pro[Y=-1]=1/2$, it is 
not hard 
to see that if 
$p_{-1,0}^*,p_{-1,1}^*,p_{1,0}^*,p_{1,1}^*$ is an optimal solution, then $p_{-1,0}^*+c,p_{-1,1}^*+c,p_{1,0}^*+c,p_{1,1}^*+c$, for any $c$ such that 
$p_{-1,0}^*+c,p_{-1,1}^*+c,p_{1,0}^*+c,p_{1,1}^*+c\in[0,1]$, is an optimal solution too. 
Hence, the derived predictor $\widehat{Y}$ might not be uniquely defined. 
All our results apply to \emph{any} derived 
EO 
predictor 
(derived via an 
arbitrary 
optimal solution to 
\eqref{eq_odds_linear_program}), with one limitation: whenever the constant classifier $\widehat{Y}=+1$ or $\widehat{Y}=-1$ is an optimal 
EO 
predictor (corresponding to optimal probabilities 
$p_{-1,0}=p_{-1,1}=p_{1,0}=p_{1,1}=1$ or $p_{-1,0}=p_{-1,1}=p_{1,0}=p_{1,1}=0$), 
we assume 
the derived 
EO 
predictor to be this constant classifier. 
 Throughout the paper, we 
use the terms predictor and classifier interchangeably.

\section{ANALYSIS UNDER PERTUR- BATION OF THE 
ATTRIBUTE}\label{section_analysis}

 We first describe our noise model for perturbing the 
 protected
 attribute. We then study the bias and the error of the derived  
 EO 
 predictor under this noise~model. 

\subsection{Noise Model}\label{subsection_noise_model}

When deriving the 
equalized odds 
predictor $\widehat{Y}$ from a given
classifier 
$\widetilde{Y}$, one needs to estimate the probabilities
$\Pro\left[Y=y, A=a, \widetilde{Y}=\tilde{y}\right]$ and
$\Pro\left[\widetilde{Y}=1 \big| Y=y,A=a\right]$ that appear in 
\eqref{eq_odds_linear_program}  
from training data
and then solve the resulting linear
program~\eqref{eq_odds_linear_program} 
for some optimal probabilities  $p_{-1,0},p_{-1,1},p_{1,0},p_{1,1}$. 
We refer to this as 
the \emph{training
  phase} in the 
  EO 
  procedure. 
When applying the derived classifier~$\widehat{Y}$ in order to predict the 
label
of a test point, which we call the \emph{test phase} 
of 
the  
EO 
procedure,
one 
tosses 
a 
biased 
coin and outputs a label
estimate of $+1$ with probability~$p_{y,a}$, or $-1$ with
probability~$1-p_{y,a}$, if $\widetilde{Y}=y$ and $A=a$ for~the~test~point.

Our noise model 
 captures the scenario that the 
 protected
attribute in the training data has been corrupted. 
Concretely, 
we assume that in the training phase the 
two 
probabilities mentioned above 
are replaced by
$\Pro\left[Y=y, \As=a, \widetilde{Y}=\tilde{y}\right]$ and
$\Pro\left[\widetilde{Y}=1 \big| Y=y,\As=a\right]$,
respectively. The random variable~$\As$ denotes the 
perturbed, 
or corrupted, 
attribute. 
In the test phase we assume that we have access to
the true 
attribute~$A$ without any corruption. 
Hence,  
the probabilities $p_{y,a}$ of the derived 
EO 
predictor for predicting $+1$
depend upon the perturbed 
attribute, but the predictions
 themselves depend on the true 
 attribute. 
Our noise model 
applies to scenarios in which a classifier is trained on unreliable
data (e.g., crowdsourced data, data obtained from a third party, or
when a classifier predicts the unavailable 
attribute) and
then applied to test data for which the 
attribute can 
be 
accessed directly or
easily verified (as it usually is the case in hiring, 
for example). 
We discuss alternative settings and 
directions for future work~in~Section~\ref{section_discussion}.

\subsection{Bias of the Derived Equalized Odds Predictor under Perturbation
}\label{subsection_bias}
 
 We define the bias for the class $Y=y$ (with $y\in\{-1,+1\}$) of the 
 predictor $\widehat{Y}$ as the  
 absolute error 
 in the equalized odds condition~\eqref{eq_eq_odds_criterion} for this class, that is 
 \begin{align*}
&\Bias_{Y=y}(\widehat{Y})=\left|\Pro\left[\widehat{Y}=1 \big| Y=y,A=0\right]-\right.\\
&~~~~~~~~~~~~~~~~~~~~~~~~~~~~\left.\Pro\left[\widehat{Y}=1 \big| Y=y,A=1\right]\right|.
\end{align*}  
Similarly, we define 
$\Bias_{Y=y}(\widetilde{Y})$. 
Note that $\Bias_{Y=y}(\widehat{Y})$ refers to the bias of $\widehat{Y}$ in the test phase,  
and recall from Section~\ref{subsection_noise_model} that in the test phase, according to our noise model, the derived 
EO 
predictor~$\widehat{Y}$ 
always 
makes its prediction based on 
$\widetilde{Y}$ and the true 
protected 
attribute~$A$, 
regardless of whether the
attribute has been corrupted in the training phase.  

Let now $\Ycorr$ be the derived 
EO 
predictor 
(derived from~$\widetilde{Y}$)  
when 
the protected attribute in the 
EO 
training phase 
has been corrupted, that is 
$\Ycorr$
is based on the linear program~\eqref{eq_odds_linear_program} with $A$ replaced by $\As$ (the need for notation $\Ycorr$ instead of $\widehat{Y}$ 
 comes from Section~\ref{subsection_error}, 
where we compare $\Ycorr$ to the derived 
EO 
predictor~$\Ytrue$ 
that is based on the true attribute; we provide a table collecting all random variables used in the paper in Table~\ref{table_of_notation} 
in Appendix~\ref{appendix_table_of_notation}). 
The main contribution of our paper is to establish that
the following assumptions

\vspace{1mm}
\refstepcounter{assumptions}
\textbf{Assumptions \theassumptions}\label{assu_bias}
\vspace{-2.5mm}
\begin{enumerate}[label=(\alph*)]
 \item\label{assu_bias_a}
 given  the ground-truth label~$Y$ and the true 
attribute~$A$, the prediction
$\widetilde{Y}$ and the corrupted attribute $\As$ are conditionally 
independent
 \item\label{assu_bias_b}
$\sum_{a\in\{0,1\}} \Pro\left[\As\neq A \big| Y=y,A=a\right]\leq 1$ and both summands 
strictly smaller than $1$, 
$y\in\{-1,+1\}$
 \end{enumerate}

guarantee that 
\begin{align}\label{claim_bias}
 \Bias_{Y=y}(\Ycorr)\leq \Bias_{Y=y}(\widetilde{Y}),\quad y\in\{-1,+1\},
\end{align} 
and that Assumptions~\ref{assu_bias} are necessary for guaranteeing~\eqref{claim_bias}. 
Furthermore, 
under Assumptions~\ref{assu_bias}\,\ref{assu_bias_a}, 
a strict inequality holds in \eqref{claim_bias} whenever
$\Bias_{Y=y}(\widetilde{Y})>0$ and a strict inequality holds in Assumptions~\ref{assu_bias}\,\ref{assu_bias_b}.  
Less surprising, 
$\Bias_{Y=y}(\Ycorr)$ tends to 
zero 
as, for $a\in\{0,1\}$, 
$\Pro\left[\As\neq A \big| Y=y,A=a\right]$ 
tends to
zero.   
These claims follow from Theorem~\ref{theorem_bias} and Lemma~\ref{lemma_bias_assumption}~below. 
Note that the goal of our paper is to \emph{analyze} the equalized odds method as it is
and we do not try to \emph{modify} 
the~method.

Before stating Theorem~\ref{theorem_bias} and Lemma~\ref{lemma_bias_assumption}, let us discuss their implications. 
To the practitioner 
who wants to run 
the EO method, 
but cannot rule out that the protected attribute might have been corrupted, it is higly relevant to know whether she 
can still expect to
benefit from running 
EO
or whether there is actually a risk of doing harm. 
According to 
Theorem~\ref{theorem_bias}, 
if she 
believes 
Assumptions~\ref{assu_bias} to be true, then she is guaranteed that \eqref{claim_bias} holds 
and that by running 
EO,
at the very least, she does not increase the unfairness of the given classifier~$\widetilde{Y}$. 
On the other hand, according to 
Lemma~\ref{lemma_bias_assumption}, if the practitioner expects Assumptions~\ref{assu_bias} to be 
violated, she 
should refrain from running 
EO 
as this might yield a predictor with higher bias than~the~given~classifier~$\widetilde{Y}$.

Assumptions~\ref{assu_bias} are fairly natural and might be satisfied in 
several practical 
situations. 
Assumptions~\ref{assu_bias}\,\ref{assu_bias_a} asks for conditional independence (given $Y$ and $A$) of the given 
classifier~$\widetilde{Y}$ and the corrupted attribute~$\As$. For example, this is the case if $\As$ is the output of a classifier that only uses features that are 
conditionally independent of the features used by $\widetilde{Y}$ (e.g., 
 body height is used 
for predicting 
gender and
$\widetilde{Y}$ uses 
GPA for predicting aptitude for a job\footnote{Another example, involving race as 
attribute, might be the following: 
if one uses a person's surname to predict her race and her income to predict her creditworthiness, 
then it seems 
very 
plausible that conditional independence holds.}).  
As another example, 
Assumptions~\ref{assu_bias}\,\ref{assu_bias_a} 
is also true if $\As$ is a crowdsourced estimate of $A$ and one 
assumes 
that a crowdworker's probability of providing an incorrect estimate depends on 
the true label of the task (in our case $A$), but not on the task~$X$ itself, which is the standard assumption in 
most 
of the 
ML 
literature on crowdsourcing \citep[cf.][]{kleindessner_crowdsourcing}.
 Assumptions~\ref{assu_bias}\,\ref{assu_bias_b} limits the level of perturbation of the 
 protected 
 attribute, but in a rather moderate way. 
For example, 
in case of 
$\Pro[A=a | Y=y]=1/2$, 
$y\in\{-1,+1\}$, $a\in\{0,1\}$,  
if 
$\Pro[\As\neq A| Y=y]< 1/2$, $y\in\{-1,+1\}$,  
then Assumptions~\ref{assu_bias}\,\ref{assu_bias_b}
is satisfied.

\vspace{1mm}
\begin{theorem}[Bias of $\Ycorr$ vs. bias of $\widetilde{Y}$]\label{theorem_bias}
Assume that Assumptions~\ref{assu_bias}\,\ref{assu_bias_a} holds and that 
$\Pro\left[\AsTheorem\neq A | Y=y,A=a\right]<1$ for $y\in\{-1,+1\}$ and $a\in\{0,1\}$. Then, for $y\in\{-1,+1\}$, 
the derived equalized odds predictor~$\YcorrTheorem$ satisfies 
\begin{align}\label{theorem_bias_formula}
\begin{split}
&\Bias_{Y=y}(\YcorrTheorem)\leq \Bias_{Y=y}(\widetilde{Y})\cdot \\
&~~~~~~~~~~~~F\left(L(y,0),L(y,1),\Pro\left[A=1 |Y=y\right]\right),
\end{split}
\end{align}
where $L(y,a)=\Pro\left[\AsTheorem\neq A | Y=y, A=a \right]$ and $F=F(\gamma_1,\gamma_2,p)$ is some 
differentiable 
function (explicitly stated in \eqref{definition_function_F} in Appendix~\ref{supp_mat_proofs})
that is strictly~increasing both in $\gamma_1$ and in $\gamma_2$ 
with $F(0,0,p)=0$ 
and $F(\gamma_1,\gamma_2,p)\leoq 1$ for all $(\gamma_1,\gamma_2,p)$ with $\gamma_1+\gamma_2\leoq 1$. 
\end{theorem}

\vspace{3mm}
\begin{lemma}[Assumptions~\ref{assu_bias} are necessary for guaranteeing~\eqref{claim_bias}]\label{lemma_bias_assumption}
If 
any of 
Assumptions~\ref{assu_bias}\,\ref{assu_bias_a} or \ref{assu_bias_b} is violated, 
inequality~\eqref{claim_bias} 
might not be true.
  \end{lemma}

The proofs of Theorem~\ref{theorem_bias} and Lemma~\ref{lemma_bias_assumption} can be found in  Appendix~\ref{supp_mat_proofs}, 
and we provide some intuition behind Theorem~\ref{theorem_bias} in Section~\ref{subsection_intuition}.  
The main difficulty in proving Theorem~\ref{theorem_bias} comes from the fact that the equalized odds method uses the protected attribute in the test phase. 
This creates conditional 
dependencies that make it necessary to characterize how the optimal probabilities $p_{-1,0},p_{-1,1},p_{1,0},p_{1,1}$ 
to 
the linear program~\eqref{eq_odds_linear_program} 
change under a perturbation of the attribute. 
Note that the counterexamples that we provide for proving Lemma~\ref{lemma_bias_assumption} are not worst-case scenarios in which 
Assumptions~\ref{assu_bias}\,\ref{assu_bias_a} or \ref{assu_bias_b} would be heavily violated. Indeed, our counterexamples show that a moderate violation of 
Assumptions~\ref{assu_bias}\,\ref{assu_bias_a} or a minimal violation of Assumptions~\ref{assu_bias}\,\ref{assu_bias_b} can result in \eqref{claim_bias} not being true. 
Also note that \eqref{theorem_bias_formula} in Theorem~\ref{theorem_bias} provides a quantitative bound on the bias of $\Ycorr$ (in our experiments in Section~\ref{subsec_simulations} we will see that 
in most cases 
this
 bound is 
 quite tight). In case the practitioner can estimate the various probabilities that it involves, this bound might provide additional benefit to her, 
 but this 
 idea 
 goes beyond the scope of our paper (cf. Section~\ref{section_discussion}).

\subsection{Error of the Derived Equalized Odds Predictor under Perturbation
}\label{subsection_error}

The error of $\widehat{Y}$ 
is given by  
$\Error(\widehat{Y})=\Pro\left[\widehat{Y}\neq Y\right]$. 
Note that just as $\Bias_{Y=y}(\widehat{Y})$, $\Error(\widehat{Y})$ refers to the error of $\widehat{Y}$ in the test phase.
As in 
Section~\ref{subsection_bias}, 
let $\Ycorr$ be the derived 
EO 
predictor based on the corrupted protected attribute~$\As$, and let
 $\Ytrue$ be the 
 EO 
 predictor that is based on the true attribute~$A$. 
In our experiments in Section~\ref{section_experiments} we observe that under Assumptions~\ref{assu_bias} 
from  Section~\ref{subsection_bias}  
and the following additional assumption

\vspace{1mm}
\refstepcounter{assumptions}
\textbf{Assumption \theassumptions}\label{assu_error}
\vspace{-2.5mm}
\begin{itemize}
 \item the given predictor~$\widetilde{Y}$ is 
correlated with the ground-truth label $Y$ 
in the sense that for $a\in\{0,1\}$
\begin{align*}
&\Psymb\left[\widetilde{Y}=1 \big| Y=1, A=a\right]>\\
&~~~~~~~~~~~~~~~~\Psymb\left[\widetilde{Y}=1 \big| Y=-1, A=a\right]
\end{align*}
 \end{itemize}
 
we most often have
\begin{align}\label{inequality_error}
 \Error(\Ycorr)\leq\Error(\Ytrue). 
\end{align}
In our experiments, 
if 
inequality~\eqref{inequality_error} 
is not true, it tends to be violated only to a negligible extent. 
In fact, 
our experiments and some intuition (outlined in Section~\ref{subsection_intuition})  
initially  
misled us to 
conjecture that  
 \eqref{inequality_error}  
would \emph{always} 
be true 
under Assumptions~\ref{assu_bias} and~\ref{assu_error} 
(cf. 
the prior arXiv-version of this 
paper), but as we show in Section~\ref{subsec_simulations},  
such a 
conjecture is wrong. It only holds  in a special balanced case as stated in  
Theorem~\ref{thm:accuracy-balanced}  below. In general, it remains an open question which assumptions on $\As$, $\widetilde{Y}$, and the 
base rates~
$\Psymb[Y=y,A=a]$ 
would 
guarantee 
inequality~\eqref{inequality_error} 
to hold (cf. Section~\ref{section_discussion}).

Assumption~\ref{assu_error} is mild and 
kind of a minimal requirement 
for $\widetilde{Y}$  to be considered useful. 
If 
$\Psymb[Y=y| A=a] = 1/2$, $y\in\{-1,+1\}$, $a\in\{0,1\}$,  
it 
is equivalent to requiring 
$\widetilde{Y}$ to be a weak learner for 
both  
groups $A=a$,  
that 
is to satisfy 
$\Psymb\left[\widetilde{Y}\neq Y \big| A=a\right]<1/2$, $a\in\{0,1\}$. 
However, in a special balanced case, together with Assumptions~\ref{assu_bias}, 
Assumption~\ref{assu_error} 
is sufficient to guarantee 
that 
\eqref{inequality_error} 
holds 
as the 
following theorem states:

\begin{theorem}[Error of $\Ycorr$ vs. error of $\Ytrue$  in a special case]
\label{thm:accuracy-balanced}
Assume that Assumptions~\ref{assu_bias} and~\ref{assu_error} hold. Furthermore, assume that  
$\Psymb[Y=y, A=a] = 1/4$, 
$y\in\{-1,+1\}$, $a\in\{0,1\}$, and  
$\Pro\left[\AsTheorem\neq A |Y=y, A=a\right]\in
(0,1/2]$  
does not depend on 
$y$ and $a$. Then we have
\begin{align*}
 \Error(\YcorrTheorem)\leq\Error(\YtrueTheorem), 
\end{align*}
with equality holding  if and only if the given classifier~$\widetilde{Y}$ is unbiased, 
that is 
$\Bias_{Y=+1}(\widetilde{Y})=\Bias_{Y=-1}(\widetilde{Y})=0$.
\end{theorem}

 The proof of Theorem~\ref{thm:accuracy-balanced} can be found  in 
 Appendix~\ref{supp_mat_proofs}.
Although 
several 
expressions in the analysis of the linear program~\eqref{eq_odds_linear_program} 
simplify in the special case considered in 
Theorem~\ref{thm:accuracy-balanced}, the proof of Theorem~\ref{thm:accuracy-balanced} 
is still 
involved and requires a case analysis 
that distinguishes 
which of 
the  probabilities in an optimal solution to~\eqref{eq_odds_linear_program} equal 1.

\subsection{Intuition behind 
\eqref{claim_bias} and \eqref{inequality_error}
}\label{subsection_intuition}

To provide some intuition behind our results 
and observations, 
consider the simple case of independently flipping each data point’s protected attribute to its complementary value with probability~$\gamma$.  
If $\gamma=0$, the 
EO 
method gets to see the true attribute and we end up with the classifier~$\Ytrue$, which has zero bias, 
but usually quite a larger error than the given classifier~$\widetilde{Y}$.   
If $\gamma=0.5$, the EO method gets to see random noise as the attribute and the given classifier~$\widetilde{Y}$ appears to be totally fair. 
In this case, unless 
$\widetilde{Y}$ is rather  bad and its accuracy 
can be improved 
simply by flipping its predictions from $+1$ to $-1$, or the other way round, for a group $A=a$ 
(in the balanced case studied in Theorem~\ref{thm:accuracy-balanced},  Assumption~\ref{assu_error} rules out such a situation),
the EO method 
returns 
the given $\widetilde{Y}$, 
which has smaller error than $\Ytrue$, but higher test-phase bias. 
For $0<\gamma<0.5$, the EO method yields a 
classifier~$\Ycorr$ that interpolates between these two extremes: some amount of random noise in the attribute makes the given classifier~$\widetilde{Y}$
appear more fair than it actually is and the EO method changes $\widetilde{Y}$ 
(i.e., 
decreases its bias~/~increases its error) in a less severe way than when it gets 
to see the true attribute. This interpolation behavior can be seen nicely in the simulations 
that we~provide~in~Section~\ref{subsec_simulations}.

While the conclusions 
in 
this simple case about the relationship between $\widetilde{Y}$, $\Ytrue$, and $\Ycorr$ with respect to the bias
carry over to the more general setting that we consider (we proved \eqref{claim_bias} to be always true), 
we do not know where our intuition breaks down with respect to the error in those rare situations in which \eqref{inequality_error} is not true 
(cf. Sections~\ref{subsec_simulations} and~\ref{section_discussion}). 

\section{RELATED WORK}\label{section_related_work}

This section is a short version of a corresponding long version provided in Appendix~\ref{appendix_related_work}. 

By now, there is a huge body of work on fairness in 
ML, 
mainly  in supervised learning 
\citep[e.g.,][]{feldman2015,hardt2016equality,kleinberg2017,pleiss2017,woodworth2017,
zafar2017,zafar2017www,agarwal2018,donini2018,FairGAN2018,kallus2019}, 
but 
more 
recently also in  unsupervised learning  
\citep[e.g.,][]{fair_clustering_Nips2017,samira2018,fair_k_center_2019,fair_SC_2019}. 
All of these papers assume to know the true value of the protected attribute for 
every 
data point. We will discuss some papers not making this assumption below. 
First we discuss the pieces of work
 related to the fairness notion of equalized odds, which is central to our paper and one of the most prominent fairness notions in the ML literature
 (see \citealp{Verma2018}, for a 
 summary of the various 
 notions and a citation~count).

 \vspace{1.5mm}
 \textbf{Equalized Odds~~}
 Our paper builds upon the 
 EO 
 method of \citet{hardt2016equality} as described in Section~\ref{section_equalized_odds_MAIN}. 
Concurrently with 
\citeauthor{hardt2016equality}, the 
fairness 
notion of 
EO 
has also been proposed by \citet{zafar2017www} under the name of disparate mistreatment. 
The seminal paper of 
 \citet{kleinberg2017} proves that, except for trivial cases, a classifier cannot satisfy the 
 EO 
 criterion and 
 calibration within groups 
 at the same time. 
 Subsequently,  
 \citet{pleiss2017} show how to achieve calibration within groups and  a relaxed 
 form of the 
 EO 
 constraints 
 simultaneously. 
 \citet{woodworth2017} show that 
 postprocessing a Bayes optimal 
 unfair  
 classifier 
 in order to obtain a fair classifier 
  can~be~suboptimal.

\vspace{1.5mm}
\textbf{Fairness with 
Only 
Limited Information about the 
Attribute~~} 
Only recently there have been works studying how to satisfy group fairness criteria when having only limited information about the protected attribute. 
Most important to mention are  the 
works by  \citet{gupta2018} and \citet{lamy2019}. 
\citet{gupta2018} empirically show that 
when the 
attribute is not known, improving a fairness metric  
for a proxy of the true 
attribute 
can 
improve the fairness metric for the true attribute. 
Our paper provides theoretical evidence 
for their observations.
\citet{lamy2019} study a scenario related to ours and consider training a fair classifier when the 
attribute is corrupted 
according to a mutually contaminated model \citep{scott2013}. 
In their case, 
training 
is done by means of constrained empirical risk minimization. 
Also important to mention is the paper by \citet{Hashimoto2018}, 
which
uses distributionally robust optimization in order to minimize the worst-case misclassification risk in a $\chi^2$-ball around the data generating distribution. 
In doing so, under the assumption that the resulting non-convex optimization problem was solved exactly, 
one provably 
controls the risk of each protected group without knowing which group a data point belongs to. \citeauthor{Hashimoto2018} 
show that their 
approach helps to avoid disparity amplification 
in a sequential classification setting in which a group's fraction in the data decreases as its misclassification risk increases. 
In Section~\ref{subsec_exp_repeated_loss_minimization} /  Appendix~\ref{supp_mat_rep_loss_min},   
we experimentally compare 
their approach 
to the 
EO 
method with perturbed 
attribute information in such a sequential setting.  
Further works around group fairness 
with 
limited information about the 
attribute
are the papers by
\citet{botros2018}, \citet{kilbertus2018}, \citet{chen_fairness_under_unawareness}, and 
\citet{coston2019}.

\begin{figure*}[t]
 \centering
 \includegraphics[width=7.1cm]{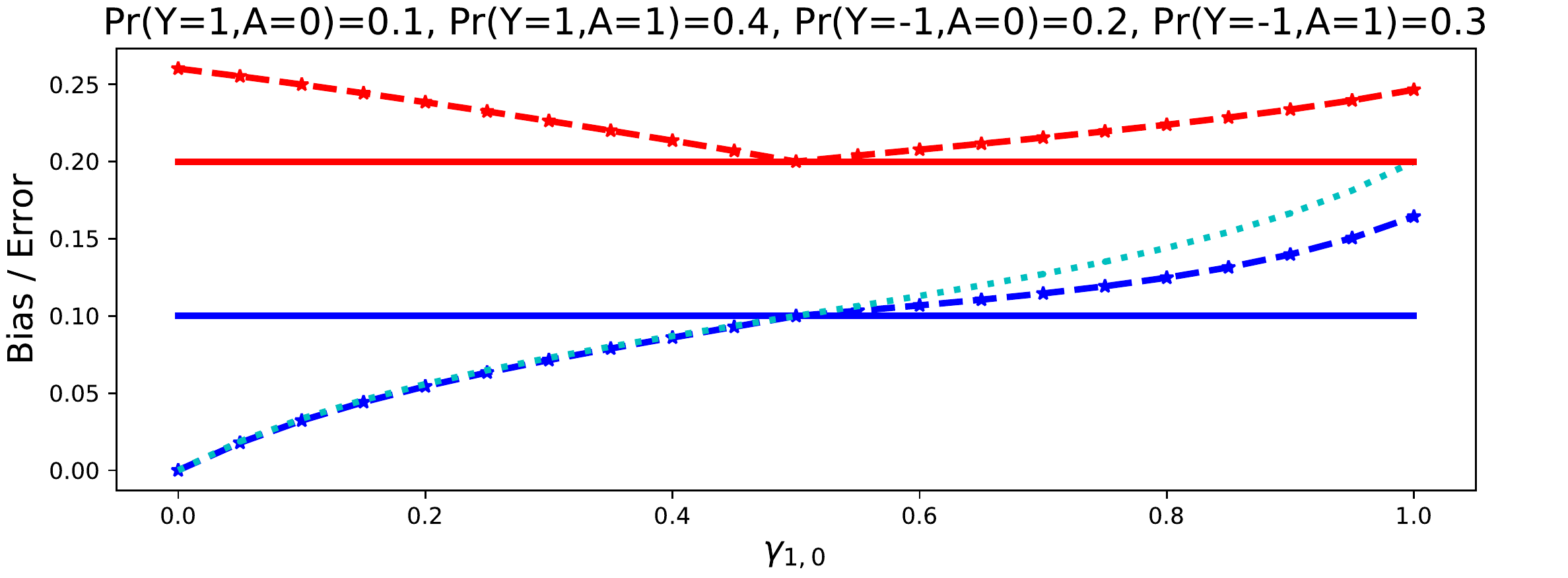}
 \hspace{8mm}
 \includegraphics[width=7.1cm]{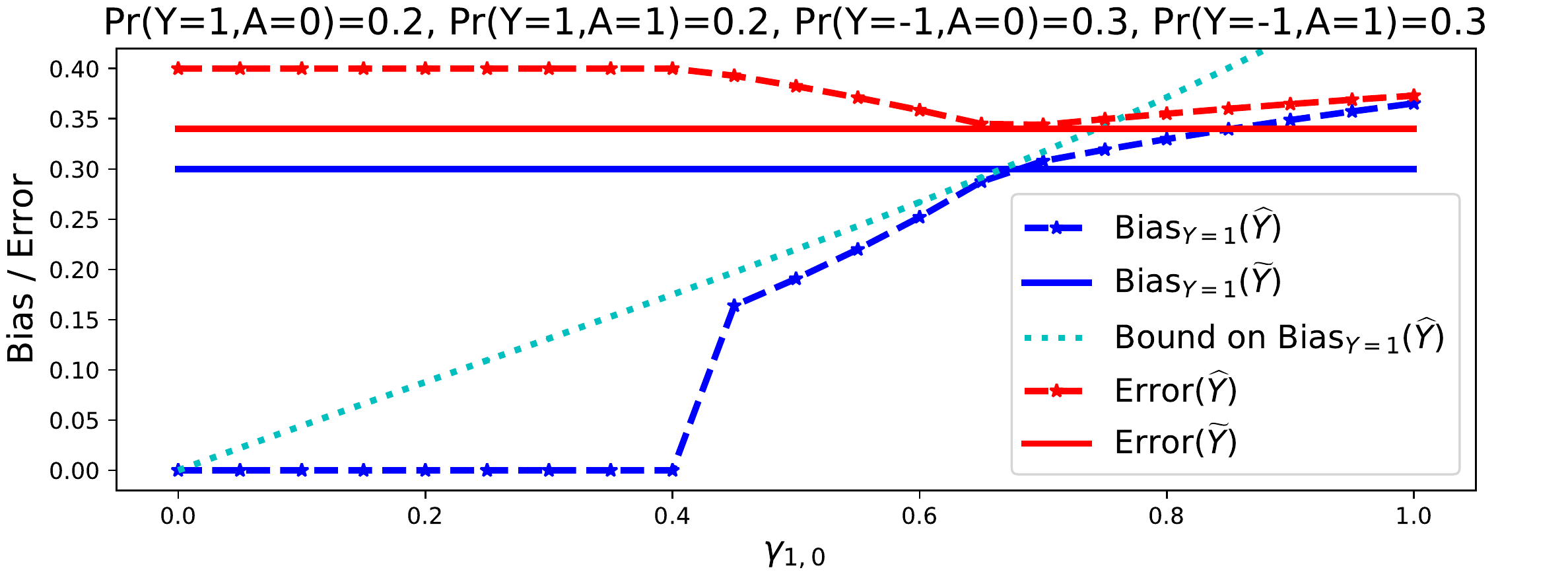}
 
 \vspace{2mm}
 \includegraphics[width=7.1cm]{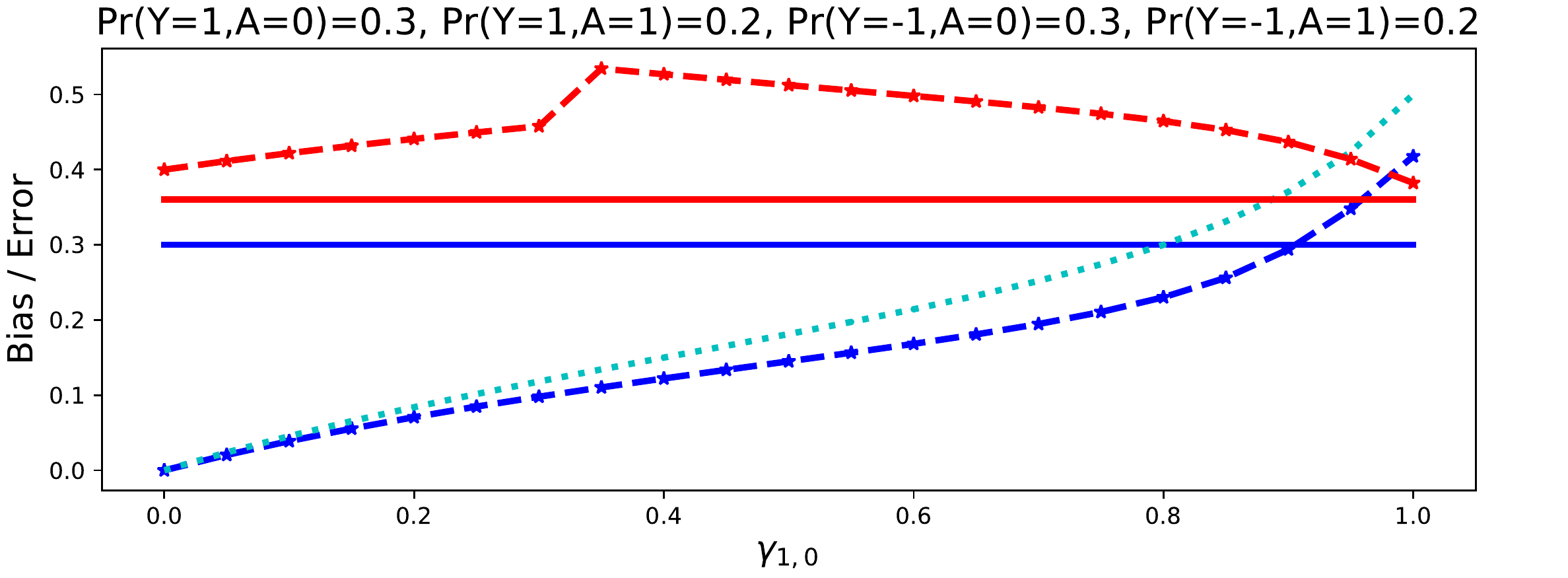}
 \hspace{8mm}
 \includegraphics[width=7.1cm]{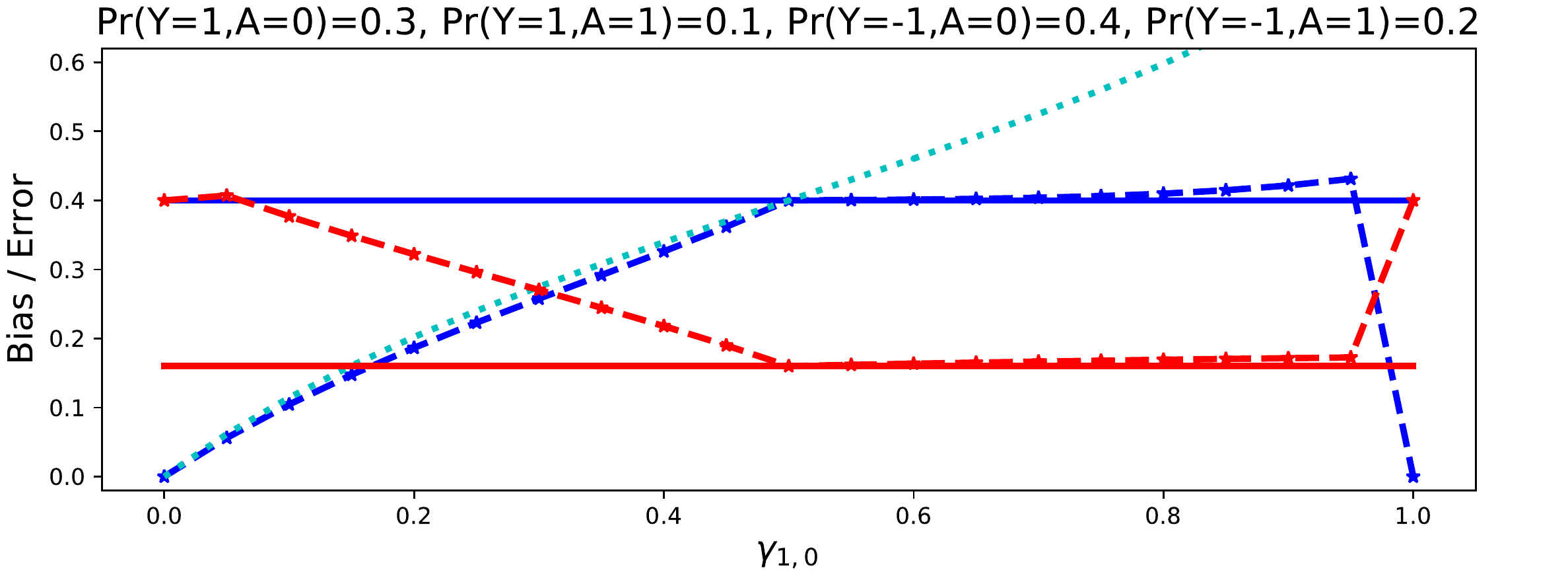}

\caption{$\Bias_{Y=1}(\widehat{Y})$ (dashed blue) and $\Error(\widehat{Y})$ (dashed red) as a function of the perturbation level for various 
problem parameters 
(see the titles of the plots and Table~\ref{table_parameters_main} in Appendix~\ref{subsec_table1}). 
For $\gamma_{1,0}=0$ it is $\widehat{Y}=\Ytrue$, and for $\gamma_{1,0}>0$ it is $\widehat{Y}=\Ycorr$.
The solid lines show the bias (blue) and the error (red) of 
$\widetilde{Y}$.
The dotted cyan 
curve shows the bound on $\Bias_{Y=1}(\widehat{Y})$ provided in \eqref{theorem_bias_formula} in Theorem~\ref{theorem_bias}. 
In the left bottom plot, 
Assumption~\ref{assu_error} is not satisfied and here the error of $\widehat{Y}$ clearly initially increases, 
that is \eqref{inequality_error} does not hold. In the right bottom plot, although Assumption~\ref{assu_error} is satisfied, the 
error of $\widehat{Y}$ also initially increases, but in this case only to a negligible~extent.}\label{figure_simulations}
\end{figure*}

\section{EXPERIMENTS}\label{section_experiments}

In this section, we present a number of experiments.\footnote{Python  code  available  on \url{https://github.com/matthklein/equalized_odds_under_perturbation}.} 
First, we study the bias and 
the error of the 
EO 
predictor~$\widehat{Y}$ 
as a function of the  level of perturbation of the 
protected 
attribute 
in extensive simulations. In doing so, we empirically validate 
Theorems~\ref{theorem_bias} and~\ref{thm:accuracy-balanced} of Section~\ref{section_analysis} and provide evidence for our claim 
of Section~\ref{subsection_error} 
that most often we observe inequality~\eqref{inequality_error} to be true. 
Next, we show some experiments on real data, 
providing 
some motivation for our paper and 
further support for 
its main results and claims. 
Finally, we consider the repeated loss minimization setting of \citet{Hashimoto2018} and demonstrate that the 
EO 
method achieves the same goal
as 
their strategy, 
even when the protected attribute is highly perturbed.

\begin{table*}[t]
  \caption{Experiment on the Drug Consumption data set.  The full table is provided in Appendix~\ref{appendix_table_drugs}.}\label{table_drugs_main}
   \centering
\renewcommand{\arraystretch}{1.5}
\begin{small}
\begin{tabular}{cccccccc}
\toprule \toprule
$Y$ & $\Psymb[Y=1]$ & $\Bias_{Y=1/-1}(\widetilde{Y})$ & $\Bias_{Y=1/-1}(\Ycorr)$ & $\Error(\widetilde{Y})$ & $\Error(\Ycorr)$ & 
$\Error(\Ytrue)$ & Co. Ind. \eqref{independence_measure}\\
\toprule
Amphet & 0.36 & 0.085 / 0.106 & 0.076 / 0.065 &  0.317 &  0.339 &  0.352 &  0.033 \\ 
\midrule
 Benzos &  0.41 & 0.074 / 0.132 & 0.064 / 0.1 &  0.351 &  0.369 &  0.39 &  0.036 \\ 
 \midrule
 Cannabis &  0.67 & 0.092 / 0.052 & 0.091 / 0.073 &  0.214 &  0.227 &  0.255 &  0.032 \\





\bottomrule \bottomrule
\end{tabular}
\end{small}
\end{table*}

\subsection{Simulations of Bias and Error}\label{subsec_simulations}

For various choices of the problem parameters $\Psymb[Y=y,A=a]$ and $\Psymb\left[\widetilde{Y}=1|Y=y,A=a\right]$, we study how the 
bias and the error of the derived 
EO 
predictor $\widehat{Y}$ change 
as the  perturbation probabilities $\Pro\left[\As\neq A | Y=y, A=a \right]$, with which the protected attribute in the 
EO 
training phase is perturbed, increase. 
For doing so, we solve the linear program~\eqref{eq_odds_linear_program} where in all probabilities 
$A$ is replaced by $\As$. We 
always assume that 
 Assumptions~\ref{assu_bias}\,\ref{assu_bias_a} is satisfied.  
The resulting linear program is provided in Appendix~\ref{subsec_supp_mat_sim_details}. 
We compare 
the bias and the error of $\widehat{Y}$ to the  bias and the error of $\widetilde{Y}$, and we also compare the bias of $\widehat{Y}$ to our theoretical bound provided
in \eqref{theorem_bias_formula} in Theorem~\ref{theorem_bias}. Let $\gamma_{y,a}:=\Pro\left[\As\neq A | Y=y, A=a\right]$, $y\in\{-1,+1\},a\in\{0,1\}$. 
Figure~\ref{figure_simulations} shows the quantities of interest as a function 
of $\gamma_{1,0}$, 
where 
$\gamma_{1,1},\gamma_{-1,0},\gamma_{-1,1}$ grow with $\gamma_{1,0}$ in a certain way, in various scenarios  
(the probabilities $\Psymb[Y=y,A=a]$ 
can be read from the titles 
of the plots, and the other parameters are provided in Table~\ref{table_parameters_main} in Appendix~\ref{subsec_table1}). 
In the notation of Section~\ref{section_analysis}, for $\gamma_{1,0}=0$ it is $\widehat{Y}=\Ytrue$  
and for $\gamma_{1,0}>0$ it is $\widehat{Y}=\Ycorr$.  
For clarity, we only show the bias for the class $Y=+1$. 
As suggested by our upper 
bound~\eqref{theorem_bias_formula},
in all four plots 
the bias of $\widehat{Y}$ is increasing as the perturbation level increases, and we can see that our upper 
bound is quite tight in most cases. For a moderate perturbation level with $\gamma_{1,0}+\gamma_{1,1}<1$, the bias of $\widehat{Y}$ is smaller than the bias of $\widetilde{Y}$ as 
claimed 
by Theorem~\ref{theorem_bias}. 
Note that all four plots show a non-balanced case, which is not captured by Theorem~\ref{thm:accuracy-balanced}.  
Still, in the two plots in the top row, 
the error of $\widehat{Y}$ 
decreases as the perturbation level increases up to the point that 
the error of $\widehat{Y}$ equals the error of~$\widetilde{Y}$, 
that is inequality~\eqref{inequality_error} is true.  
In the bottom left plot, Assumption~\ref{assu_error} is not satisfied   
and we do not expect inequality~\eqref{inequality_error} to hold here. In the bottom right plot, 
Assumption~\ref{assu_error} is  satisfied, but \eqref{inequality_error} does not hold either 
since the error of~$\widehat{Y}$ initially increases; however, the violation of \eqref{inequality_error} 
is negligible (for $\gamma_{1,0}=0$, it is $\Error(\widehat{Y})=0.4$, and for $\gamma_{1,0}=0.05$, it is $\Error(\widehat{Y})=0.407$). 
Note that we do not fully understand this behavior (cf. Sections~\ref{subsection_error}, \ref{subsection_intuition} and~\ref{section_discussion}).
We make 
similar 
observations in a 
number of further experiments of this type 
presented
in Appendix~\ref{supp_mat_section_experiments}, and our findings 
confirm 
the main claims~of~our~paper.

\subsection{Experiments on Real Data}\label{subsec_exp_real_data}

We first present an experiment in which we train a classifier to predict the protected attribute and replace the true attribute by the prediction in the 
EO 
training 
phase. Such a scenario is one of our motivations for studying the 
EO 
method under a perturbation of the protected attribute. 
We perform the experiment on the Drug Consumption data set \citep{drug_consumption_data}. It comprises 1885 records of human subjects, and
for each subject, it provides 
five demographic features (e.g., Age, Gender, or Education), seven features measuring personality traits (e.g., Nscore is a measure of neuroticism and Ascore of agreeableness), 
and 18 features 
each 
of which 
describes the 
subject's last use of a certain drug (e.g., Cannabis). 
We set the protected attribute~$A$ to be Gender, and, fixing a drug, 
we set the ground-truth label~$Y$ to indicate whether 
a subject has used the drug within the last decade ($Y=1$) or not ($Y=-1$).
Randomly splitting the data set into three batches of equal size, 
we use the first batch to train a logistic regression classifier that predicts $A$ using 
the features 
Nscore and Ascore (these two turned out to 
work best), 
and 
 a one-hidden-layer perceptron that predicts $Y$ using the demographic features except Gender 
 and the five 
 features for personality traits
 other than Nscore and Ascore.  We consider the first classifier to provide a perturbed version $\As$ of the true 
 attribute~$A$ and  
 the second classifier to be the given classifier~$\widetilde{Y}$. 
  We use the second batch to derive 
  EO 
  predictors~$\Ycorr$ and~$\Ytrue$ from $\widetilde{Y}$, where 
 $\Ycorr$ is based on $\As$ and $\Ytrue$ is based on $A$.  
 The third batch is our test batch, on which 
 we evaluate the bias and the error of $\widetilde{Y}$, $\Ycorr$, and $\Ytrue$, 
 the probability of $\As$ not equaling $A$, 
 and also 
 whether Assumptions~\ref{assu_bias} and~\ref{assu_error} are satisfied. 
We measure the extent to which Assumptions~\ref{assu_bias}\,\ref{assu_bias_a} is violated by the 
estimated 
$l_{\infty}$-distance between the conditional (given $Y$ and $A$) 
joint distribution of $\widetilde{Y}$ and $\As$ and the product of their conditional marginal distributions, that is 
\begin{align}\label{independence_measure}
\begin{split}
 &\max_{\substack{y,\tilde{y}\in\{-1,+1\}\\a,\tilde{a}\in\{0,1\}}}\left|\Pro\left[\widetilde{Y}=\tilde{y},\As=\tilde{a}\big|Y=y,A=a\right]-\right.\\
&\left.\Pro\left[\widetilde{Y}=\tilde{y}\big|Y=y,A=a\right]\hspace{-0.9pt}\cdot\hspace{-0.7pt} \Pro\left[\As=\tilde{a}\big|Y=y,A=a\right]\right|.
\end{split}
\end{align}
Note that, in the distributional setting,  Assumptions~\ref{assu_bias}\,\ref{assu_bias_a} is satisfied if and only if this quantity~is~zero.

Table~\ref{table_drugs_main} shows the results for three of the drugs, where we report average results obtained from running the experiment for 200 times. 
For the sake of readability, we do not report $\Bias_{Y=y}(\Ytrue)$ (which equals zero in the distributional setting) in 
Table~\ref{table_drugs_main}. 
A full table that  shows $\Bias_{Y=y}(\Ytrue)$ as well as the results for the other 
drugs is provided as Table~\ref{table_drugs_appendix} in Appendix~\ref{appendix_table_drugs}. 
It is 
$\Psymb[A=a]=1/2$ 
and  
$\Psymb[\As\neq A | A=a]=0.4$, $a\in\{0,1\}$. 
Almost always,
Assumptions~\ref{assu_bias}\,\ref{assu_bias_b} and Assumption~\ref{assu_error} are 
satisfied (see Table~\ref{table_assumptions_are_satisfied} in  
Appendix~\ref{appendix_table_drugs} for details). 
As we can see from the last column of Table~\ref{table_drugs_main} or Table~\ref{table_drugs_appendix}, 
for all the drugs,  the measure~\eqref{independence_measure} is rather small, indicating that also 
Assumptions~\ref{assu_bias}\,\ref{assu_bias_a} might be (almost) satisfied. 
In this light, 
the results for the bias and the error of the various classifiers are in accordance with the claims of our paper: 
we most often have  
$\Bias_{Y=y}(\Ycorr)< \Bias_{Y=y}(\widetilde{Y})$ 
and we 
always have 
$\Error(\Ycorr)\leq\Error(\Ytrue)$.
We 
consider 
finite-sample effects to be responsible for the first inequality not always being true  
since for Cannabis it even happens that $\Bias_{Y=-1}(\Ytrue)> \Bias_{Y=-1}(\widetilde{Y})$ 
(cf. Table~\ref{table_drugs_appendix}).

\newcommand{\wire}{7.3cm}
\begin{figure*}[t]
  \centering
 \includegraphics[width=\wire]{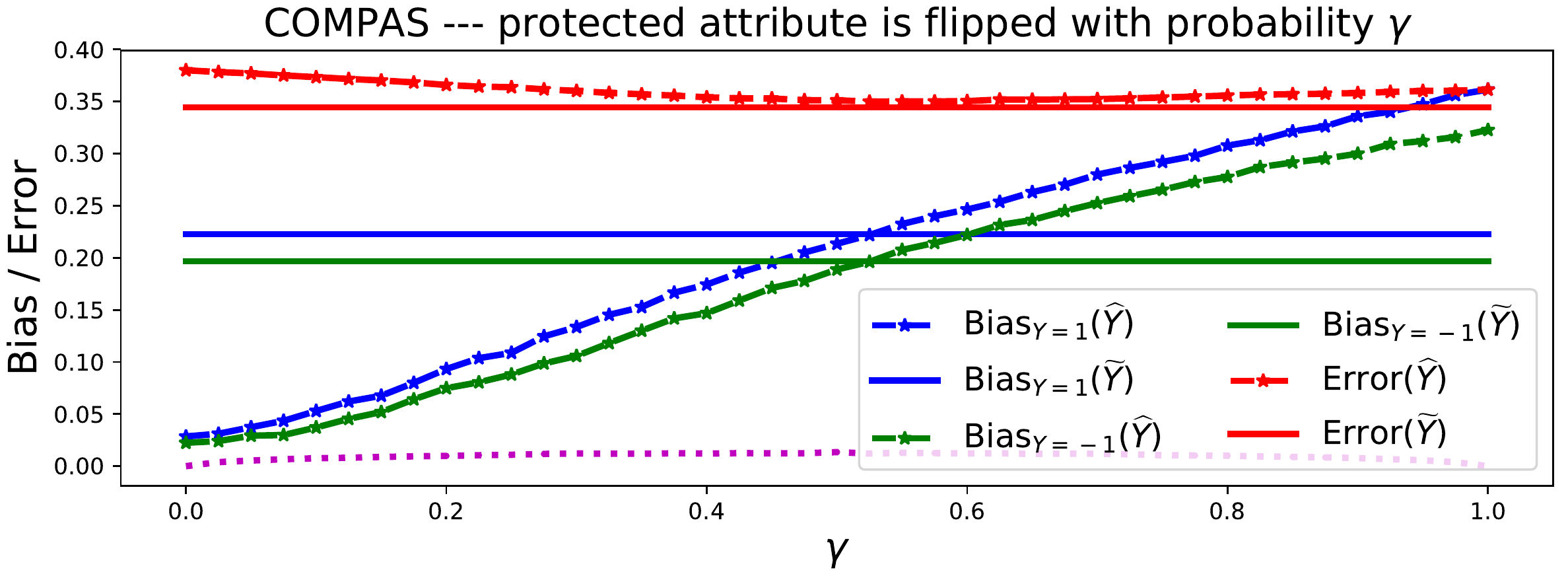}
 \hspace{8mm}
 \includegraphics[width=\wire]{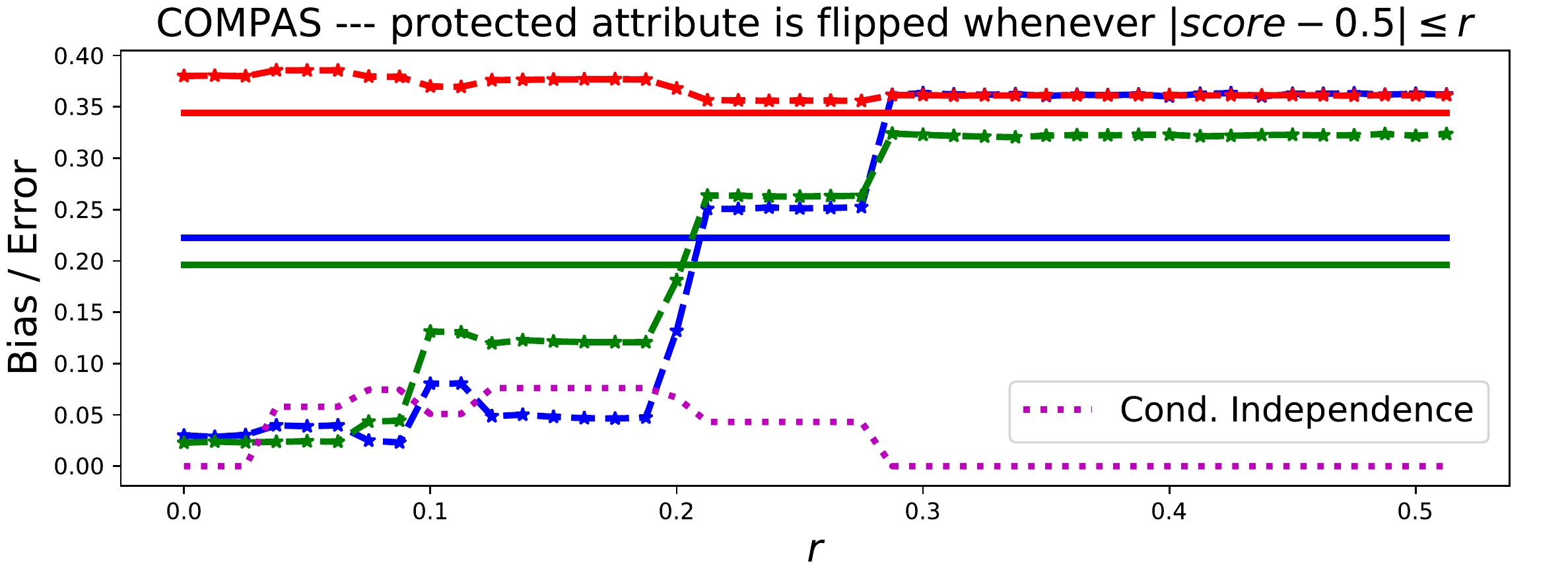}
 
 \vspace{2mm}
 \includegraphics[width=\wire]{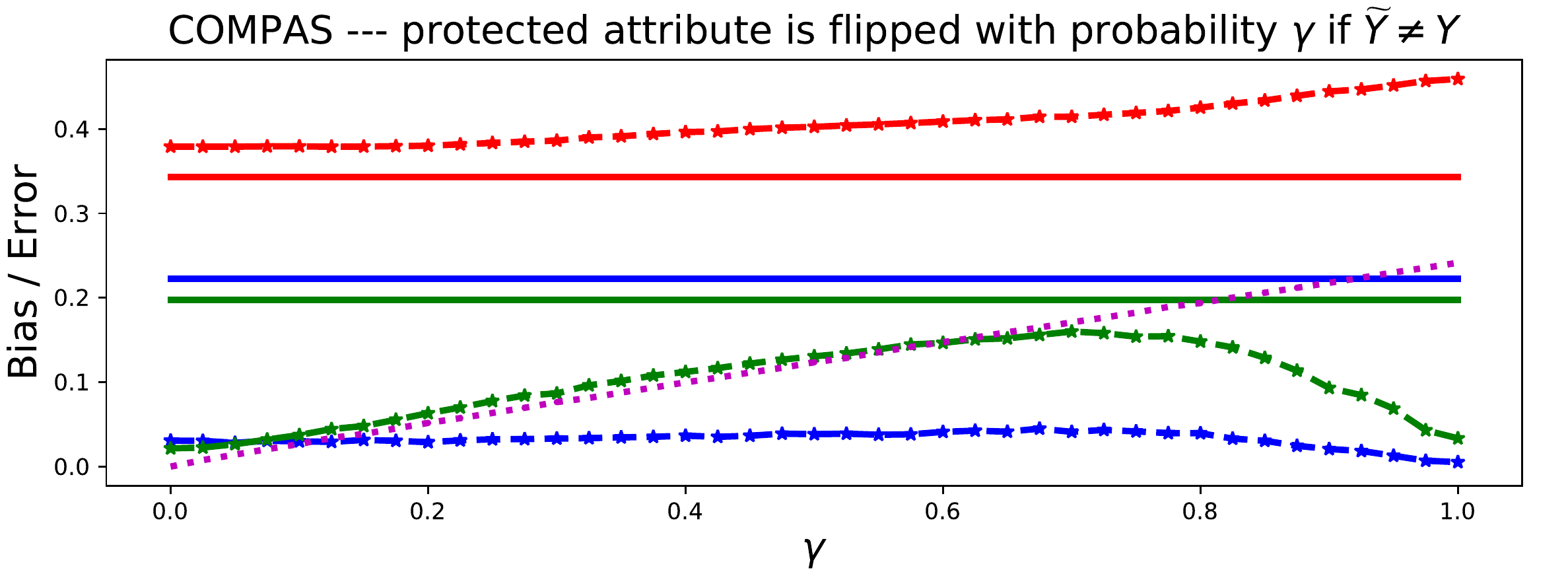}
 \hspace{8mm}
 \includegraphics[width=\wire]{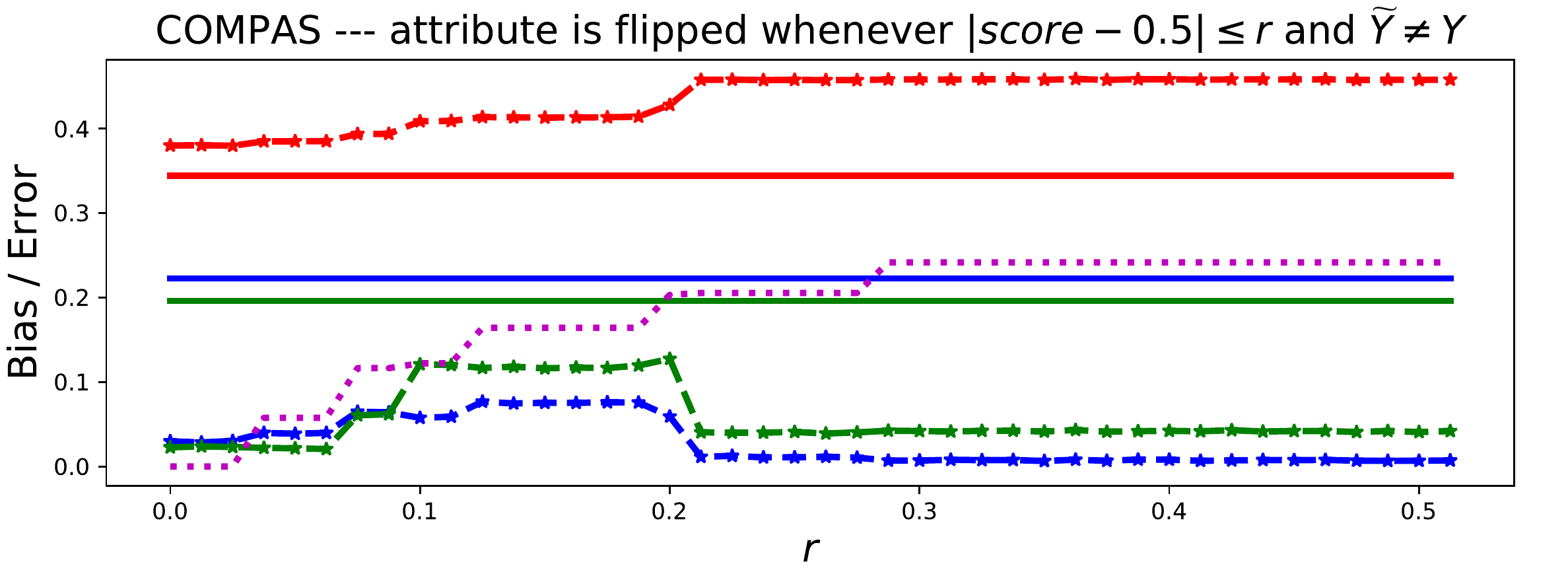}

\caption{COMPAS data set. 
$\Bias_{Y=+1 / -1}(\widehat{Y})$ (dashed blue / dashed green) 
and $\Error(\widehat{Y})$ (dashed red) as a function of the perturbation level 
in four perturbation scenarios.
The solid lines show the bias (blue and green) and the error (red) of 
$\widetilde{Y}$. 
The magenta line shows an estimate of \eqref{independence_measure} and how heavily Assumptions~\ref{assu_bias}\,\ref{assu_bias_a} is violated.}\label{figure_experiments_real_data}
\end{figure*}

In our 
second 
experiment, we run the 
EO 
method on two real data sets when we artificially perturb the protected attribute in one of four ways: either 
 we  set 
 the
 attribute of each data point to its complementary value independently with probability~$\gamma$, or we 
deterministically flip the attribute of every data point whose score lies in the interval 
$[0.5-r,0.5+r]$, or we perturb the attribute in one of these two ways only for those data points for which $\widetilde{Y}\neq Y$. 
The score of a data point is the likelihood predicted by a classifier for the data point to belong to the 
class~$Y=1$  
and is related to the given predictor~$\widetilde{Y}$ in that $\widetilde{Y}$ 
predicts $+1$ whenever the score is greater than 
$0.5$. 
We build upon the data provided by \citet{pleiss2017}. 
It contains the ground-truth labels, the true protected attributes and the
predicted scores for 
the COMPAS criminal recidivism risk assessment data set \citep{dieterich2016} and 
the Adult data set \citep{Dua:2019}. 
The scores for the COMPAS data set are the actual scores from the COMPAS risk assessment tool, the scores for the Adult data set are obtained from 
a multilayer perceptron. We randomly split the 
data sets into a training and a test set of equal size (we report several statistics such as the sizes of the original data sets in Appendix~\ref{supp_mat_setion_stats_real_data}).
Figure~\ref{figure_experiments_real_data} 
shows the bias and the error of 
$\widetilde{Y}$
and the derived 
EO 
predictor~$\widehat{Y}$ as well as an estimate of
\eqref{independence_measure} 
in the four perturbation scenarios as a function of the perturbation level $\gamma$ and~$r$, respectively, for the COMPAS data set. 
Figure~\ref{figure_experiments_real_data_APPENDIX} in Appendix~\ref{supp_mat_setion_stats_real_data} shows analogous plots for the Adult data set.  
The shown curves are obtained 
from averaging the results of 200 runs of the experiment. 
In the first two perturbation scenarios, where \eqref{independence_measure} is small and 
Assumptions~\ref{assu_bias}\,\ref{assu_bias_a}  (almost) satisfied, 
the curves look quite similar to the ones that we obtained in the experiments of 
Section~\ref{subsec_simulations}. In the third and the fourth perturbation scenario, 
Assumptions~\ref{assu_bias}\,\ref{assu_bias_a} is clearly violated, and here the error of $\widehat{Y}$ does not initially decrease. 
Also, for the Adult data set, the bias of  $\widehat{Y}$ explodes even for a moderate perturbation level, which once again 
shows that our identified Assumptions~\ref{assu_bias} are necessary for guaranteeing \eqref{claim_bias}. 
Overall, also the findings of this experiment confirm the main claims of our paper.

\subsection{Repeated Loss Minimization}\label{subsec_exp_repeated_loss_minimization}

 As another application of our results, 
 we compare the 
 EO 
 method to the method of \citet{Hashimoto2018}, discussed in Section~\ref{section_related_work}, in 
 a sequential classification setting. 
 This 
 experiment is presented in Appendix~\ref{supp_mat_rep_loss_min}. It shows that just as the method of  \citeauthor{Hashimoto2018}, 
 the 
 EO
 method can help avoid disparity amplification, even when the protected attribute is
 highly perturbed.

\section{DISCUSSION}\label{section_discussion}

We studied the EO  
postprocessing 
method of \citet{hardt2016equality} for fair classification when 
the protected attribute is perturbed. 
We identified 
conditions on the perturbation that guarantee
that the bias of a classifier is reduced even by
running 
the 
EO 
method 
with the perturbed
attribute. 
We showed that our conditions are necessary for providing such a guarantee. 
For the error of the resulting classifier, we 
empirically observed that 
under our conditions and a mild additional assumption, most often the error is not larger than the error of the 
EO 
classifier based on the true 
attribute. In a special 
case, we formally proved this observation. 
 Importantly, we \emph{analyzed} the 
EO 
method as it is and did not try to \emph{modify}
it in order to make it more robust.
We believe that 
often 
the practitioner with domain knowledge 
can 
assess 
whether our 
conditions hold 
and hence 
will 
benefit from our analysis. 
In contrast, 
modifying the 
method  
would require additional knowledge about the perturbation 
probabilities $\Pro\left[\As\neq A |Y=y, A=a\right]$ (e.g., access to some estimates or 
knowledge about their order) 
that the practitioner often~does~not~have.

There are several directions for future work: 
generally, one could analyze any of the many existing methods for fair ML 
(cf. Section~\ref{section_related_work}) 
with respect to a perturbation of~the~protected~attribute.  
Specifically related to our paper, 
a key question is to fully understand when inequality~\eqref{inequality_error} holds and to provide upper bounds on its violation in case it does not hold.  
It would also be interesting to study 
alternative noise models in which the 
attribute is also corrupted in the test phase (where the corruption can be caused by either the same or a different mechanism as in the training phase). 
Finally, it would be interesting to extend our results to multiple 
groups (i.e., a non-binary attribute) or when one only requires 
$\widehat{Y}$ to have equal true positive rates 
(aka equality of opportunity). Based on our intuition as outlined in Section~\ref{subsection_intuition}, we 
believe 
that such extensions are possible, 
but they still need to be formally~established.

\bibliography{mybibfile_fairness}
\bibliographystyle{plainnat}

\newpage
\onecolumn
\normalsize

\runningtitle{Appendix to \emph{Equalized odds postprocessing under imperfect group information}}

\appendix

\newcommand{\abstparaapp}{4mm}

\section{APPENDIX}

\subsection{
List of Random Variables Used in the Paper}\label{appendix_table_of_notation}

\begin{table}[h!]
  \caption{Random variables used in the paper.}\label{table_of_notation}
  \centering
\renewcommand{\arraystretch}{1.4}
\begin{tabular}{cc l l}
\toprule
 Variable & Range &  Meaning \\
 \midrule
 $X$ & $\mathcal{X}$ & data point (i.e., features representing a data point)\\
$A$ & $\{0,1\}$ & true protected attribute \\ 
$\As$ & $\{0,1\}$ & perturbed / corrupted protected attribute  \\
$Y$ & $\{-1,+1\}$ & ground-truth label \\
$\widetilde{Y}$ & $\{-1,+1\}$ &  given predictor for predicting $Y$ \\
$\widehat{Y}$ & $\{-1,+1\}$ & EO predictor derived from $\widetilde{Y}$ and based on some protected attribute (i.e., $A$ or $\As$)  \\ 
$\Ycorr$ & $\{-1,+1\}$ & EO predictor derived from $\widetilde{Y}$ and based on $\As$\\
$\Ytrue$ & $\{-1,+1\}$ & EO predictor derived from $\widetilde{Y}$ and based on $A$  \\
\bottomrule
\end{tabular}
\end{table}

\vspace{\abstparaapp}

\subsection{Proofs}\label{supp_mat_proofs}

We first require a simple technical lemma.

\vspace{2mm}

\begin{lemma}\label{lem:sum-of-probabilities_gen}
Let $D=[0,1)\times [0,1)\times(0,1)$ and consider $F:D \rightarrow \R$ with
\begin{align}\label{definition_function_F}
F(\gamma_1,\gamma_2,p)=\frac{\gamma_1 p}{\gamma_1 p + (1-\gamma_2)(1-p)}  - \frac{(1-\gamma_1)p}{(1-\gamma_1)p + \gamma_2(1-p)} + 1.
\end{align}
We have:
\begin{enumerate}[leftmargin=1cm,label=(\roman*)]
\item $0\leq F(\gamma_1,\gamma_2,p)\leq 2$ for all $(\gamma_1,\gamma_2,p)\in D$
 \item $F(0,0,p)=0$ for all $p\in (0,1)$
 \item $F(\gamma_1,\gamma_2,p)<1$ for all $(\gamma_1,\gamma_2,p)\in D$ with $\gamma_1+\gamma_2<1$ 
\item $F(\gamma_1,\gamma_2,p)=1$ for all $(\gamma_1,\gamma_2,p)\in D$ with $\gamma_1+\gamma_2=1$
\item $F(\gamma_1,\gamma_2,p)=F(\gamma_2,\gamma_1,1-p)$ for all $(\gamma_1,\gamma_2,p)\in D$
\item $\frac{\partial}{\partial \gamma_1} F(\gamma_1,\gamma_2,p)> 0$ and $\frac{\partial}{\partial \gamma_2} F(\gamma_1,\gamma_2,p)> 0$ for all $(\gamma_1,\gamma_2,p)\in D$ 
\end{enumerate}

\end{lemma}

\begin{proof}
First note that for $(\gamma_1,\gamma_2,p)\in D$ both denominators are greater than zero and $F$ is well-defined. Both fractions are not smaller than zero and not greater than one, 
which implies 
(i).
It is trivial to show (ii). 
It is 
\begin{small}
\begin{align*}
\frac{\gamma_1 p}{\gamma_1 p + (1-\gamma_2)(1-p)}   - \frac{(1-\gamma_1)p}{(1-\gamma_1)p + \gamma_2(1-p)}
=\frac{p(1-p)[\gamma_1+\gamma_2-1]}{\big[\gamma_1 p + (1-\gamma_2)(1-p)\big]\cdot \big[(1-\gamma_1)p + \gamma_2(1-p)\big]},
\end{align*}
\end{small}
from which (iii), (iv) and (v) follow. Finally, it is
\begin{small}
\begin{align*}
 \frac{\partial}{\partial \gamma_1} F(\gamma_1,\gamma_2,p)&=\frac{\partial}{\partial \gamma_1} \frac{p(1-p)[\gamma_1+\gamma_2-1]}{\big[\gamma_1 p + (1-\gamma_2)(1-p)\big]\cdot \big[(1-\gamma_1)p + \gamma_2(1-p)\big]}\\
&=\frac{p(1-p)\Big[1-(\gamma_1+\gamma_2-1)\cdot\big\{p\cdot[(1-\gamma_1)p + \gamma_2(1-p)\big]-p\cdot \big[\gamma_1 p + (1-\gamma_2)(1-p)\big]\big\}\Big]}{\big[\gamma_1 p + (1-\gamma_2)(1-p)\big]^2\cdot \big[(1-\gamma_1)p + \gamma_2(1-p)\big]^2}.
\end{align*}
\end{small}
We have
\begin{small}
\begin{align*}
 \left|p\cdot[(1-\gamma_1)p + \gamma_2(1-p)\big]-p\cdot \big[\gamma_1 p + (1-\gamma_2)(1-p)\big]\right|&=|p|\cdot\left|[p(1-2\gamma_1)+(1-p)(2\gamma_2-1)]\right|\\
 &\leq |p|
\end{align*}
\end{small}
for all $(\gamma_1,\gamma_2,p)\in D$ and hence
\begin{small}
\begin{align*}
 1-(\gamma_1+\gamma_2-1)\cdot\big\{p\cdot[(1-\gamma_1)p + \gamma_2(1-p)\big]-p\cdot \big[\gamma_1 p + (1-\gamma_2)(1-p)\big]\big\}\geq~~~~~~~~~~~~~~~~~~~~~~~\\
 1-\left|\gamma_1+\gamma_2-1\right|\cdot\left|p\cdot[(1-\gamma_1)p + \gamma_2(1-p)\big]-p\cdot \big[\gamma_1 p + (1-\gamma_2)(1-p)\big]\right|\geq 1-p>0.
 \end{align*}
 \end{small}
This shows $\frac{\partial}{\partial \gamma_1} F(\gamma_1,\gamma_2,p)> 0$. It follows from (v) that also $\frac{\partial}{\partial \gamma_2} F(\gamma_1,\gamma_2,p)> 0$ for all $(\gamma_1,\gamma_2,p)\in D$. 
\end{proof}

\vspace{8mm}
Now we can prove Theorem~\ref{theorem_bias}.

\vspace{2mm}
\textbf{Proof of Theorem~\ref{theorem_bias}:}

\vspace{2mm}
Let
\begin{small} \begin{align}\label{def_alpha_beta}
\begin{split}
\alpha_1 &:= \Psymb\left[\widetilde{Y}=1 \,\big|\, Y=1, A=0\right], \qquad \beta_1 := \Psymb\left[\widetilde{Y}=1 \,\big|\, Y=1, A=1\right],\\[2pt]
\alpha_2 &:= \Psymb\left[\widetilde{Y}=1 \,\big|\, Y=-1, A=0\right], \qquad \beta_2 := \Psymb\left[\widetilde{Y}=1 \,\big|\, Y=-1, A=1\right].
 \end{split}
 \end{align} \end{small}
Then 
\begin{small} \begin{align}\label{bias_tildeY}
 \Bias_{Y=+1}(\widetilde{Y})=|\alpha_1-\beta_1|, \quad  \Bias_{Y=-1}(\widetilde{Y})=|\alpha_2-\beta_2|.
 \end{align} \end{small}

When computing the probabilities $p_{-1,0},p_{-1,1},p_{1,0},p_{1,1}$ for $\Ycorr$, we have to replace $\Pro\left[Y=y, A=a, \widetilde{Y}=\tilde{y}\right]$ and 
$\Pro\left[\widetilde{Y}=1 \,\big|\, Y=y,A=a\right]$  by 
$\Pro\left[Y=y, \As=a, \widetilde{Y}=\tilde{y}\right]$ and $\Pro\left[\widetilde{Y}=1 \,\big|\, Y=y,\As=a\right]$, respectively, in the linear program~\eqref{eq_odds_linear_program}.
Note that the assumption $\Pro\left[\As\neq A\,|\, A=a, Y=y\right]<1$ for $y\in\{-1,+1\}$ and $a\in\{0,1\}$ implies that $\Pro[Y=y,\As=a]>0$ for $y\in\{-1,+1\}$ and $a\in\{0,1\}$. 
It is
\begin{small} \begin{align*}
\Pro\left[Y=y, \As=a, \widetilde{Y}=\tilde{y}\right]=\Pro\left[\widetilde{Y}=\tilde{y}\,\big|\, Y=y, \As=a \right]\cdot \Pro\left[Y=y, \As=a\right]
 \end{align*} \end{small}
and because of Assumptions~\ref{assu_bias}\,\ref{assu_bias_a}, for $a\in\{0,1\}$,
\begin{small} \begin{align*}
\Psymb\left[\widetilde{Y}=1 \,\big|\, Y=1, \As=a\right]=\,&\beta_1\cdot\Psymb\left[A=1 \,|\, Y=1, \As=a\right]+\alpha_1\cdot\left(1-\Psymb\left[A=1 \,|\, Y=1, \As=a\right]\right),\\[2pt]
\Psymb\left[\widetilde{Y}=1 \,\big|\, Y=-1, \As=a\right]=\,&\beta_2\cdot\Psymb\left[A=1 \,|\, Y=-1, \As=a\right]+\alpha_2\cdot\left(1-\Psymb\left[A=1 \,|\, Y=-1, \As=a\right]\right).
 \end{align*} \end{small}

Hence, we end up with  the new linear program
\begin{small}
\begin{align}\label{lin_pr_supp}
\begin{split}
&\min_{\substack{p_{1,0},~p_{1,1},\\p_{-1,0},~p_{-1,1}\in[0,1]}}~ \sum_{\substack{y\in\{-1,+1\}\\a\in\{0,1\}}}
\left\{\Pro\left[Y=-1, \As=a, \widetilde{Y}=y\right]-\Pro\left[Y=1, \As=a, \widetilde{Y}=y\right]\right\}\cdot p_{y,a}\\[3pt]
&~\text{s.t.}~~\left\{\beta_1\cdot\Psymb[A=1  \,|\, Y=1, \As=0]+\alpha_1\cdot(1-\Psymb[A=1  \,|\, Y=1, \As=0])\right\}\cdot p_{1,0}~~~~~~~~~~~~~~~~~~\\
&~~~~~~~~~~~~~~~~~~ +\left\{1-\beta_1\cdot\Psymb[A=1  \,|\, Y=1, \As=0]-\alpha_1\cdot(1-\Psymb[A=1  \,|\, Y=1, \As=0])\right\}\cdot  p_{-1,0}=\\
&~~~~~~~~~~\left\{\beta_1\cdot\Psymb[A=1  \,|\, Y=1, \As=1]+\alpha_1\cdot(1-\Psymb[A=1  \,|\, Y=1, \As=1])\right\}\cdot p_{1,1}\\
&~~~~~~~~~~~~~~~~~~~ +\left\{1-\beta_1\cdot\Psymb[A=1  \,|\, Y=1, \As=1]-\alpha_1\cdot(1-\Psymb[A=1  \,|\, Y=1, \As=1])\right\}\cdot p_{-1,1}, \\[5pt]
&~~~~~~~~\left\{\beta_2\cdot\Psymb[A=1  \,|\, Y=-1, \As=0]+\alpha_2\cdot(1-\Psymb[A=1  \,|\, Y=-1, \As=0])\right\}\cdot p_{1,0}\\
&~~~~~~~~~~~~~~~~ + \left\{1-\beta_2\cdot\Psymb[A=1  \,|\, Y=-1, \As=0]-\alpha_2\cdot(1-\Psymb[A=1  \,|\, Y=-1, \As=0])\right\}\cdot p_{-1,0}=\\
&~~~~~~~~~~ \left\{\beta_2\cdot\Psymb[A=1  \,|\, Y=-1, \As=1] +\alpha_2\cdot(1-\Psymb[A=1  \,|\, Y=-1, \As=1])\right\}\cdot p_{1,1}\\
&~~~~~~~~~~~~~~~~~~~+ \left\{1-\beta_2\cdot\Psymb[A=1  \,|\, Y=-1, \As=1] -\alpha_2\cdot(1-\Psymb[A=1  \,|\, Y=-1, \As=1])\right\}\cdot p_{-1,1}.
\end{split}
\end{align}
\end{small}

Some elementary calculations yield that the objective function $\Delta=\Delta(p_{1,0},p_{1,1},p_{-1,0},p_{-1,1})$ in \eqref{lin_pr_supp} equals 
\begin{small} \begin{align}\label{form_for_details_1}
\begin{split}
\Delta=\,\Psymb\left[Y=-1, \As=0\right] \Big[ (p_{1,0} - p_{-1,0}) \cdot\left\{ \alpha_2 + (\beta_2 - \alpha_2)\cdot\Psymb\left[A=1  \,|\, Y=-1, \As =0\right] \right\} + p_{-1,0}\Big]&\\
+ \,\Psymb\left[Y=-1, \As=1\right] \Big[ (p_{1,1} - p_{-1,1}) \cdot\left\{ \alpha_2 + (\beta_2 - \alpha_2)\cdot\Psymb\left[A=1  \,|\, Y=-1, \As =1\right]\right\}  + p_{-1,1} \Big]&\\
- \,\Psymb\left[Y=1, \As=0\right] \Big[ (p_{1,0} - p_{-1,0}) \cdot\left\{ \alpha_1 + (\beta_1 - \alpha_1)\cdot\Psymb\left[A=1  \,|\, Y=1, \As =0\right] \right\} + p_{-1,0} \Big]&\\
- \,\Psymb\left[Y=1, \As=1\right] \Big[(p_{1,1} - p_{-1,1})\cdot \left\{ \alpha_1 + (\beta_1 - \alpha_1)\cdot\Psymb\left[A=1  \,|\, Y=1, \As =1\right]\right\}  + p_{-1,1} \Big]&.
\end{split}
\end{align}
\end{small}
and that the constraints are equivalent to
\begin{small} \begin{align}\label{form_for_details_2}
\begin{split}
&(p_{1,0} - p_{-1,0})\cdot \left\{ \alpha_1 + (\beta_1 - \alpha_1)\cdot\Psymb\left[A=1  \,|\, Y=1, \As =0\right] \right\} + p_{-1,0} \\
&~~~~~~~~~~~~~~~~~~~~~~~~= (p_{1,1} - p_{-1,1})\cdot \left\{ \alpha_1 + (\beta_1 - \alpha_1)\cdot\Psymb\left[A=1  \,|\, Y=1, \As =1\right] \right\}+ p_{-1,1},\\[6pt]
&(p_{1,0} - p_{-1,0})\cdot \left\{ \alpha_2 + (\beta_2 - \alpha_2)\cdot\Psymb\left[A=1  \,|\, Y=-1, \As =0\right] \right\} + p_{-1,0}\\
&~~~~~~~~~~~~~~~~~~~~~~~~= (p_{1,1} - p_{-1,1})\cdot \left\{ \alpha_2 + (\beta_2 - \alpha_2)\cdot\Psymb\left[A=1  \,|\, Y=-1, \As =1\right] \right\}+ p_{-1,1}.
 \end{split}
 \end{align} \end{small}
Let
\begin{small} \begin{align}
 e&:=\alpha_1 + (\beta_1 - \alpha_1)\cdot\Psymb\left[A=1  \,|\, Y=1, \As =0\right],\label{def_e}\\
f&:=\alpha_1 + (\beta_1 - \alpha_1)\cdot\Psymb\left[A=1  \,|\, Y=1, \As =1\right],\label{def_f}\\
g&:= \alpha_2 + (\beta_2 - \alpha_2)\cdot\Psymb\left[A=1  \,|\, Y=-1, \As =0\right],\label{def_g}\\
h&:= \alpha_2 + (\beta_2 - \alpha_2)\cdot\Psymb\left[A=1  \,|\, Y=-1, \As =1\right].\label{def_h}
 \end{align} \end{small}
Then the constraints are
\begin{small} \begin{align}
(p_{1,0} - p_{-1,0})\cdot e + p_{-1,0} = (p_{1,1} - p_{-1,1})\cdot f + p_{-1,1},\label{proof_bias_cond1}\\
(p_{1,0} - p_{-1,0})\cdot g + p_{-1,0} = (p_{1,1} - p_{-1,1})\cdot h + p_{-1,1}.\notag
 \end{align} \end{small}

Because of the constraints we have 
\begin{small} \begin{align}\label{delta_two_param}
\begin{split}
\Delta &= p_{-1,0}\cdot \left\{ \Psymb[Y=-1] - \Psymb[Y=1]\right\} + (p_{1,0} - p_{-1,0})\cdot u\\
&= p_{-1,1}\cdot \left\{ \Psymb[Y=-1] - \Psymb[Y=1]\right\} + (p_{1,1} - p_{-1,1})\cdot v,
\end{split}
 \end{align} \end{small}
where
\begin{small} \begin{align}\label{def_u_and_v}
u := g\cdot\Psymb[Y=-1] - e\cdot \Psymb[Y=1],\qquad v := h \cdot\Psymb[Y=-1] - f \cdot\Psymb[Y=1].
 \end{align} \end{small}
If $u=0$ or $v=0$, 
one optimal solution to \eqref{lin_pr_supp} is $p_{1,0}=p_{1,1}=p_{-1,0}=p_{-1,1}=1$ or $p_{1,0}=p_{1,1}=p_{-1,0}=p_{-1,1}=0$, depending on whether  $\Psymb[Y=-1] \leq \Psymb[Y=1]$ or 
$\Psymb[Y=-1] > \Psymb[Y=1]$. In this case the derived equalized odds predictor $\Ycorr$ is the constant predictor 
$\Ycorr=+1$ or $\Ycorr=-1$ with $\Bias_{Y=y}(\Ycorr)=0$, $y\in\{-1,+1\}$, and 
\eqref{theorem_bias_formula} is true.

\vspace{2mm}
So let us assume that $u\neq 0$ and $v\neq 0$. 
Let $\theta := \Psymb[Y=-1] - \Psymb[Y=1]$.  Because of
\begin{small} \begin{align*}
\Psymb\left[\Ycorr=1\,\big|\,Y=1,A=0\right]
&=p_{1,0}\cdot\alpha_1+p_{-1,0}\cdot(1-\alpha_1),\\ 
\Psymb\left[\Ycorr=1\,\big|\,Y=1,A=1\right]
&=p_{1,1}\cdot\beta_1+p_{-1,1}\cdot(1-\beta_1),\\ 
\Psymb\left[\Ycorr=1\,\big|\,Y=-1,A=0\right]
&=p_{1,0}\cdot\alpha_2+p_{-1,0}\cdot(1-\alpha_2),\\ 
\Psymb\left[\Ycorr=1\,\big|\,Y=-1,A=1\right]
&=p_{1,1}\cdot\beta_2+p_{-1,1}\cdot(1-\beta_2),
 \end{align*} \end{small}
we have
\begin{small} \begin{align}\label{bias_widehatY_in_proof_theo_bias}
 \begin{split}
 \Bias_{Y=+1}(\Ycorr) &= \left|\alpha_1 \cdot(p_{1,0} - p_{-1,0}) - \beta_1\cdot (p_{1,1} - p_{-1,1}) + p_{-1,0} - p_{-1,1}\right|,\\
 \Bias_{Y=-1}(\Ycorr) &= \left|\alpha_2  \cdot (p_{1,0} - p_{-1,0}) - \beta_2  \cdot (p_{1,1} - p_{-1,1}) + p_{-1,0} - p_{-1,1}\right|.
  \end{split}
  \end{align} \end{small}

It is
\begin{small} \begin{align*}
\Bias_{Y=+1}(\Ycorr)&\stackrel{\eqref{delta_two_param}}{=} \left|\frac{\Delta \alpha_1}{u} - \frac{\Delta \beta_1}{v} + p_{-1,0}\left(1-\frac{\theta \alpha_1}{u}\right) - p_{-1,1}\left(1-\frac{\theta \beta_1}{v}\right)\right|\\
&~= \left|\frac{\Delta \alpha_1}{u} - \frac{\Delta \beta_1}{v} + p_{-1,0}\left(1-\frac{\theta e}{u}\right) - 
p_{-1,1}\left(1-\frac{\theta f}{v}\right) + p_{-1,0}\frac{\theta (e-\alpha_1)}{u} - p_{-1,1}\frac{\theta(f-\beta_1)}{v}\right|.
\end{align*}
\end{small}
From \eqref{proof_bias_cond1} and \eqref{delta_two_param} we obtain that
\begin{small} \begin{align*}
p_{-1,0}\left(1-\frac{\theta e}{u}\right) - p_{-1,1}\left(1-\frac{\theta f}{v}\right) = \frac{\Delta f}{v} - \frac{\Delta e}{u}.
 \end{align*} \end{small}
From this we get that 
 \begin{small} \begin{align}\label{bound_bias_1}
\begin{split}
&\Bias_{Y=+1}(\Ycorr) = \left|\left(\frac{\Delta}{u}- \frac{p_{-1,0} \theta}{u}\right)(\alpha_1-e) - \left(\frac{\Delta}{v} - \frac{p_{-1,1} \theta}{v}\right)(\beta_1-f)\right|\\[1pt]
&~~~~~~~~~~~~~\stackrel{\eqref{def_e}\&\eqref{def_f}}{=} |\alpha_1 - \beta_1|\cdot \left| \left(\frac{\Delta}{u}- \frac{p_{-1,0} \theta}{u}\right)\cdot\Psymb[A=1 \,|\, Y=1, \As=0] +\left(\frac{\Delta}{v} - \frac{p_{-1,1} \theta}{v}\right)\cdot \Psymb[A=0\,|\,Y=1, \As=1]\right|\\[1pt]
&~~~~~~~~~~~~~~~~\stackrel{\eqref{delta_two_param}}{=} |\alpha_1 - \beta_1|\cdot \left| (p_{1,0}-p_{-1,0})\cdot\Psymb[A=1 \,|\, Y=1, \As=0] +(p_{1,1}-p_{-1,1})\cdot\Psymb[A=0\,|\,Y=1, \As=1] \right|\\[1pt]
&~~~~~~~~~~~~~~~~~\leq  |\alpha_1 - \beta_1|\cdot\left\{\Psymb[A=1 \,|\, Y=1, \As=0] + \Psymb[A=0\,|\,Y=1, \As=1]\right\},
\end{split}
 \end{align} \end{small}
  where the last inequality follows from the triangle inequality and $|p_{1,0}-p_{-1,0}|\leq 1$ and $|p_{1,1}-p_{-1,1}|\leq 1$ because of 
$p_{-1,0},p_{-1,1},p_{1,0},p_{1,1}\in[0,1]$.

\vspace{2mm}
Similarly, we obtain
\begin{small} \begin{align}\label{bound_bias_2}
\Bias_{Y=-1}(\Ycorr) &\leq |\alpha_2 - \beta_2| \cdot\left\{\Psymb[A=1 \,|\, Y=-1, \As=0] + \Psymb[A=0\,|\,Y=-1, \As=1]\right\}.
 \end{align} \end{small}

It is, for $y\in\{-1,+1\}$,
\begin{small}
\begin{align}\label{prob_in_proof_end_1}
\begin{split}
&\Psymb[A=1 \,|\, Y=y, \As=0]=\frac{\Psymb[A=1, \As=0 \,|\, Y=y]}{\Psymb[\As=0 \,|\, Y=y]}\\
&~~~~~~~~=\frac{\Psymb[\As=0\,|\, Y=y,A=1]\cdot \Psymb[A=1\,|\, Y=y]}{\Psymb[\As=0\,|\, Y=y,A=1]\cdot \Psymb[A=1\,|\, Y=y]+\Psymb[\As=0\,|\, Y=y,A=0]\cdot 
\Psymb[A=0\,|\, Y=y]}
\end{split}
\end{align}
\end{small}
and $\Psymb[A=0 \,|\, Y=y, \As=1]=1-\Psymb[A=1 \,|\, Y=y, \As=1]$ with
\begin{small}
\begin{align}\label{prob_in_proof_end_2}
\begin{split}
&\Psymb[A=1 \,|\, Y=y, \As=1]=\frac{\Psymb[A=1, \As=1 \,|\, Y=y]}{\Psymb[\As=1 \,|\, Y=y]}\\
&~~~~~~~=\frac{\Psymb[\As=1\,|\, Y=y,A=1]\cdot \Psymb[A=1\,|\, Y=y]}{\Psymb[\As=1\,|\, Y=y,A=1]\cdot \Psymb[A=1\,|\, Y=y]+
\Psymb[\As=1\,|\, Y=y,A=0]\cdot \Psymb[A=0\,|\, Y=y]}.
\end{split}
\end{align}
\end{small}

 Combining \eqref{bias_tildeY}, \eqref{bound_bias_1}, \eqref{bound_bias_2}, \eqref{prob_in_proof_end_1}, \eqref{prob_in_proof_end_2}
and Lemma~\ref{lem:sum-of-probabilities_gen} yields Theorem~\ref{theorem_bias}. \hfill$\square$

\vspace{8mm}
We prove Lemma~\ref{lemma_bias_assumption} by means of counterexamples.

\vspace{2mm}
\textbf{Proof of Lemma~\ref{lemma_bias_assumption}:}

\begin{itemize}
 \item Assumptions~\ref{assu_bias}\,\ref{assu_bias_a} violated \& Assumptions~\ref{assu_bias}\,\ref{assu_bias_b} satisfied: 
\end{itemize}

\vspace{-2mm}
Assume that 
\begin{small}
\begin{align}\label{proof_lemma_necessary_case1}
\begin{split}
\Pro\left[Y=y,A=a\right]&=\frac{1}{4}, \quad y\in\{-1,+1\}, a\in\{0,1\},\\
\Psymb\left[\widetilde{Y}=1 \,\big|\, Y=1, A=0\right]&=0.65, \qquad \Psymb\left[\widetilde{Y}=1 \,\big|\, Y=1, A=1\right]=0.6,\\
\Psymb\left[\widetilde{Y}=1 \,\big|\, Y=-1, A=0\right]&=0, \qquad \Psymb\left[\widetilde{Y}=1 \,\big|\, Y=-1, A=1\right]=0
\end{split}
\end{align}
\end{small}
and that
\begin{small}
\begin{align*}
\Pro\left[\As\neq A \,\big|\, Y=1,A=0,\widetilde{Y}=-1\right]=0.15, \quad~~ \Pro\left[\As\neq A \,\big|\, Y=y,A=a,\widetilde{Y}=\tilde{y}\right]=0, \quad (y,a,\tilde{y})\neq(1,0,-1).
\end{align*}
\end{small}
Then $\Pro\left[\As\neq A \,\big|\, Y=1,A=0\right]=0.15\cdot 0.35=0.0525$ and $\Pro\left[\As\neq A \,\big|\, Y=y,A=a\right]=0$, $(y,a)\neq(1,0)$, and Assumptions~\ref{assu_bias}\,\ref{assu_bias_b} is satisfied. 
However, Assumptions~\ref{assu_bias}\,\ref{assu_bias_a} is not satisfied since 
$\Pro\left[\As=1 \,\big|\, Y=1,A=0,\widetilde{Y}=-1\right]\neq \Pro\left[\As=1 \,\big|\, Y=1,A=0,\widetilde{Y}=1\right]$. 
It is 
$\Bias_{Y=+1}(\widetilde{Y}) =0.05$ and $\Bias_{Y=-1}(\widetilde{Y}) =0$. 

It is straightforward to compute all probabilities~$\Pro\left[Y=y, \As=a, \widetilde{Y}=\tilde{y}\right]$ and $\Pro\left[\widetilde{Y}=1 \,\big|\, Y=y,\As=a\right]$ and solve the 
the linear program~\eqref{eq_odds_linear_program} with $\Pro\left[Y=y, A=a, \widetilde{Y}=\tilde{y}\right]$ and 
$\Pro\left[\widetilde{Y}=1 \,\big|\, Y=y,A=a\right]$ replaced by 
$\Pro\left[Y=y, \As=a, \widetilde{Y}=\tilde{y}\right]$ and $\Pro\left[\widetilde{Y}=1 \,\big|\, Y=y,\As=a\right]$, respectively. 
In doing so, one ends up with 
an optimal solution $(p_{-1,0}^*,p_{-1,1}^*,p_{1,0}^*,p_{1,1}^*)\approx(0,0,0.83,1)$. 
The bias of the equalized odds predictor~$\Ycorr$ for the class $Y=+1$ is
\begin{small}
\begin{align*}
 \Bias_{Y=+1}(\Ycorr) &= \left|\Psymb\left[\widetilde{Y}=1 \,\big|\, Y=1, A=0\right] \cdot(p_{1,0}^* - p_{-1,0}^*) - \Psymb\left[\widetilde{Y}=1 \,\big|\, Y=1, A=1\right]\cdot (p_{1,1}^* - p_{-1,1}^*) + p_{-1,0}^* - p_{-1,1}^*\right|\\
  &\approx \left|0.65 \cdot 0.83 -0.6\right|\approx 0.06>0.05=\Bias_{Y=+1}(\widetilde{Y}).
\end{align*}
\end{small}

\vspace{-4mm}
\begin{itemize}
 \item Assumptions~\ref{assu_bias}\,\ref{assu_bias_a} satisfied \& Assumptions~\ref{assu_bias}\,\ref{assu_bias_b} violated:
\end{itemize}

\vspace{-2mm}
The top left plot of Figure~\ref{figure_simulations} in Section~\ref{subsec_simulations} provides an example where Assumptions~\ref{assu_bias}\,\ref{assu_bias_a} is satisfied and  for 
$\Pro\left[\As\neq A \,\big|\, Y=1,A=0\right]=\Pro\left[\As\neq A \,\big|\, Y=1,A=1\right]>0.5$ (and hence Assumptions~\ref{assu_bias}\,\ref{assu_bias_b} being violated) we have 
$\Bias_{Y=+1}(\Ycorr)>\Bias_{Y=+1}(\widetilde{Y})$.
\hfill$\square$

\vspace{8mm}
Next, we prove Theorem~\ref{thm:accuracy-balanced}.

\vspace{2mm}
\textbf{Proof of Theorem~\ref{thm:accuracy-balanced}:}

\vspace{2mm}
We use the same notation as in the proof of Theorem~\ref{theorem_bias}. In particular, let 
$\alpha_1,\alpha_2,\beta_1,\beta_2$ be the probabilities defined in \eqref{def_alpha_beta}. Since we assume 
Assumption~\ref{assu_error} to hold, 
we have $\alpha_1>\alpha_2$ and $\beta_1>\beta_2$. 
Furthermore, without loss of generality, we may assume that $\alpha_2 \beta_1 \geq \alpha_1 \beta_2$ (otherwise, 
we can simply swap the role of the groups $A=0$ and $A=1$ so that this condition holds). 
 
 \vspace{2mm}
 Let $\gamma:=\Pro\left[\As\neq A\,|\, A=a, Y=y\right]$, which does not depend on the values of $a$ and $y$, 
 be the perturbation probability.   
 In the training phase for $\Ycorr$ we have $\gamma=\gamma_0$ for some $\gamma_0\in(0,\frac{1}{2}]$, 
 and in the training phase for $\Ytrue$ we have $\gamma=0$. 
 
Note that we have 
 $\Pro[Y=+1]=\Pro[Y=-1]=\frac{1}{2}$. It follows from \eqref{def_e} to \eqref{def_u_and_v}, \eqref{prob_in_proof_end_1} and \eqref{prob_in_proof_end_2}
that for any fixed value of the perturbation probability~$\gamma\in[0,1]$ the equalized odds method 
solves the following linear program:
\begin{small}\begin{align}\label{lin_pr_proof_theo2}
\begin{split}
&\min_{p_{1,0},p_{1,1},p_{-1,0},p_{-1,1}\in[0,1]}~  \Delta \\
&~\text{s.t.}~~(p_{1,0}-p_{-1,0})\cdot\{(1-\gamma)\alpha_1 + \gamma \beta_1\} + p_{-1,0} = (p_{1,1}-p_{-1,1})\cdot\{(1-\gamma)\beta_1 + \gamma \alpha_1\} + p_{-1,1},\\
&~~~~~~~~(p_{1,0}-p_{-1,0})\cdot\{(1-\gamma)\alpha_2 + \gamma \beta_2\} + p_{-1,0} = (p_{1,1}-p_{-1,1})\cdot\{(1-\gamma)\beta_2 + \gamma \alpha_2\} + p_{-1,1},
\end{split}
\end{align}\end{small}
where
\begin{small}\begin{align}\label{def_delta_proof_theo2}
 \Delta = (p_{1,0} - p_{-1,0})u = (p_{1,1}-p_{-1,1})v
\end{align}\end{small}
with
\begin{small}\begin{align}\label{defxy_proof_theo2}
u = \frac{1}{2}\left[(1-\gamma)(\alpha_2 - \alpha_1) + \gamma(\beta_2 - \beta_1)\right],\qquad v = \frac{1}{2}\left[(1-\gamma)(\beta_2 - \beta_1) + \gamma(\alpha_2 - \alpha_1)\right].
\end{align}\end{small}
Note that $u<0$ and $v<0$ for any $\gamma \in [0,1]$ 
because of  $\alpha_1>\alpha_2$ and $\beta_1>\beta_2$. 
Since $p_{1,0}=p_{1,1}=p_{-1,0}=p_{-1,1}=0$ satisfies the constraints in \eqref{lin_pr_proof_theo2} 
and has objective value $\Delta=0$, 
in an equalized odds solution (i.e., an optimal solution to \eqref{lin_pr_proof_theo2}) we must have 
 $\Delta\leq 0$, $p_{-1,0} \leq p_{1,0}$ and $p_{-1,1} \leq p_{1,1}$ for any $\gamma \in [0,1]$. 
Furthermore, for $\gamma \in [0,\frac{1}{2}]$ we obtain from the first constraint in \eqref{lin_pr_proof_theo2}
 that
\begin{small}\begin{align}\label{diff_pm10_pm11_proof_theo2}
\begin{split}
p_{-1,0} - p_{-1,1}&= (p_{1,1}-p_{-1,1})\cdot\{(1-\gamma)\beta_1 + \gamma \alpha_1\} - (p_{1,0}-p_{-1,0})\cdot\{(1-\gamma)\alpha_1 + \gamma \beta_1\}\\ 
&\stackrel{\eqref{def_delta_proof_theo2}}{=} \frac{\Delta}{v}((1-\gamma)\beta_1 + \gamma \alpha_1) - \frac{\Delta}{u}((1-\gamma)\alpha_1 + \gamma \beta_1)\\
&= \frac{\Delta}{uv}\Big( \beta_1((1-\gamma)u - \gamma v) - \alpha_1((1-\gamma)v - \gamma u) \Big)\\
&\stackrel{\eqref{defxy_proof_theo2}}{=} \frac{\Delta (1-2\gamma)}{2uv} (\alpha_2 \beta_1 - \alpha_1 \beta_2) \\
&\leq 0, 
\end{split}
\end{align}\end{small}
where the last inequality holds because of $\Delta \leq 0$, $1-2\gamma\geq 0$, $u< 0$, $v< 0$ and $\alpha_2 \beta_1 \geq \alpha_1 \beta_2$.
Hence, in an equalized odds solution, for any $\gamma \in [0,1/2]$, we must have 
 $p_{-1,0}\leq p_{-1,1}$ and 
$p_{-1,0}=\min\{p_{1,0},p_{1,1},p_{-1,0},p_{-1,1}\}$. 
It is straightforward to check that the error $\Error(\widehat{Y})$ of a derived equalized odds 
predictor~$\widehat{Y}$ with probabilities $p_{1,0},p_{1,1},p_{-1,0},p_{-1,1}$ is given by
\begin{small}\begin{align}\label{error_widehat_Y_proof_theo2}
 \Error(\widehat{Y})=\frac{1}{4}\cdot\left\{(p_{1,0} - p_{-1,0})(\alpha_2 - \alpha_1) + (p_{1,1} - p_{-1,1})(\beta_2 - \beta_1)\right\}+\frac{1}{2}
\end{align}\end{small}
and hence is invariant under translations of the probabilities (compare with the end of Section~\ref{section_equalized_odds_MAIN}).  
Hence, without loss of generality, we may assume  that $p_{-1,0}=0$.  
Substituting 
in the expressions computed above we get that
\begin{small}\begin{align}
p_{1,0} &\stackrel{\eqref{def_delta_proof_theo2}}{=} \frac{\Delta}{u},\label{p10_proof_theo2}\\
p_{-1,1} &\stackrel{\eqref{diff_pm10_pm11_proof_theo2}}{=} \frac{\Delta (1-2\gamma)}{2uv} (\alpha_1 \beta_2 - \alpha_2 \beta_1),\label{pm11_proof_theo2}\\
p_{1,1} &\stackrel{\eqref{def_delta_proof_theo2}}{=} \frac{\Delta}{v} + p_{-1,1} = \Delta \left[\frac{1}{v} + \frac{(1-2\gamma)(\alpha_1 \beta_2 - \alpha_2 \beta_1)}{2uv}\right].\label{p11_proof_theo2}
\end{align}\end{small}
The value of $\Delta$ must be the smallest value such that all these three probabilities are in $[0,1]$. 
It follows that in an equalized odds solution, for any $\gamma\in[0,\frac{1}{2}]$, either 
$p_{1,0}$ or $p_{1,1}$ (or both) equals $1$ and this depends on the sign of the difference 
\begin{small}\begin{align}\label{difference_proof_theo2}
\begin{split}
p_{1,0} - p_{1,1} &\stackrel{\eqref{p10_proof_theo2}\&\eqref{p11_proof_theo2}}{=} {\Delta} \Big(\frac{1}{u} - \frac{1}{v} - \frac{(1-2\gamma)(\alpha_1 \beta_2 - \alpha_2 \beta_1)}{2uv} \Big)\\
&~~~\stackrel{\eqref{defxy_proof_theo2}}{=} \frac{\Delta (1-2\gamma)}{2uv} \Big( \beta_2 - \beta_1  + \alpha_1 - \alpha_2 + \alpha_2 \beta_1 - \alpha_1 \beta_2 \Big).
\end{split}
\end{align}\end{small}
Importantly, the difference \eqref{difference_proof_theo2} 
has the same sign for any $\gamma \in [0,\frac{1}{2}]$. 
We distinguish two cases depending on whether $\beta_2 - \beta_1  + \alpha_1 - \alpha_2 + \alpha_2 \beta_1 - \alpha_1 \beta_2$ is smaller than zero or not:

\vspace{4mm}
\textbf{Case 1: $\beta_2 - \beta_1  + \alpha_1 - \alpha_2 + \alpha_2 \beta_1 - \alpha_1 \beta_2<0$}. In this case, 
for $\gamma \in [0,\frac{1}{2}]$, the  difference \eqref{difference_proof_theo2} is non-negative and we have $p_{1,0}=1$.

\vspace{1mm}
Let $p^0_{1,0}, p^0_{-1,0}, p^0_{1,1}, p^0_{-1,1}$ be an equalized odds solution for $\gamma=0$ 
(corresponding to $\Ytrue$) 
and $p^{\gamma_0}_{1,0}, p^{\gamma_0}_{-1,0}, p^{\gamma_0}_{1,1}, p^{\gamma_0}_{-1,1}$ be an equalized odds solution 
for $\gamma=\gamma_0\in(0,\frac{1}{2}]$ (corresponding to $\Ycorr$).  
It is $p^0_{1,0} = p^{\gamma_0}_{1,0}=1$ and $p^0_{-1,0} = p^{\gamma_0}_{-1,0}=0$. 
It follows from \eqref{error_widehat_Y_proof_theo2} that 
\begin{small}\begin{align*}
 \Error(\Ytrue) - \Error(\Ycorr) = \frac{1}{4}\cdot\{(p^0_{1,1} - p^0_{-1,1})(\beta_2 - \beta_1) -(p^{\gamma_0}_{1,1} - p^{\gamma_0}_{-1,1})(\beta_2 - \beta_1)\}.
\end{align*}\end{small}
Using the fact  that 
$(p^0_{1,0} - p^0_{-1,0})(\alpha_2 - \alpha_1) = (p^0_{1,1} - p^0_{-1,1})(\beta_2 - \beta_1)$, which follows from subtracting 
the first from the second constraint in \eqref{lin_pr_proof_theo2} with $\gamma=0$, we get that
\begin{small}\begin{align*}
 \Error(\Ytrue) - \Error(\Ycorr) =  \frac{1}{4}\cdot\{(\alpha_2 - \alpha_1) - (p^{\gamma_0}_{1,1} - p^{\gamma_0}_{-1,1})(\beta_2 - \beta_1)\}.
\end{align*}\end{small}
We write $u(\gamma_0)$ and $v(\gamma_0)$ for $u$ 
or $v$ 
with $\gamma=\gamma_0$. 
Because of $p^{\gamma_0}_{1,0} - p^{\gamma_0}_{-1,0}=1$, we have that 
\begin{small}\begin{align*}
p^{\gamma_0}_{1,1} - p^{\gamma_0}_{-1,1} \stackrel{\eqref{def_delta_proof_theo2}}{=} \frac{u(\gamma_0)}{v(\gamma_0)}
\end{align*}\end{small}
and hence
\begin{small}
\begin{align*}
\Error(\Ytrue) - \Error(\Ycorr) =  \frac{1}{4}\cdot\{(\alpha_2 - \alpha_1) - 
\frac{u(\gamma_0)}{v(\gamma_0)}(\beta_2 - \beta_1)\}\stackrel{\eqref{defxy_proof_theo2}}{=} 
\frac{\gamma_0}{4} \frac{(\alpha_2 - \alpha_1)^2 - (\beta_2-\beta_1)^2}{2v(\gamma_0)}.
\end{align*}
\end{small}
Because of $\beta_2 - \beta_1 + \alpha_1 - \alpha_2 + \alpha_2 \beta_1 - \alpha_1 \beta_2 < 0$ 
and $\alpha_2 \beta_1 - \alpha_1 \beta_2 \geq 0$, we have
$\beta_1 - \beta_2 > \alpha_1 - \alpha_2 > 0$, and because of $v(\gamma_0)<0$ it follows that
\begin{small}\begin{align*}
\Error(\Ytrue) - \Error(\Ycorr) > 0
\end{align*}\end{small}
for all $\gamma_0\in(0,\frac{1}{2}]$.

\vspace{4mm}
\textbf{Case 2: $\beta_2 - \beta_1  + \alpha_1 - \alpha_2 + \alpha_2 \beta_1 - \alpha_1 \beta_2\geq 0$}. 
In this case, 
for $\gamma \in [0,\frac{1}{2}]$, the  difference \eqref{difference_proof_theo2} is non-positive
and we have $p_{1,1}=1$.

\vspace{1mm}
As before in Case~1, let $p^0_{1,0}, p^0_{-1,0}, p^0_{1,1}, p^0_{-1,1}$ 
be an equalized odds solution for $\gamma=0$ 
(corresponding to $\Ytrue$) 
and $p^{\gamma_0}_{1,0}, p^{\gamma_0}_{-1,0}, p^{\gamma_0}_{1,1}, p^{\gamma_0}_{-1,1}$ be an equalized odds solution 
for $\gamma=\gamma_0\in(0,\frac{1}{2}]$ (corresponding to $\Ycorr$).  
It is $p^0_{1,1} = p^{\gamma_0}_{1,1}=1$ and $p^0_{-1,0} = p^{\gamma_0}_{-1,0}=0$. 
Similarly as in Case~1 we obtain that 
\begin{small}
\begin{align}\label{proof_theo2_errordiff_case2}
\Error(\Ytrue) - \Error(\Ycorr)& = \frac{1}{4}\Big\{2(1-p^0_{-1,1})(\beta_2 - \beta_1) - (1 - p^{\gamma_0}_{-1,1})(\beta_2 - \beta_1) 
-\frac{v(\gamma_0)}{u(\gamma_0)}(1-p^{\gamma_0}_{-1,1})(\alpha_2 - \alpha_1)\Big\}.
\end{align}\end{small}
When $p_{1,1}=1$, we obtain from \eqref{p11_proof_theo2} that 
\begin{small}\begin{align*}
\Delta = \frac{2uv}{2u + (1-2\gamma)(\alpha_1 \beta_2 - \alpha_2 \beta_1)}.
\end{align*}\end{small}
This implies that
\begin{small}\begin{align}\label{temp_proof_theo2_c2}
1-p^{\gamma_0}_{-1,1} \stackrel{\eqref{pm11_proof_theo2}}{=} \frac{2u(\gamma_0)}{2u(\gamma_0) + (1-2\gamma_0)(\alpha_1 \beta_2 - \alpha_2 \beta_1)}
\end{align}\end{small}
and 
\begin{small}\begin{align*}
1-p^0_{-1,1} \stackrel{\eqref{temp_proof_theo2_c2}\&\eqref{defxy_proof_theo2}}{=} \frac{\alpha_2 - \alpha_1}{\alpha_2 - \alpha_1 + \alpha_1 \beta_2 - \alpha_2 \beta_1}.
\end{align*}\end{small}
Substituting these in \eqref{proof_theo2_errordiff_case2} we get that
\begin{small}
\begin{align}\label{case2_temp}
\begin{split}
\Error(\Ytrue) - \Error(\Ycorr)&= \frac{1}{4}\left\{2\frac{(\beta_2 - \beta_1)(\alpha_2  - \alpha_1)}{\alpha_2 - \alpha_1 + \alpha_1 \beta_2 - \alpha_2 \beta_1} - \frac{(\beta_2 - \beta_1)2u(\gamma_0) +
(\alpha_2 - \alpha_1)2v(\gamma_0)}{2u(\gamma_0) + (1-2\gamma_0)(\alpha_1 \beta_2 - \alpha_2 \beta_1)}\right\}\\
&=\frac{1}{4}\left\{ 2\frac{(\beta_2 - \beta_1)(\alpha_2 - \alpha_1)}{\alpha_2 - \alpha_1 + \alpha_1 \beta_2 - \alpha_2 \beta_1} 
- \frac{\gamma_0  (\beta_2 - \beta_1 - \alpha_2 + \alpha_1)^2 + 2(\beta_2-\beta_1)(\alpha_2-\alpha_1)}{2u(\gamma_0) + (1-2\gamma_0)(\alpha_1 \beta_2 - \alpha_2 \beta_1)}\right\}\\
&=\frac{1}{4}\left\{ 2\frac{(\beta_2 - \beta_1)(\alpha_2  - \alpha_1)}{\alpha_2 - \alpha_1 + \alpha_1 \beta_2 - \alpha_2 \beta_1} +
\frac{\gamma_0  (\beta_2 - \beta_1 - \alpha_2 + \alpha_1)^2 + 2(\beta_2-\beta_1)(\alpha_2-\alpha_1)}{-2u(\gamma_0) + (1-2\gamma_0)(\alpha_2 \beta_1 - \alpha_1 \beta_2)}\right\}.
\end{split}
\end{align}
\end{small}
Notice that in the second term 
the denominator is positive.
Hence, we get that
\begin{small}\begin{align*}
\frac{\gamma_0  (\beta_2 - \beta_1 - \alpha_2 + \alpha_1)^2 + 2(\beta_2-\beta_1)(\alpha_2-\alpha_1)}{-2u(\gamma_0) + (1-2\gamma_0)(\alpha_2 \beta_1 - \alpha_1 \beta_2)}
\geq \frac{2(\beta_2-\beta_1)(\alpha_2-\alpha_1)}{-2u(\gamma_0) + (1-2\gamma_0)(\alpha_2 \beta_1 - \alpha_1 \beta_2)},
\end{align*}\end{small}
where for $\gamma_0\in(0,\frac{1}{2}]$ equality holds if and only if $\alpha_1-\alpha_2=\beta_1-\beta_2$. 
Next, we have that
\begin{small}\begin{align*}
&-2u(\gamma_0) + (1-2\gamma_0)(\alpha_2 \beta_1 - \alpha_1 \beta_2) = (1-\gamma_0)(\alpha_1 - \alpha_2) + \gamma_0(\beta_1 - \beta_2)+ (1-2\gamma_0)(\alpha_2 \beta_1 - \alpha_1 \beta_2)\\
&~~~~~~~~~~~~~~~~~= \alpha_1 - \alpha_2 + (1-\gamma_0)(\alpha_2 \beta_1 - \alpha_1 \beta_2) - \gamma_0 (\alpha_1 - \alpha_2 + \beta_2 - \beta_1 + \alpha_2 \beta_1 - \alpha_1 \beta_2).
\end{align*}\end{small}
Because of $\gamma_0>0$, $\beta_2 - \beta_1 + \alpha_1 - \alpha_2 +  \alpha_2 \beta_1 - \alpha_1 \beta_2 \geq 0$ and $\alpha_2 \beta_1 - \alpha_1 \beta_2 \geq 0$ we obtain that
\begin{small}\begin{align*}
-2u(\gamma_0) + (1-2\gamma_0)(\alpha_2 \beta_1 - \alpha_1 \beta_2)& \leq \alpha_1 - \alpha_2 + (1-\gamma_0)(\alpha_2 \beta_1 - \alpha_1 \beta_2)\\
& \leq \alpha_1 - \alpha_2 + (\alpha_2 \beta_1 - \alpha_1 \beta_2),
\end{align*}\end{small}
where for $\gamma_0>0$ equality holds if and only if $\alpha_2 \beta_1 = \alpha_1 \beta_2$ and $\alpha_1-\alpha_2=\beta_1-\beta_2$. We 
conclude that
\begin{small}\begin{align*}
\frac{\gamma_0  (\beta_2 - \beta_1 - \alpha_2 + \alpha_1)^2 + 2(\beta_2-\beta_1)(\alpha_2-\alpha_1)}{-2u(\gamma_0) + (1-2\gamma_0)(\alpha_2 \beta_1 - \alpha_1 \beta_2)} \geq 2\frac{(\beta_2 - \beta_1)(\alpha_2  - \alpha_1)}
{\alpha_1 - \alpha_2 + \alpha_2 \beta_1 - \alpha_1 \beta_2},
\end{align*}\end{small}
where equality  holds if and only if $\alpha_2 \beta_1 = \alpha_1 \beta_2$ and $\alpha_1-\alpha_2=\beta_1-\beta_2$. It is not hard  to see that 
$\alpha_1-\alpha_2=\beta_1-\beta_2$ and $\alpha_2 \beta_1 = \alpha_1 \beta_2$ is equivalent to $\alpha_1=\beta_1$ and $\alpha_2=\beta_2$.  
It follows from \eqref{case2_temp} that
\begin{small}\begin{align*}
\Error(\Ytrue) - \Error(\Ycorr) \geq 0,
\end{align*}\end{small}
where equality holds if and only if $\alpha_1=\beta_1$ and $\alpha_2=\beta_2$.

\vspace{4mm}
Note that in Case 1 we can never have $\alpha_1=\beta_1$ and $\alpha_2=\beta_2$  and that $\alpha_1=\beta_1$ and $\alpha_2=\beta_2$ is
equivalent to $\Bias_{Y=+1}(\widetilde{Y})=\Bias_{Y=-1}(\widetilde{Y})=0$ (compare with \eqref{bias_tildeY}). 
Hence, we have proved Theorem~\ref{thm:accuracy-balanced}. \hfill$\square$

\vspace{\abstparaapp}
\subsection{Long Version of Section~\ref{section_related_work} on Related Work}\label{appendix_related_work}

By now, there is a huge body of work on fairness in 
ML, 
mainly  in supervised learning 
\citep[e.g.,][]{Kamiran2012,Kamishima2012,zemel2013,feldman2015,hardt2016equality,kleinberg2017,pleiss2017,woodworth2017,
zafar2017,zafar2017www,agarwal2018,donini2018,menon2018,FairGAN2018,kallus2019}, 
but 
more 
recently also in  unsupervised learning  
\citep[e.g.,][]{fair_clustering_Nips2017,celis2018,sohler_kmeans,samira2018,fair_k_center_2019,fair_SC_2019,samira2019}. 
All of these papers assume to know the true value of the protected attribute for 
every 
data point. We will discuss some papers not making this assumption below. 
First we discuss the pieces of work
 related to the fairness notion of equalized odds, which is central to our paper and one of the most prominent fairness notions in the ML literature
 (see \citealp{Verma2018}, for a 
 summary of the various 
 notions and a citation~count).

\paragraph{Equalized Odds}
 Our paper builds upon the EO postprocessing method of \citet{hardt2016equality} as described in Section~\ref{section_equalized_odds_MAIN}. 
\citeauthor{hardt2016equality} also show how to derive an optimal predictor
satisfying
the EO criterion based on a biased score function 
rather than a binary classifier $\widetilde{Y}$. 
However, in this case the resulting optimization problem is no longer a linear program and it is unclear how to 
extend 
our analysis to it. 
Concurrently with 
the paper by 
\citeauthor{hardt2016equality}, the fairness notion of EO has also been proposed by \citet{zafar2017www} under the name of disparate mistreatment. 
\citeauthor{zafar2017www} incorporate a proxy for the EO criterion into the training phase of a decision boundary-based classifier, 
which leads to a convex-concave optimization problem and does not come with any theoretical guarantees.
The seminal paper of 
 \citet{kleinberg2017} proves that, except for trivial cases, a classifier cannot satisfy the EO criterion and the fairness notion of 
 calibration within groups 
 at the same time. Subsequently,  \citet{pleiss2017} show how to achieve calibration within groups and  a relaxed 
 form of the EO constraints 
 simultaneously. 
 \citet{woodworth2017} show that 
 postprocessing a Bayes optimal 
 unfair  
 classifier 
 in order to obtain a fair classifier (fair / unfair with respect to the notion of EO) can be highly suboptimal  
 and propose a two-step procedure as remedy. In the first step, some approximate fairness constraints are incorporated into the empirical risk minimization framework to get a classifier that 
 is fair to a non-trivial degree, and in the second step, the EO postprocessing method of \citet{hardt2016equality} is used to obtain the final classifier. 
 This procedure is computationally intractable, however, and \citeauthor{woodworth2017} propose the notion of equalized correlations as a relaxation of the 
 notion of EO, which leads to a 
 computationally tractable learning problem. 
 We also want to mention the critical work of \citet{corbett_davies_goel_2018}, 
 which points out some 
 limitations  of 
  prominent 
  group fairness notions 
  based on inframarginal statistics,  
  including~equalized~odds.

\paragraph{Fairness with 
Only 
Limited Information about the 
Protected 
Attribute} 
\citet{fta2012} phrased the notion of individual fairness 
mentioned in Section~\ref{section_intro}, 
according to which similar data points (as measured by a given metric) should be 
treated similarly by a randomized classifier. 
Only recently there have been works studying how to satisfy group fairness criteria when having only limited information about the protected attribute. 
Most important to mention are  the 
works by  \citet{gupta2018} and \citet{lamy2019}. 
\citet{gupta2018} empirically show that 
when the protected attribute is not known, improving a fairness metric  
for a proxy of the true 
attribute 
can 
improve the fairness metric for the true attribute. 
Our paper provides theoretical evidence 
for their observations.
\citet{lamy2019} study a scenario related to ours and consider training a fair classifier when the 
protected  attribute is corrupted 
according to a mutually contaminated model \citep{scott2013}. 
In their case, 
training 
is done by means of constrained empirical risk minimization 
and requires to solve a non-convex optimization problem. 
Similarly to our Theorem~\ref{theorem_bias}, they show that the bias  of a classifier trained with the corrupted 
attribute grows in a certain way with the amount of corruption (where the bias is defined according to the fairness notions of EO or demographic parity). 
However, they do not  
investigate the error 
of such a classifier. 
Importantly, \citeauthor{lamy2019} only consider classifiers that do not use the protected attribute when making a prediction 
for a test point. 
Also important to mention is the paper by \citet{Hashimoto2018}, 
which
uses distributionally robust optimization in order to minimize the worst-case misclassification risk in a $\chi^2$-ball around the data generating distribution. 
In doing so, under the assumption that the resulting non-convex optimization problem was solved exactly, 
one provably 
controls the risk of each protected group without knowing which group a data point belongs to. \citeauthor{Hashimoto2018} also show that their approach helps to avoid disparity amplification 
in a sequential classification setting in which a group's fraction in the data decreases as its misclassification risk increases. 
As an application of our results, in Section~\ref{subsec_exp_repeated_loss_minimization} /  Appendix~\ref{supp_mat_rep_loss_min}  
we experimentally compare the approach of \citeauthor{Hashimoto2018} to the 
EO method with perturbed 
attribute information in such a sequential setting.  
There are a couple of more works that we want to discuss. 
\citet{botros2018} propose a variational autoencoder for learning fair representations \citep{zemel2013,louizos2016} that also works when the protected attribute is only partially 
observed.   
%
\citet{kilbertus2018} provide an approach to fair classification when users to be classified are not willing to share their protected attribute but only
an encrypted version of it. 
Their approach assumes the existence of a regulator with fairness aims and is based on secure multi-party computation.  
\citet{chen_fairness_under_unawareness} study the problem of assessing the demographic disparity of a classifier when the protected attribute is 
unknown and has to be estimated from data.  
%
\citet{coston2019} study
fair classification in a covariate shift setting where the 
attribute is  
only available in the source domain but not in the target domain (or the other way round). 
Finally, we want to mention the recent line of work on 
\emph{rich subgroup fairness} \citep{Hebert2018,kearns_2018_subgroup1,kearns_2018_subgroup2}.
This notion 
falls between the categories of individual and group fairness in that it requires some statistic to be similar 
for a \emph{large} (or even infinite) number of subgroups,  
which are 
defined 
via a function class 
rather than a protected attribute.

\vspace{\abstparaapp}
\subsection{Detailed Expressions Required for the Experiments of Section~\ref{subsec_simulations}}\label{subsec_supp_mat_sim_details}

We need to solve the linear program
\begin{small}
\begin{align}\label{linear_program_details}
\begin{split}
&\min_{\substack{p_{1,0},~p_{1,1},\\p_{-1,0},~p_{-1,1}\in[0,1]}}~ \sum_{\substack{y\in\{-1,+1\}\\a\in\{0,1\}}}
\left\{\Pro\left[Y=-1, \As=a, \widetilde{Y}=y\right]-\Pro\left[Y=1, \As=a, \widetilde{Y}=y\right]\right\}\cdot p_{y,a}\\
&~\text{s.t.}~~\Pro\left[\widetilde{Y}=1 \,\big|\, Y=y,\As=0\right]\cdot p_{1,0} + \Pro\left[\widetilde{Y}=-1 \,\big|\, Y=y,\As=0\right]\cdot p_{-1,0} =\\
&~~~~~~~~~~~~~~~~ \Pro\left[\widetilde{Y}=1 \,\big|\, Y=y,\As=1\right]\cdot  p_{1,1}+ \Pro\left[\widetilde{Y}=-1 \,\big|\, Y=y,\As=1\right]\cdot p_{-1,1}, \quad y\in\{-1,1\},
\end{split}
\end{align}
\end{small}
where we have to express all coefficients in terms of the problem parameters $\Psymb[Y=y,A=a]$ and $\Psymb\left[\widetilde{Y}=1|Y=y,A=a\right]$ and 
the  perturbation probabilities $\Pro\left[\As\neq A | Y=y,A=a\right]$. As in Section~\ref{subsec_simulations}, we let 
$\gamma_{y,a}:=\Pro\left[\As\neq A |Y=y, A=a\right]$, $y\in\{-1,+1\},a\in\{0,1\}$. From \eqref{form_for_details_1} to \eqref{def_h} in the proof 
of Theorem~\ref{theorem_bias}  
we obtain that the objective function equals
\begin{small} \begin{align*}
\Psymb\left[Y=-1, \As=0\right]\cdot \{p_{1,0} \cdot g + p_{-1,0}\cdot (1-g)\}+\Psymb\left[Y=-1, \As=1\right] \cdot\{p_{1,1} \cdot h + p_{-1,1}\cdot (1-h)\}&\\
- \,\Psymb\left[Y=1, \As=0\right]\cdot \{p_{1,0} \cdot e + p_{-1,0}\cdot (1-e) \}-\Psymb\left[Y=1, \As=1\right] \cdot\{p_{1,1} \cdot f + p_{-1,1}\cdot (1-f) \}&
\end{align*}
\end{small}
and that the constraints are equivalent to
\begin{small}
\begin{align*}
 p_{1,0} \cdot e + p_{-1,0}\cdot (1-e) = p_{1,1} \cdot f + p_{-1,1}\cdot (1-f),\\
p_{1,0} \cdot g + p_{-1,0}\cdot (1-g) = p_{1,1} \cdot h + p_{-1,1}\cdot (1-h)
\end{align*}
\end{small}
with
\begin{small} \begin{align*}
e&:=\alpha_1 + (\beta_1 - \alpha_1)\cdot\Psymb\left[A=1  \,|\, Y=1, \As =0\right],\qquad~~ f:=\alpha_1 + (\beta_1 - \alpha_1)\cdot\Psymb\left[A=1  \,|\, Y=1, \As =1\right],\\
g&:= \alpha_2 + (\beta_2 - \alpha_2)\cdot\Psymb\left[A=1  \,|\, Y=-1, \As =0\right],\qquad h:= \alpha_2 + (\beta_2 - \alpha_2)\cdot\Psymb\left[A=1  \,|\, Y=-1, \As =1\right]
\end{align*} \end{small}
 and $\alpha_1,\beta_1,\alpha_2,\beta_2$ defined in \eqref{def_alpha_beta}.
 It is 
 \begin{small}
 \begin{align*}
  \Psymb\left[Y=y, \As=a\right]=\sum_{a'\in\{0,1\}}\underbrace{\Psymb\left[\As=a\,|\,Y=y,A=a'\right]}_{\gamma_{y,a'}~\text{or}~1-\gamma_{y,a'}}\cdot\Psymb\left[Y=y, A=a'\right]
 \end{align*}
 \end{small}
 and from \eqref{prob_in_proof_end_1} and \eqref{prob_in_proof_end_2} in the proof of Theorem~\ref{theorem_bias} we obtain that
\begin{small}
 \begin{align*}
&\Psymb[A=1 \,|\, Y=y, \As=0]=\frac{\gamma_{y,1}\cdot \Psymb[A=1, Y=y]}{\gamma_{y,1}\cdot \Psymb[A=1, Y=y]+(1-\gamma_{y,0})\cdot 
\Psymb[A=0, Y=y]},\\
&\Psymb[A=1 \,|\, Y=y, \As=1]=\frac{(1-\gamma_{y,1})\cdot \Psymb[A=1, Y=y]}{(1-\gamma_{y,1})\cdot \Psymb[A=1, Y=y]+\gamma_{y,0}\cdot \Psymb[A=0, Y=y]}.
 \end{align*}
 \end{small}
Hence, we have written all coefficients of \eqref{linear_program_details} in terms of the problem parameters and perturbation probabilities.

\vspace{2mm}
After solving \eqref{linear_program_details} and obtaining 
a solution $p_{1,0},\,p_{1,1},\,p_{-1,0},\,p_{-1,1}$, we need to compute the bias and the error of the equalized odds predictor~$\widehat{Y}$ that is based on 
$p_{1,0},\,p_{1,1},\,p_{-1,0},\,p_{-1,1}$. From \eqref{bias_widehatY_in_proof_theo_bias} in the proof of Theorem~\ref{theorem_bias} we obtain that
\begin{small} \begin{align*}
 \Bias_{Y=+1} &= \left|\alpha_1 \cdot(p_{1,0} - p_{-1,0}) - \beta_1\cdot (p_{1,1} - p_{-1,1}) + p_{-1,0} - p_{-1,1}\right|,\\
 \Bias_{Y=-1} &= \left|\alpha_2  \cdot (p_{1,0} - p_{-1,0}) - \beta_2  \cdot (p_{1,1} - p_{-1,1}) + p_{-1,0} - p_{-1,1}\right|.
  \end{align*} \end{small}
  It is easy to verify that the error of $\widehat{Y}$ is given by (recall that the error refers to the test error and that in the test phase $\widehat{Y}$ gets to see the true protected 
  attribute)
  \begin{small}
   \begin{align}\label{error_widehatY_in_details}
   \begin{split}
&  \Error(\widehat{Y})= \Psymb[Y=1]+\big\{\alpha_2\Psymb[Y=-1,A=0]-\alpha_1\Psymb[Y=1,A=0]\big\}\cdot p_{1,0}\\
&~~~~+\big\{\beta_2\Psymb[Y=-1,A=1]-\beta_1\Psymb[Y=1,A=1]\big\}\cdot p_{1,1}\\
&~~~~+\big\{\Psymb[Y=-1,A=0]-\Psymb[Y=1,A=0]-\alpha_2\Psymb[Y=-1,A=0]+\alpha_1\Psymb[Y=1,A=0]\big\}\cdot p_{-1,0}\\
&~~~~+\big\{\Psymb[Y=-1,A=1]-\Psymb[Y=1,A=1]-\beta_2\Psymb[Y=-1,A=1]+\beta_1\Psymb[Y=1,A=1]\big\}\cdot p_{-1,1}.
  \end{split}
  \end{align}
  \end{small}

 \vspace{2mm}
 Finally, we have
 \begin{small} \begin{align*}
 \Bias_{Y=+1}(\widetilde{Y})=|\alpha_1-\beta_1|, \quad  \Bias_{Y=-1}(\widetilde{Y})=|\alpha_2-\beta_2|
 \end{align*} \end{small}
and (simply set $p_{1,0}=p_{1,1}=1$ and $p_{-1,0}=p_{-1,1}=0$ in \eqref{error_widehatY_in_details})
\begin{small} \begin{align*}
&\Error(\widetilde{Y})= \Psymb[Y=1]+\alpha_2\Psymb[Y=-1,A=0]-\alpha_1\Psymb[Y=1,A=0]+\beta_2\Psymb[Y=-1,A=1]-\beta_1\Psymb[Y=1,A=1].
 \end{align*} \end{small}

\vspace{\abstparaapp}
\subsection{Problem Parameters for the Experiments of Figure~\ref{figure_simulations}}\label{subsec_table1}

Table~\ref{table_parameters_main} provides the problem parameters for the experiments shown in Figure~\ref{figure_simulations}.

\vspace{2mm}
\begin{table}[h!]
  \caption{Problem parameters for the experiments of Figure~\ref{figure_simulations}.}\label{table_parameters_main}
  \centering
\renewcommand{\arraystretch}{1.4}
\begin{tabular}{lccccc}
\toprule
\multirow{3}{*}{Plot}&\multicolumn{4}{c}{$\Psymb[\widetilde{Y}=1\,|\,Y=y,A=a]$} & \multirow{3}{*}{~~$(\gamma_{1,1},\gamma_{-1,0},\gamma_{-1,1})$~~}\\
    \cmidrule(r){2-5}
 & ~~$y=1$~~ & ~~$y=1$~~ & ~~$y=-1$~~ & ~~$y=-1$~~ &\\[-1mm]
 & $a=0$ & $a=1$ & $a=0$ & $a=1$ &\\
 \midrule
top left & 0.9 & 0.8 & 0.4 & 0.1 & $(\gamma_{1,0},\gamma_{1,0},\gamma_{1,0})$\\ 
top right & 0.9 & 0.6 & 0.7 & 0.1 & $(\gamma_{1,0},\frac{\gamma_{1,0}}{2},\frac{\gamma_{1,0}}{2})$\\ 
bottom left~~ & 0.9 & 0.6 & 0.3 & 0.8 & $(\frac{\gamma_{1,0}}{2},\frac{\gamma_{1,0}}{4},\frac{\gamma_{1,0}}{8})$\\
bottom right~~ & 0.9 & 0.5 & 0.0 & 0.4 & $(\gamma_{1,0},\gamma_{1,0},\gamma_{1,0})$\\
\bottomrule
\end{tabular}
\end{table}

\vspace{\abstparaapp}
\subsection{Further Experiments as in Section~\ref{subsec_simulations}}\label{supp_mat_section_experiments}

In Figure~\ref{figure_simulations_appendix} and Figure~\ref{figure_simulations_appendix_2}, we present a number of 
further experiments as described in Section~\ref{subsec_simulations}.  
The problem parameters can be read from the titles of the plots and Tables~\ref{table_parameters_appendix} and~\ref{table_parameters_appendix_2}, respectively.  
In these tables, we also report whether 
inequality~\eqref{inequality_error} is true or not (with $\Ytrue$ corresponding to $\widehat{Y}$ for $\gamma_{1,0}=0$, and $\Ycorr$   
corresponding to $\widehat{Y}$ for $\gamma_{1,0}=0.05$).  
We chose the parameters 
to be presented here 
in a rather non-systematic way, but such that 
(i) the given classifier~$\widetilde{Y}$ is biased (i.e., $\Bias_{Y=+1}(\widetilde{Y})>0$), 
(ii) 
$\widetilde{Y}$ 
satisfies Assumption~\ref{assu_error}, 
(iii) we do not only observe constant curves in a plot (i.e., the EO method does not  yield the same classifier for all values of $\gamma_{1,0}$),  
and  
(iv) the parameters cover a wide range of settings.  
In these experiments, we make the same observations as in the experiments of 
Section~\ref{subsec_simulations}, and 
we 
obtain further 
confirmation of  
the main claims of our paper.

\vspace{2mm}
\begin{table}[h]
  \caption{Problem parameters for the experiments of Figure~\ref{figure_simulations_appendix}. We use $r(\gamma_{1,0}):=\min\{2\gamma_{1,0},0.8\}$.}\label{table_parameters_appendix}
  \centering
\renewcommand{\arraystretch}{1.4}
\begin{tabular}{lcccccc}
\toprule
\multirow{3}{*}{Plot}&\multicolumn{4}{c}{$\Psymb[\widetilde{Y}=1\,|\,Y=y,A=a]$} & \multirow{3}{*}{~~$(\gamma_{1,1},\gamma_{-1,0},\gamma_{-1,1})$~~}&\multirow{3}{*}{\eqref{inequality_error} is true}\\
    \cmidrule(r){2-5}
 & ~~$y=1$~~ & ~~$y=1$~~ & ~~$y=-1$~~ & ~~$y=-1$~~ &\\[-1mm]
 & $a=0$ & $a=1$ & $a=0$ & $a=1$ &\\
 \midrule
1st row left & 0.8 & 0.9 & 0.1 & 0.0 & $(\gamma_{1,0},\gamma_{1,0},\gamma_{1,0})$ & yes\\ 
1st row right & 0.8 & 0.9 & 0.1 & 0.0 & $(\frac{\gamma_{1,0}}{2},\frac{\gamma_{1,0}}{4},\frac{\gamma_{1,0}}{8})$ & yes\\ 
2nd row left & 0.8 & 0.9 & 0.1 & 0.0 & $(\gamma_{1,0},\frac{\gamma_{1,0}}{2},\frac{\gamma_{1,0}}{2})$ & yes\\
2nd row right~~ & 0.8 & 0.9 & 0.1 & 0.0 & ~~~~$(\gamma_{1,0},r(\gamma_{1,0}),r(\gamma_{1,0}))$~~~~& yes\\
3rd row left & 0.9 & 0.6 & 0.7 & 0.1 & $(\gamma_{1,0},\gamma_{1,0},\gamma_{1,0})$& yes\\
3rd row right~~ & 0.9 & 0.4 & 0.1 & 0.1 & $(\frac{\gamma_{1,0}}{2},\frac{\gamma_{1,0}}{4},\frac{\gamma_{1,0}}{8})$& yes\\
4th row left & 0.7 & 0.9 & 0.3 & 0.0 & $(\gamma_{1,0},\frac{\gamma_{1,0}}{2},\frac{\gamma_{1,0}}{2})$& yes\\
4th row right~~~~ & 0.7 & 0.9 & 0.3 & 0.0 & ~~~~$(\gamma_{1,0},r(\gamma_{1,0}),r(\gamma_{1,0}))$~~~~& yes\\
5th row left & 0.3 & 0.8 & 0.1 & 0.2 & $(\gamma_{1,0},\gamma_{1,0},\gamma_{1,0})$& yes\\
5th row right~~ & 0.3 & 0.8 & 0.1 & 0.2 & $(\gamma_{1,0},\gamma_{1,0},\gamma_{1,0})$& yes\\
6th row left & 0.9 & 0.6 & 0.4 & 0.1 & $(\frac{\gamma_{1,0}}{2},\frac{\gamma_{1,0}}{4},\frac{\gamma_{1,0}}{8})$& yes\\
6th row right~~~~ & 0.9 & 0.6 & 0.4 & 0.4 & $(\frac{\gamma_{1,0}}{2},\frac{\gamma_{1,0}}{4},\frac{\gamma_{1,0}}{8})$& yes\\
7th row left & 0.5 & 0.8 & 0.1 & 0.4 & $(\gamma_{1,0},\gamma_{1,0},\gamma_{1,0})$& \textbf{no}\\
7th row right~~ &0.6 & 0.8 & 0.1 & 0.4 & ~~~~$(\gamma_{1,0},r(\gamma_{1,0}),r(\gamma_{1,0}))$~~~~& \textbf{no}\\
\bottomrule
\end{tabular}
\end{table}

\vspace{2mm}
\begin{table}[h!]
  \caption{Problem parameters for the experiments of Figure~\ref{figure_simulations_appendix_2}. 
  }
  \label{table_parameters_appendix_2}
  \centering
\renewcommand{\arraystretch}{1.4}
\begin{tabular}{lcccccc}
\toprule
\multirow{3}{*}{Plot}&\multicolumn{4}{c}{$\Psymb[\widetilde{Y}=1\,|\,Y=y,A=a]$} & \multirow{3}{*}{~~$(\gamma_{1,1},\gamma_{-1,0},\gamma_{-1,1})$~~}&\multirow{3}{*}{\eqref{inequality_error} is true}\\
    \cmidrule(r){2-5}
 & ~~$y=1$~~ & ~~$y=1$~~ & ~~$y=-1$~~ & ~~$y=-1$~~ &\\[-1mm]
 & $a=0$ & $a=1$ & $a=0$ & $a=1$ &\\
 \midrule
1st row left & 0.6 & 0.55 & 0.1 & 0.3 & $(\gamma_{1,0},\gamma_{1,0},\gamma_{1,0})$ & yes\\ 
1st row right &  0.9 & 0.6 & 0.4 & 0.1 & $(\gamma_{1,0},\gamma_{1,0},\gamma_{1,0})$ & yes\\ 
2nd row left & 1.0 & 0.8 & 0.0 & 0.1 & $(\gamma_{1,0},\gamma_{1,0},\gamma_{1,0})$ & yes\\
2nd row right~~ & 0.4 & 0.95 & 0.1 & 0.15 & $(\gamma_{1,0},\gamma_{1,0},\gamma_{1,0})$ & yes\\
3rd row left & 0.3 & 0.7 & 0.1 & 0.5 & $(\frac{\gamma_{1,0}}{2},\frac{\gamma_{1,0}}{4},\frac{\gamma_{1,0}}{8})$ & \textbf{no}\\
3rd row right~~ & 0.35 & 0.95 & 0.1 & 0.15 & $(\gamma_{1,0},\frac{\gamma_{1,0}}{2},\frac{\gamma_{1,0}}{2})$ & \textbf{no}\\
\bottomrule
\end{tabular}
\end{table}

\newcommand{\WiSimAppendix}{7.2cm}
\newcommand{\DiSimAppendix}{6mm}

\begin{figure}[h!]
 \centering
\includegraphics[width=\WiSimAppendix]{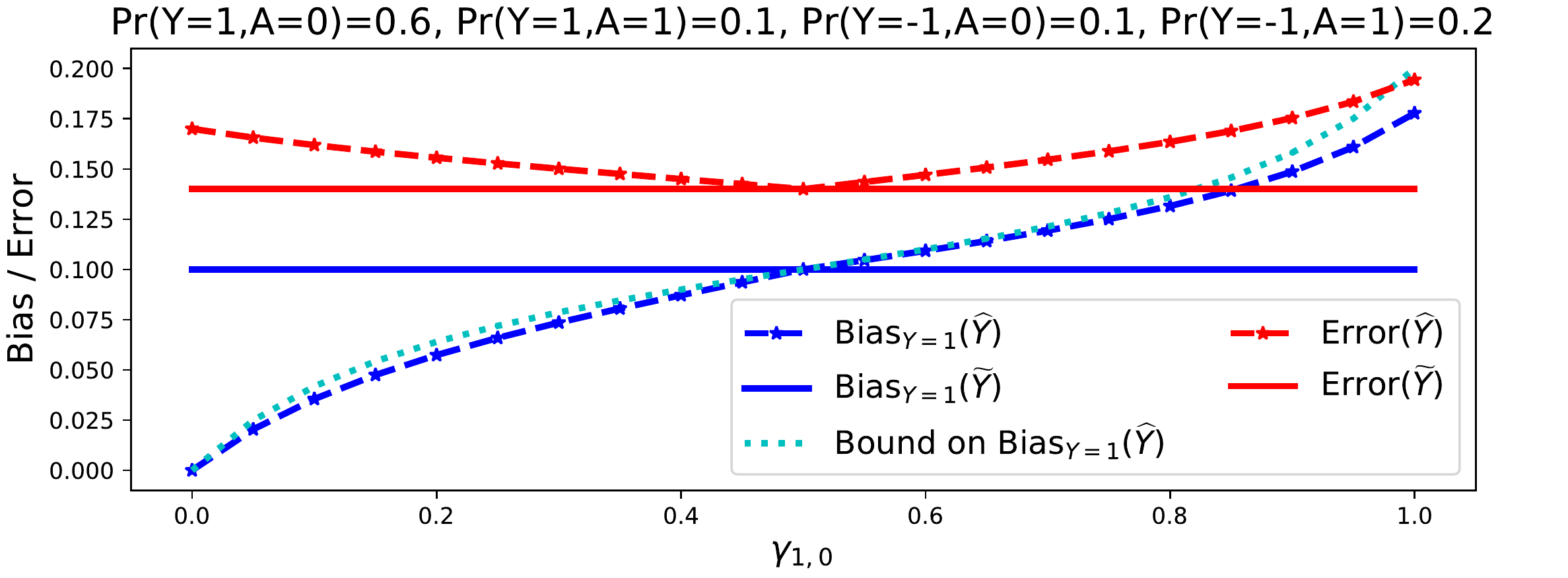}
\hspace{\DiSimAppendix}
\includegraphics[width=\WiSimAppendix]{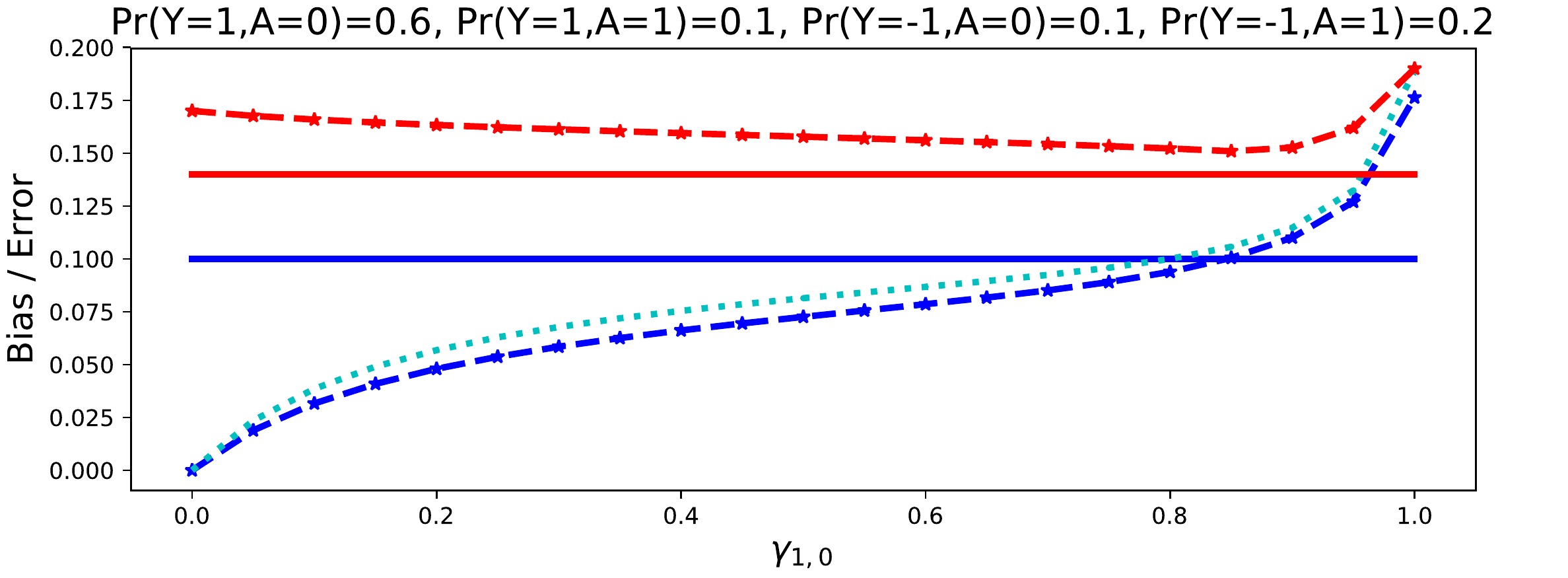}

\vspace{3mm}
\includegraphics[width=\WiSimAppendix]{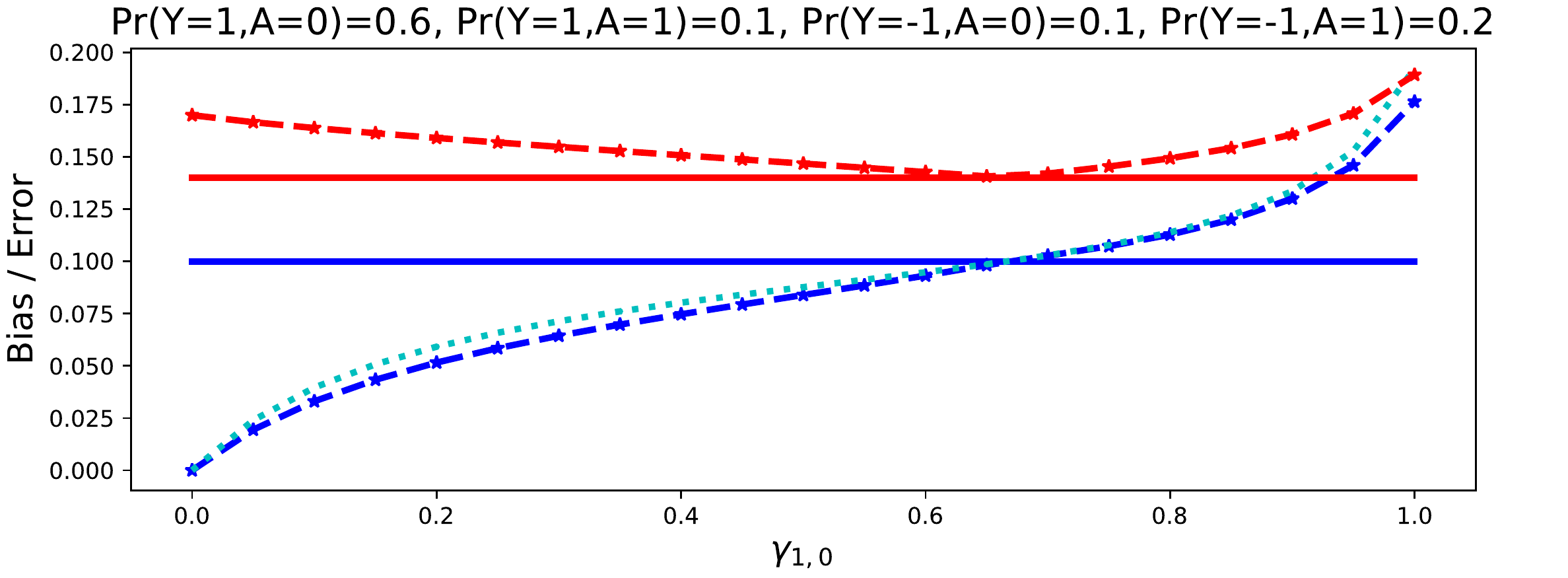}
\hspace{\DiSimAppendix}
\includegraphics[width=\WiSimAppendix]{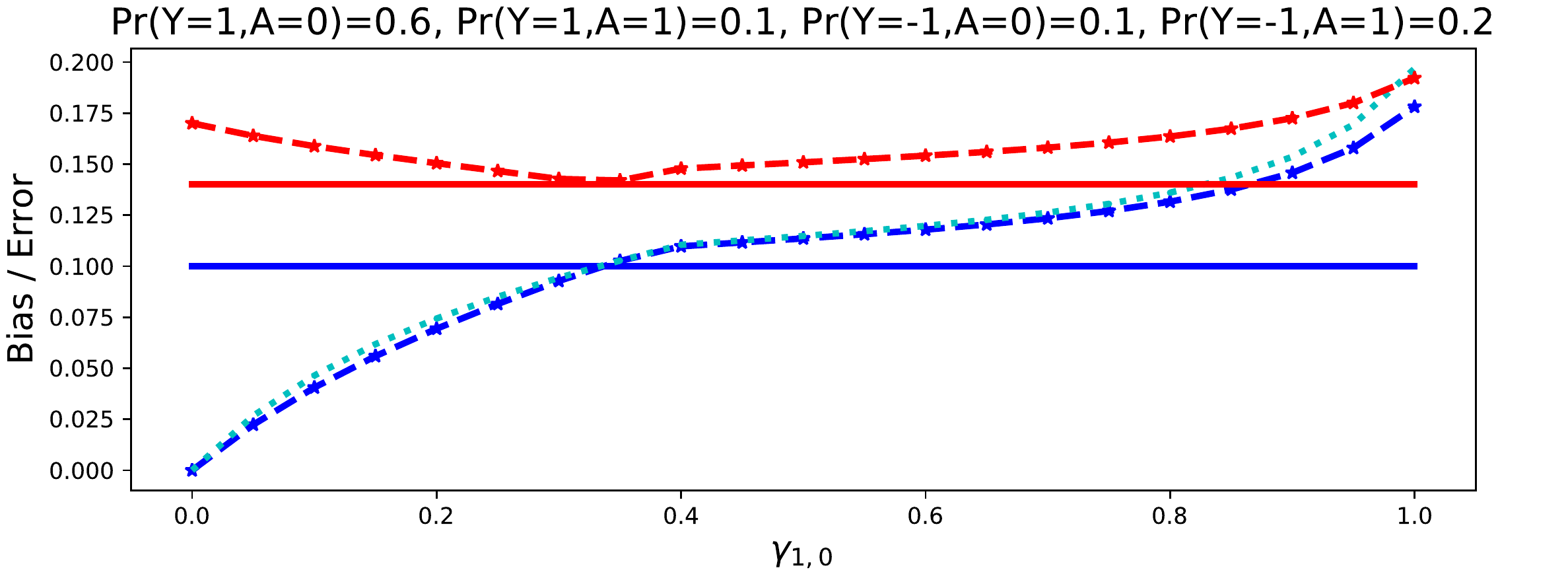}

\vspace{3mm}
\includegraphics[width=\WiSimAppendix]{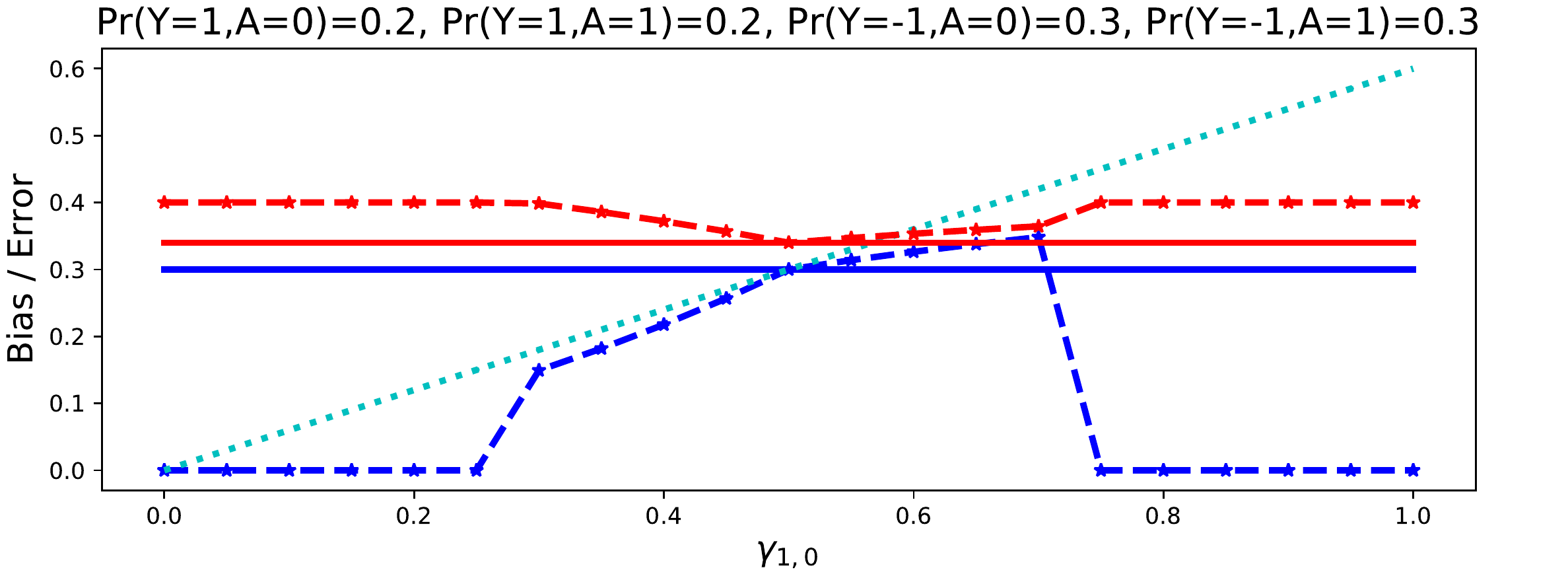}
\hspace{\DiSimAppendix}
\includegraphics[width=\WiSimAppendix]{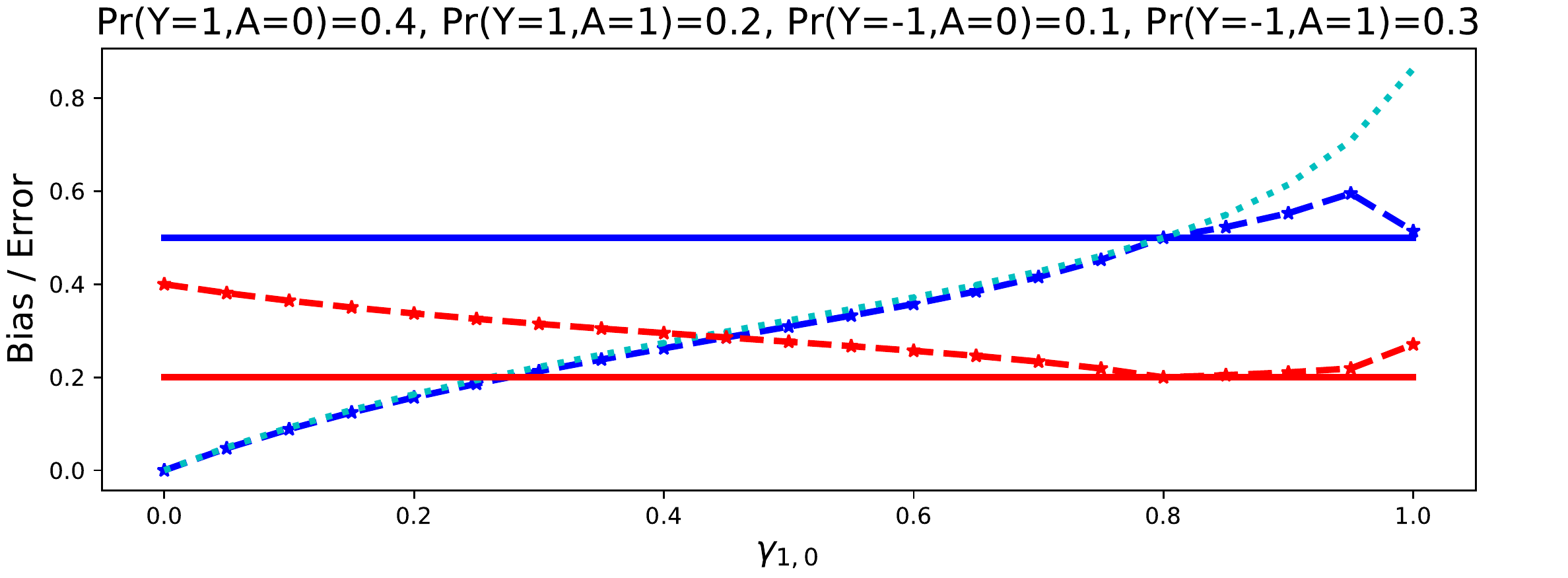}

\vspace{3mm}
\includegraphics[width=\WiSimAppendix]{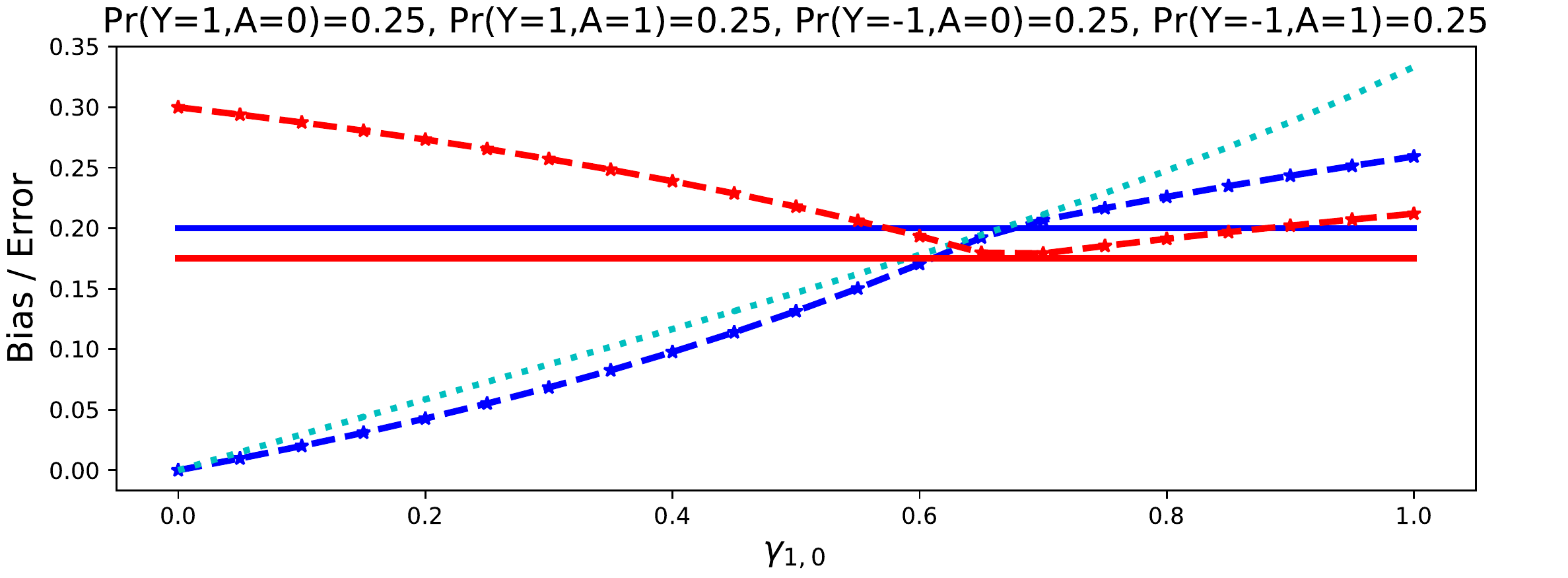}
\hspace{\DiSimAppendix}
\includegraphics[width=\WiSimAppendix]{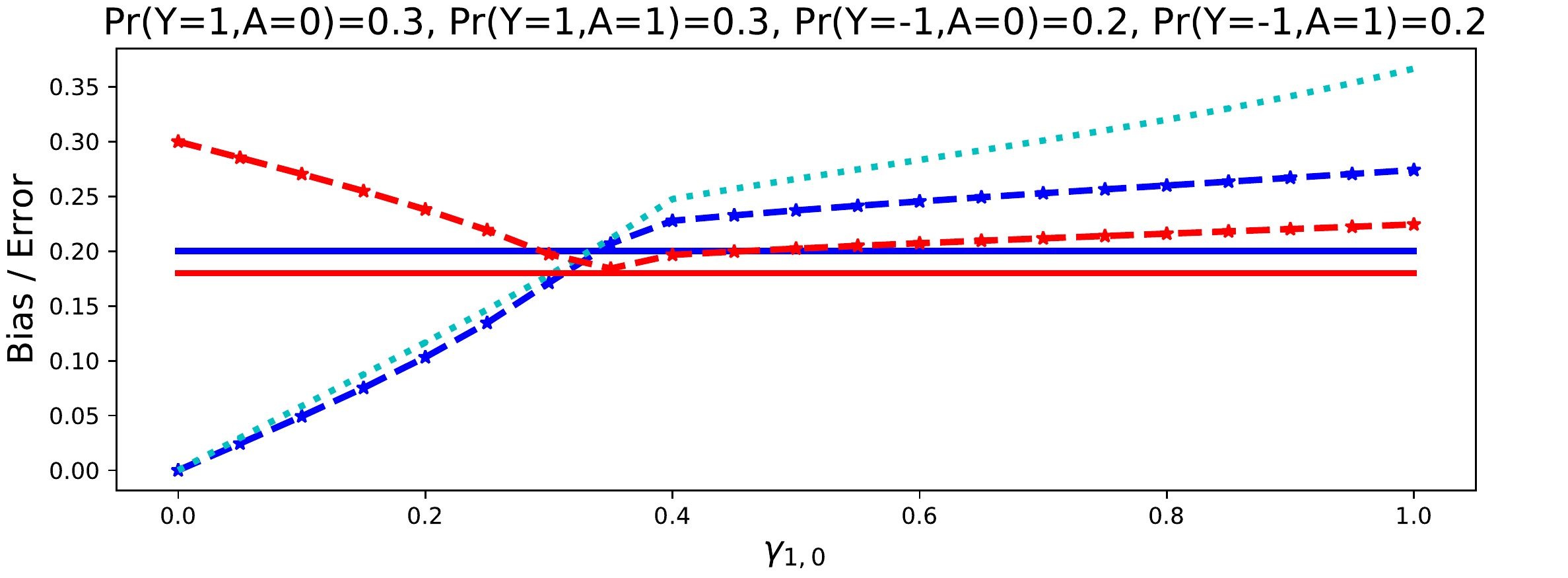}

\vspace{3mm}
\includegraphics[width=\WiSimAppendix]{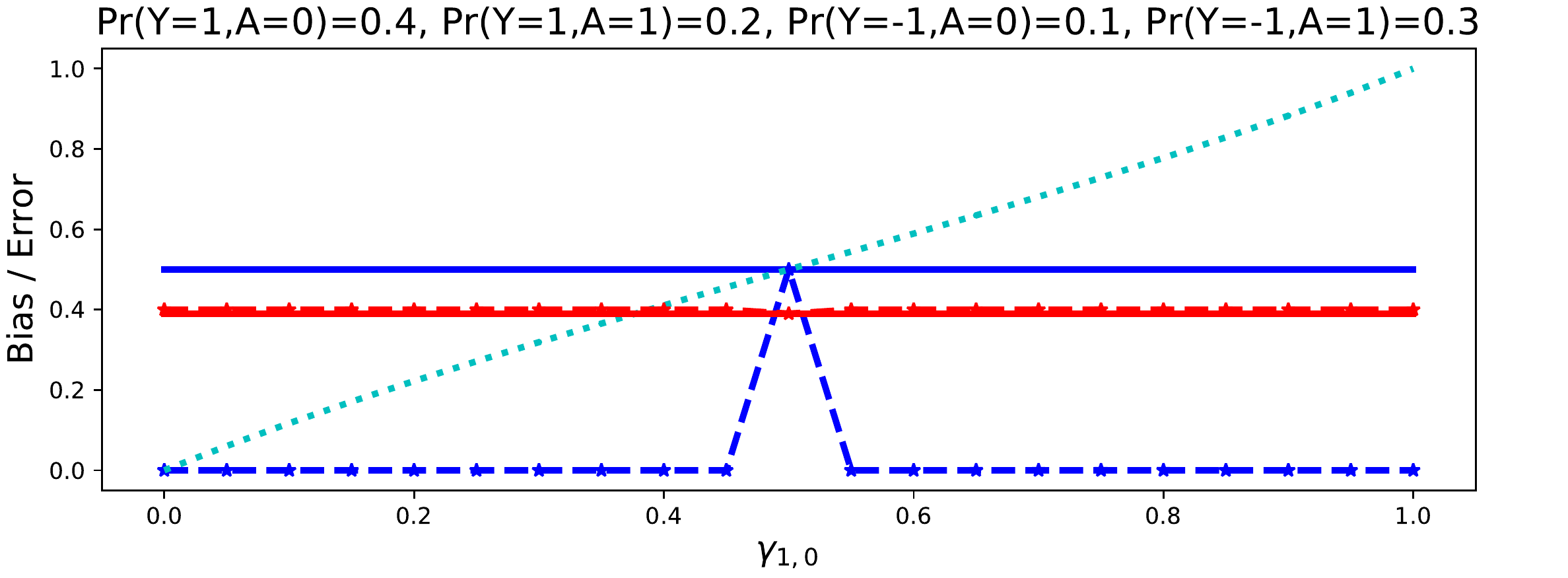}
\hspace{\DiSimAppendix}
\includegraphics[width=\WiSimAppendix]{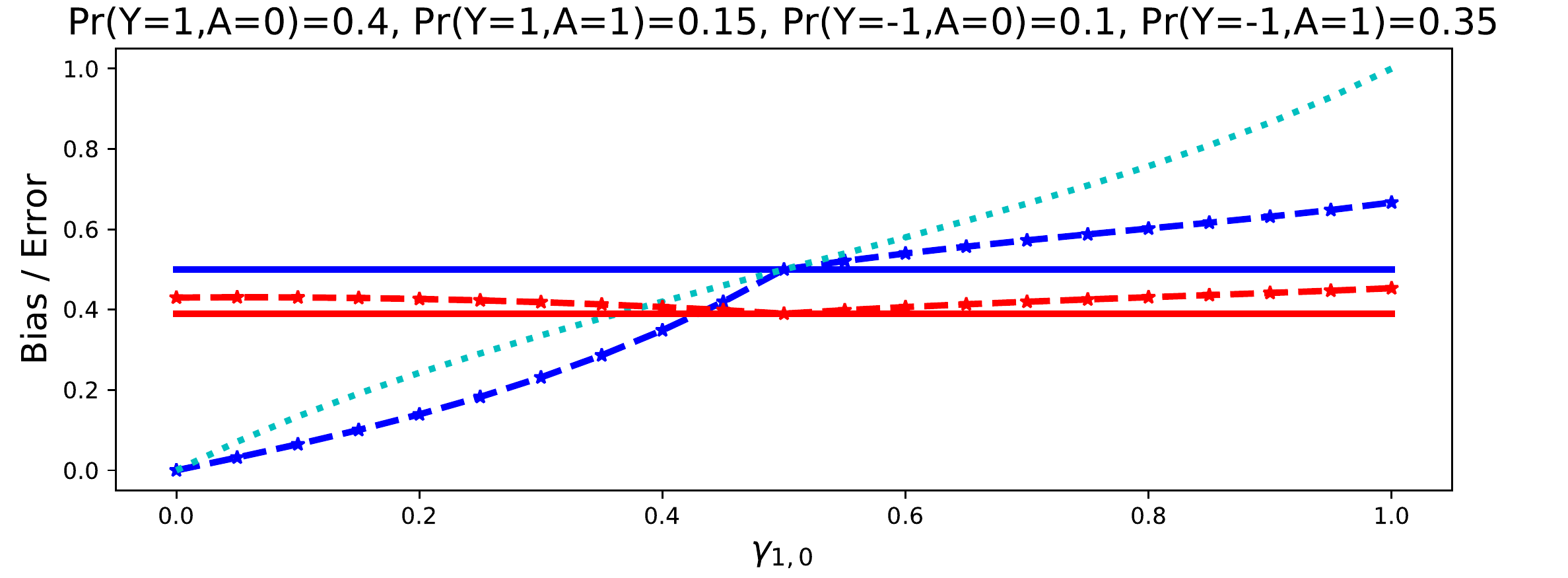}

\vspace{3mm}
\includegraphics[width=\WiSimAppendix]{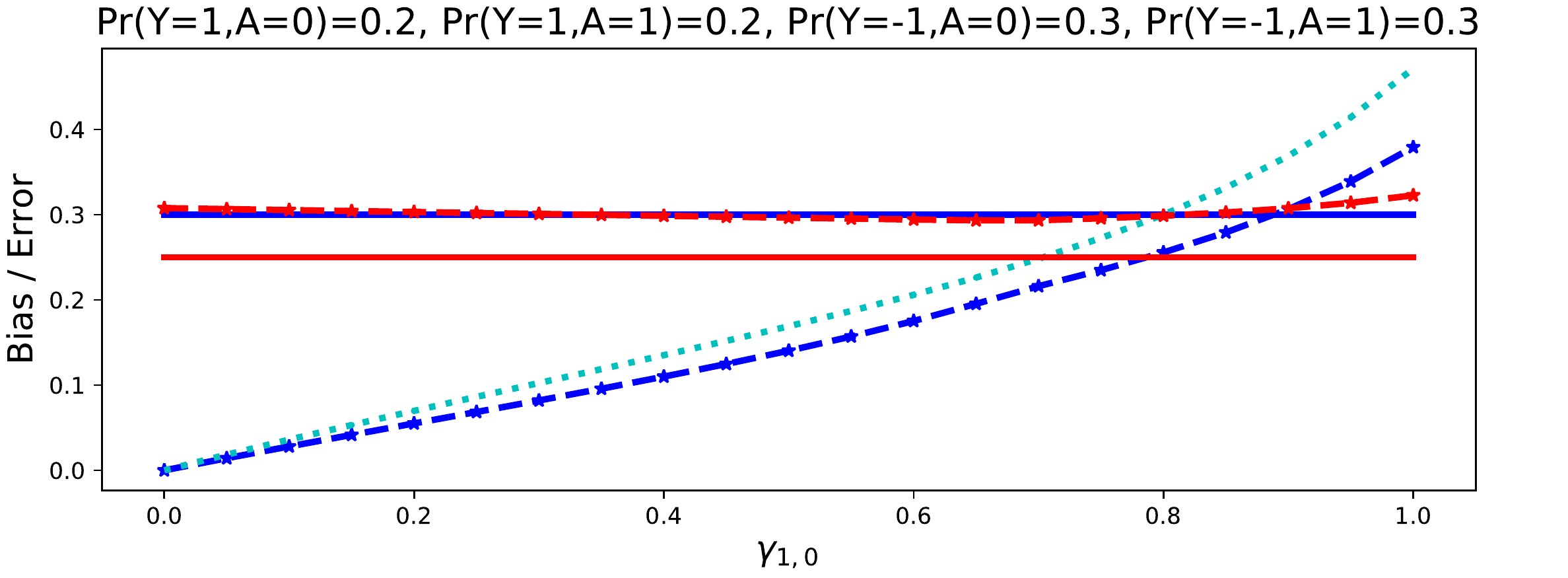}
\hspace{\DiSimAppendix}
\includegraphics[width=\WiSimAppendix]{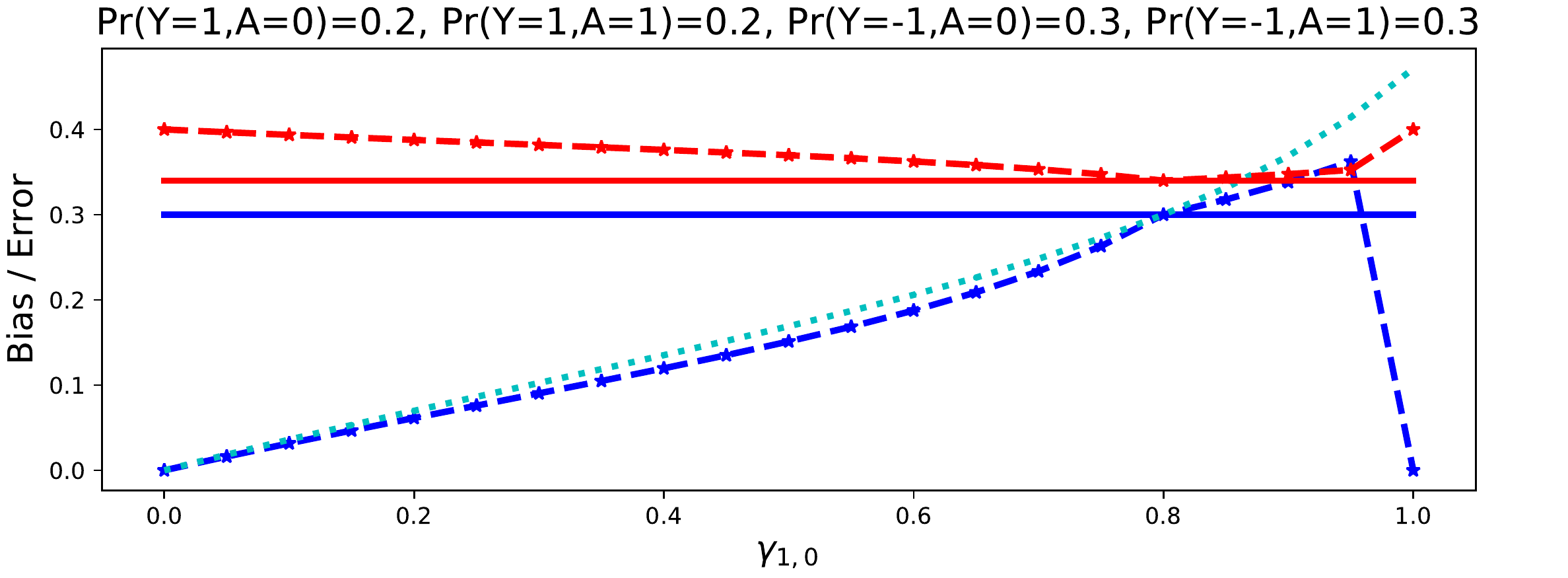}

\vspace{3mm}
\includegraphics[width=\WiSimAppendix]{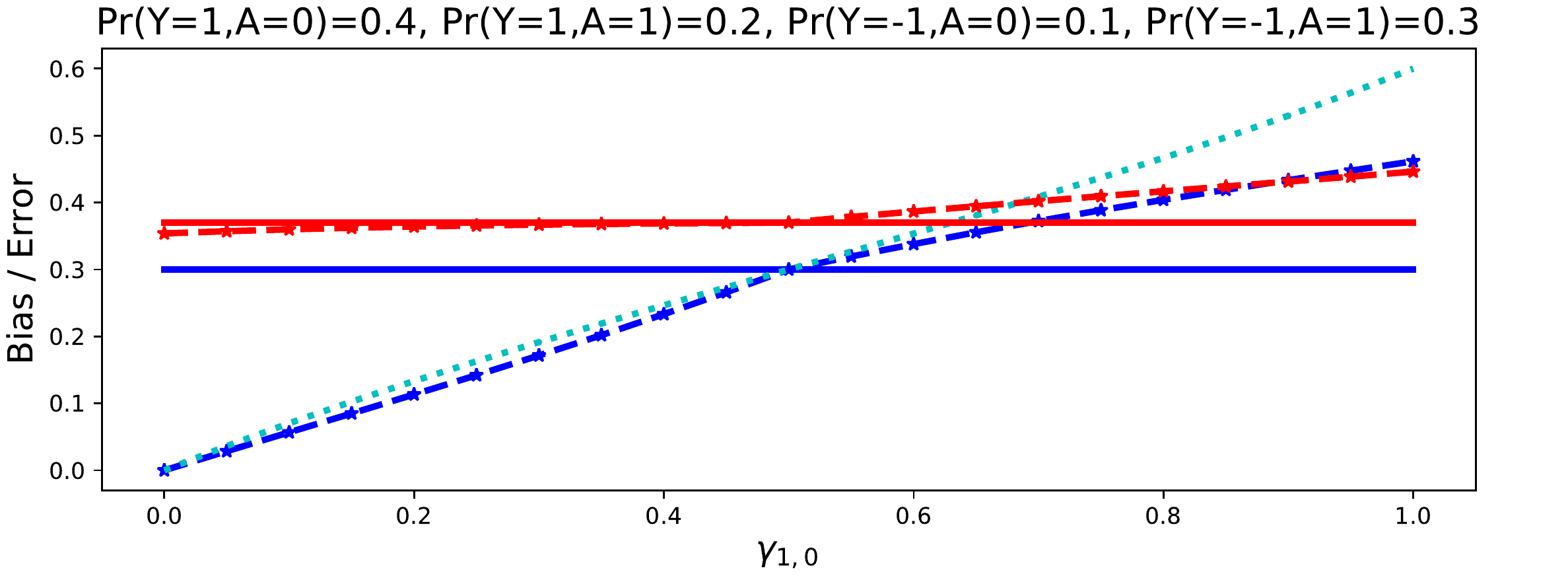}
\hspace{\DiSimAppendix}
\includegraphics[width=\WiSimAppendix]{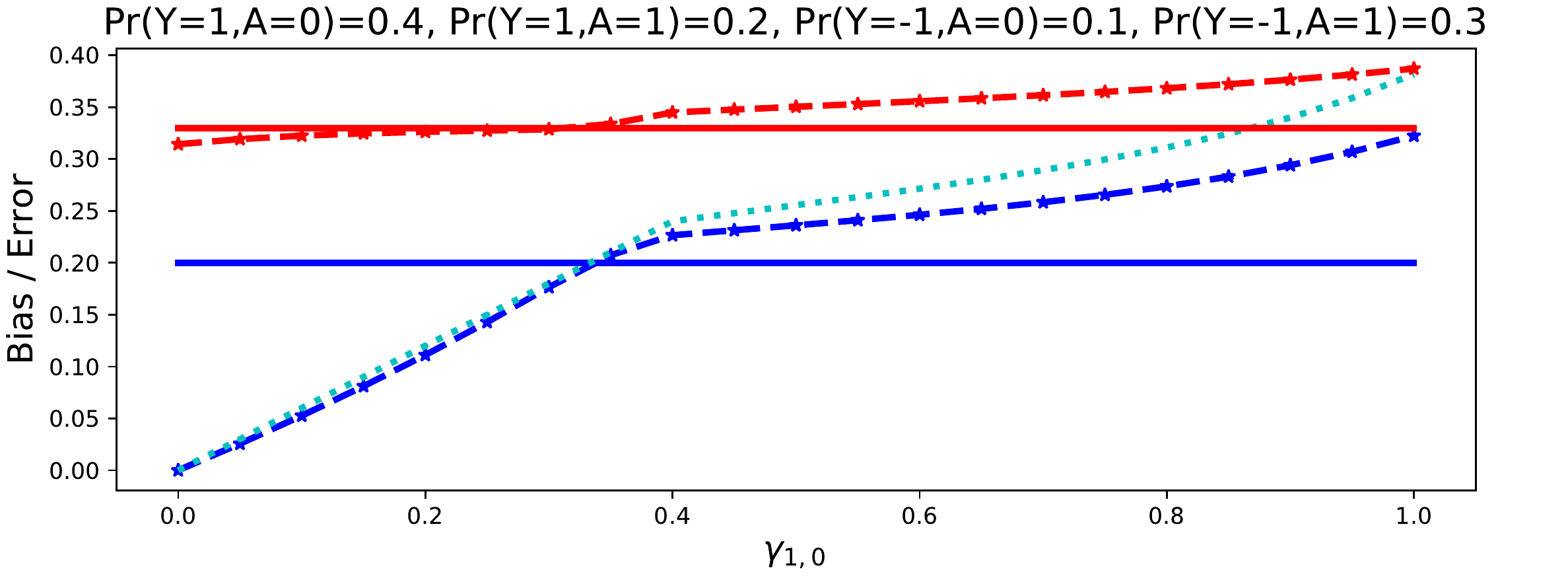}

\caption{Similar experiments as shown in Figure~\ref{figure_simulations}.
The dashed blue curve shows $\Bias_{Y=1}(\widehat{Y})$ and the dashed red curve shows $\Error(\widehat{Y})$ as a function of the perturbation level. 
The solid blue line shows $\Bias_{Y=1}(\widetilde{Y})$  and the solid red line 
shows $\Error(\widetilde{Y})$. 
The dotted cyan 
curve shows the upper bound on $\Bias_{Y=1}(\widehat{Y})$ provided in \eqref{theorem_bias_formula} in Theorem~\ref{theorem_bias}. 
The problem parameters can be read from the titles of the plots and Table~\ref{table_parameters_appendix}.}\label{figure_simulations_appendix}
\end{figure}

\begin{figure}[h!]
 \centering
\includegraphics[width=\WiSimAppendix]{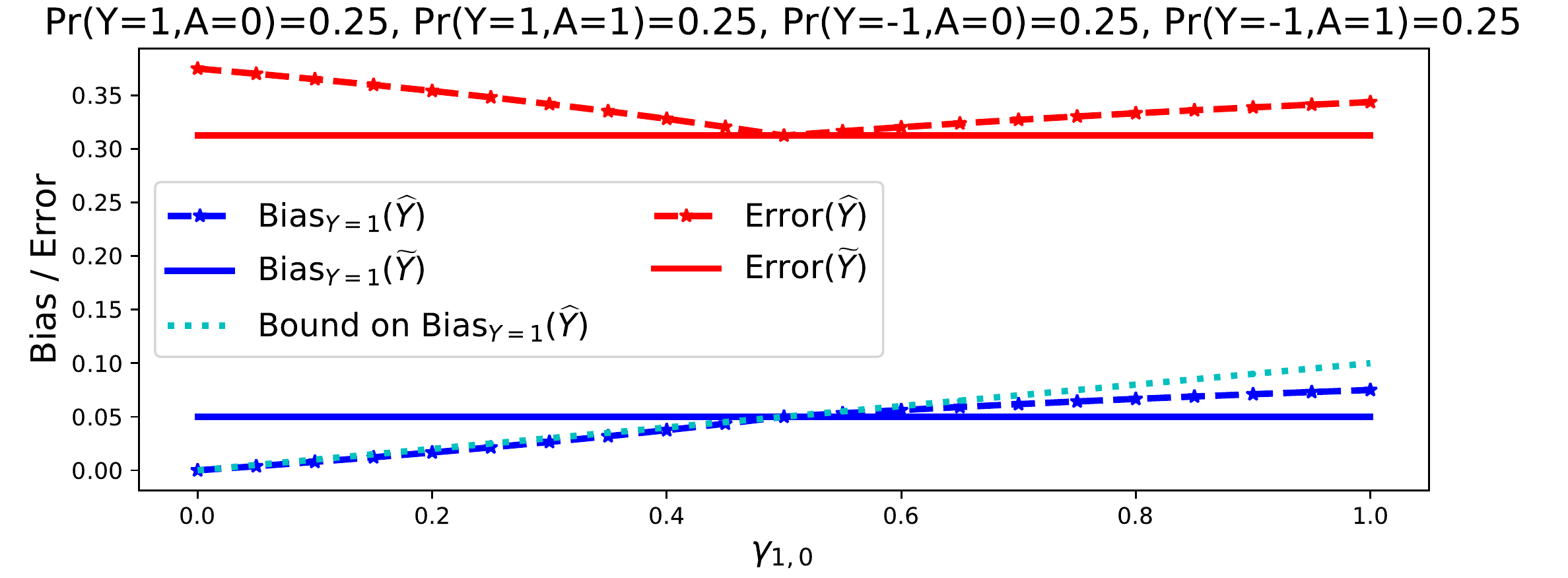}
\hspace{\DiSimAppendix}
\includegraphics[width=\WiSimAppendix]{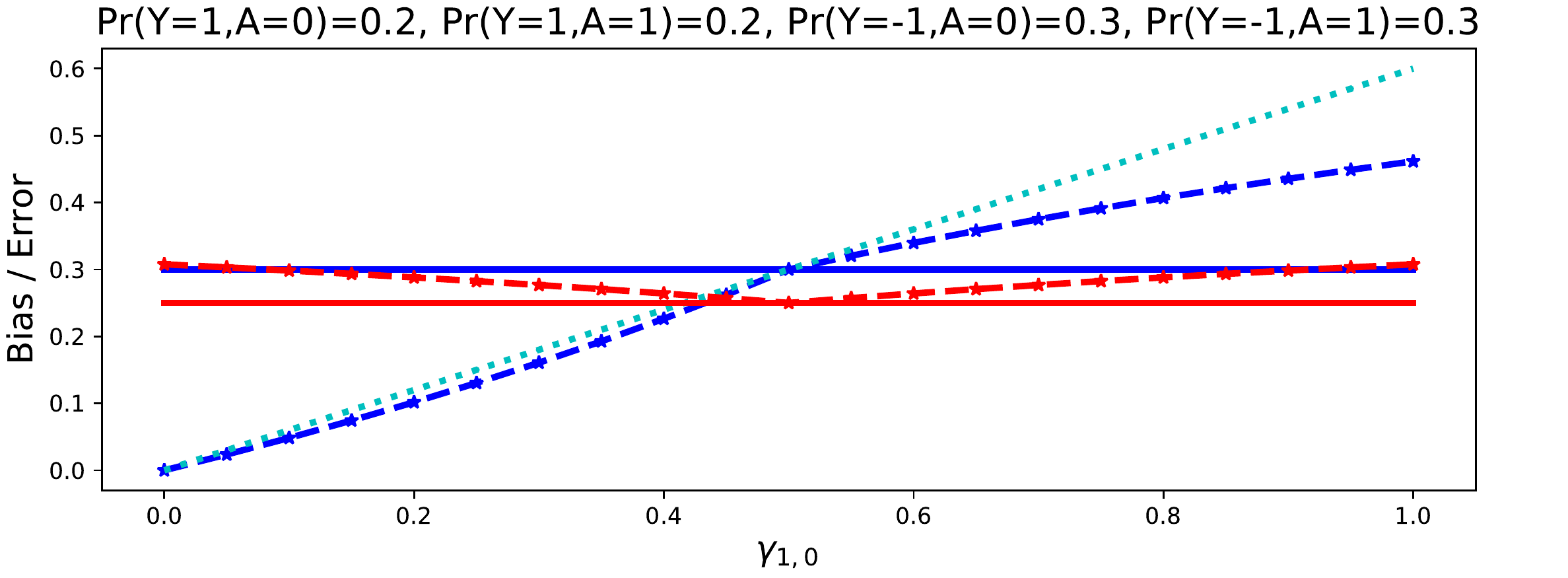}

\vspace{3mm}
\includegraphics[width=\WiSimAppendix]{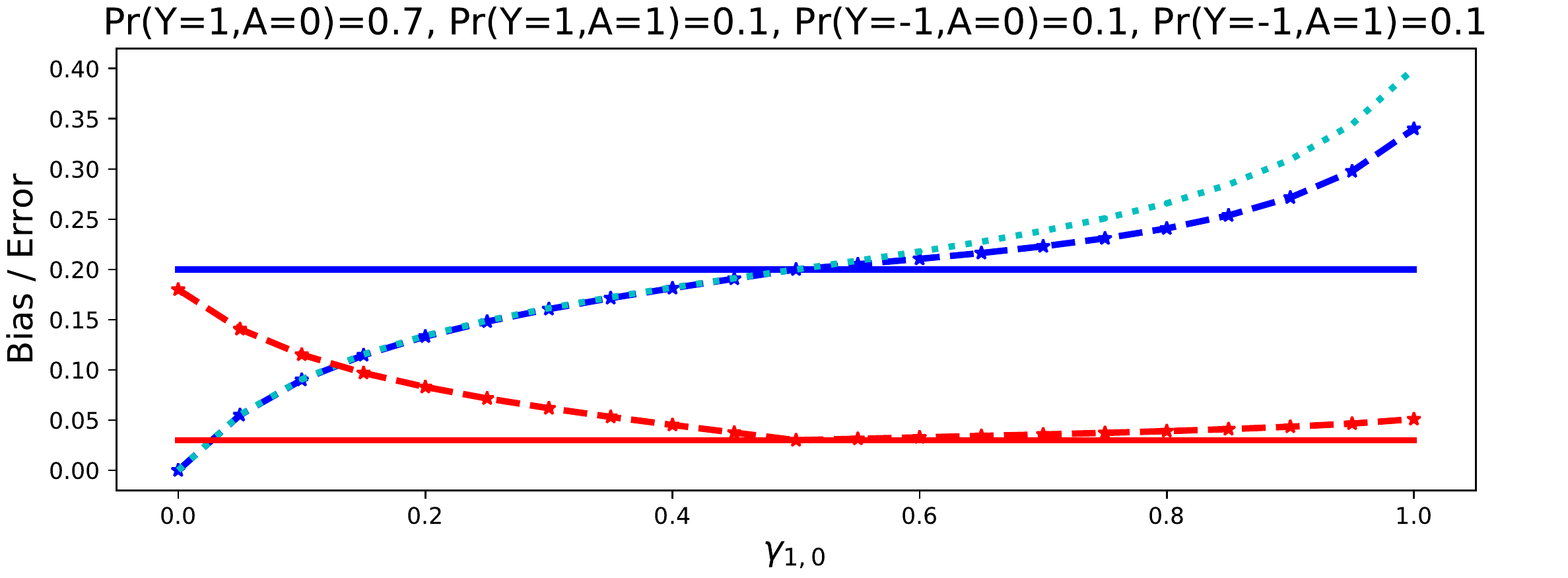}
\hspace{\DiSimAppendix}
\includegraphics[width=\WiSimAppendix]{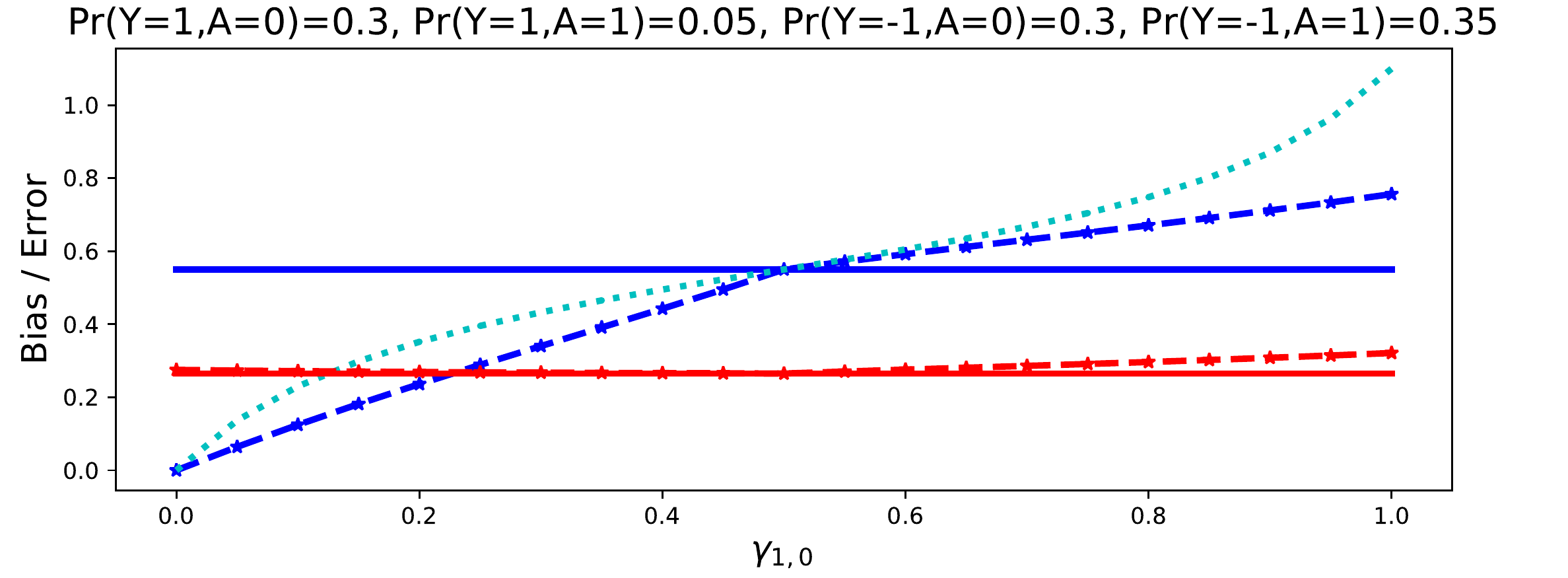}

\vspace{3mm}
\includegraphics[width=\WiSimAppendix]{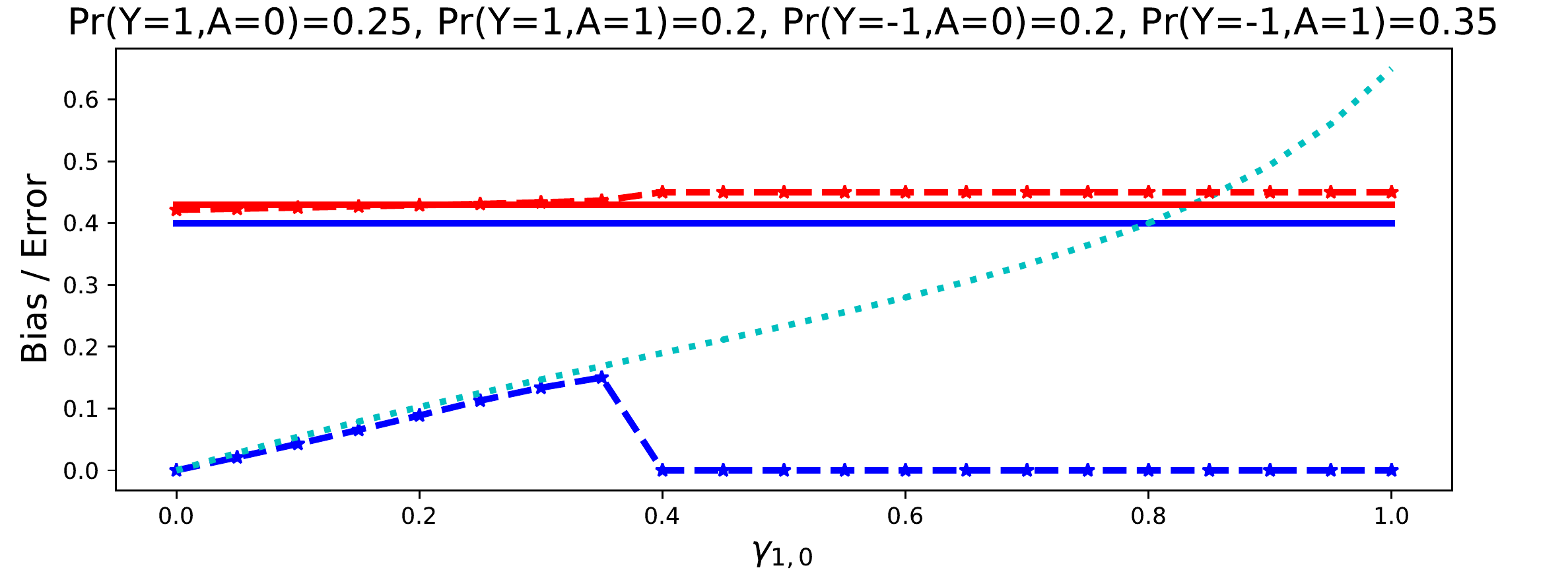}
\hspace{\DiSimAppendix}
\includegraphics[width=\WiSimAppendix]{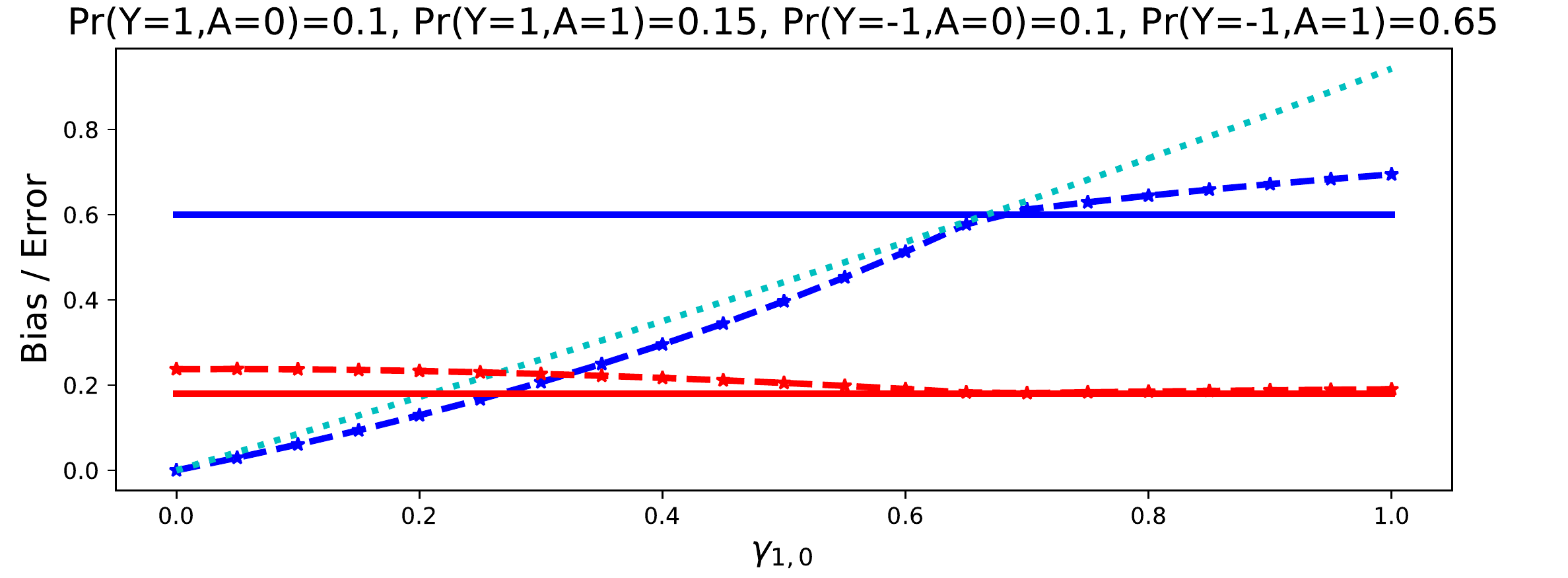}

\caption{Similar experiments as shown in Figures~\ref{figure_simulations} and~\ref{figure_simulations_appendix}.
The dashed blue curve shows $\Bias_{Y=1}(\widehat{Y})$ and the dashed red curve shows $\Error(\widehat{Y})$ as a function of the perturbation level. 
The solid blue line shows $\Bias_{Y=1}(\widetilde{Y})$  and the solid red line 
shows $\Error(\widetilde{Y})$. 
The dotted cyan 
curve shows the upper bound on $\Bias_{Y=1}(\widehat{Y})$ provided in~\eqref{theorem_bias_formula} in Theorem~\ref{theorem_bias}. 
The problem parameters can be read from the titles of the plots and Table~\ref{table_parameters_appendix_2}.}\label{figure_simulations_appendix_2}
\end{figure}

\subsection{Full Table and Additional Statistics of the Experiment on the Drug Consumption Data Set of Section~\ref{subsec_exp_real_data}}\label{appendix_table_drugs}

Table~\ref{table_drugs_appendix} provides the complete results for the experiment on the drug consumption data set of Section~\ref{subsec_exp_real_data}. Note that we do not 
consider the drugs Alcohol, Caff, Choc and the  fictitious drug Semer
since for these drugs it is $\Psymb[Y=1]>0.96$ or $\Psymb[Y=1]<0.01$ and there is a significant chance of observing 
$\Psymb[Y=y,A=a]=0$ for some $y\in\{-1,+1\}$ and $a\in\{0,1\}$ when working with only a random third of the data set. However, the equalized odds postprocessing method requires 
$\Pro[Y=y,A=a]>0$ for $y\in\{-1,+1\}$ and $a\in\{0,1\}$.

Table~\ref{table_assumptions_are_satisfied} provides for each drug the number of runs (out of the 200 in total) in which 
Assumptions~\ref{assu_bias}\,\ref{assu_bias_b} and Assumption~\ref{assu_error}, respectively, is satisfied.

\vspace{2mm}
\begin{table*}[h!]
  \caption{Number of runs (out of 200) with Assumptions~\ref{assu_bias}\,\ref{assu_bias_b} / Assumption~\ref{assu_error} being satisfied.}\label{table_assumptions_are_satisfied}.
   \centering
\renewcommand{\arraystretch}{1.5}
\begin{small}
\begin{tabular}{ccccccccc}
\toprule 
& Amphet & Amyl & Benzos & Cannabis & Coke & Crack & Ecstasy & Heroin  \\
\midrule
Assumptions~\ref{assu_bias}\,\ref{assu_bias_b} & 200 & 196 & 198 & 200 & 199 & 191 & 200 & 166\\
Assumption~\ref{assu_error} & 200 & 199 & 200 & 200 & 200 & 164 & 200 & 194\\
\midrule\midrule
& Ketamine & Legalh & LSD & Meth & Mushroom & Nicotine & VSA &  \\
\midrule
Assumptions~\ref{assu_bias}\,\ref{assu_bias_b} & 169 & 197 & 197 & 164 & 198 & 200 & 197 &\\
Assumption~\ref{assu_error} & 191 & 200 & 200 & 200 & 200 & 200 & 197 &\\
\bottomrule
\end{tabular}
\end{small}
\end{table*}

\begin{landscape}
    \centering
\begin{table}[t]
  \caption{Experiment on the Drug Consumption data set.}\label{table_drugs_appendix}
   \centering
\renewcommand{\arraystretch}{1.5}
\begin{tabular}{ccccccccc}
\toprule \toprule
$Y$ & $\Psymb[Y=1]$ & $\Bias_{Y=1/-1}(\widetilde{Y})$ & $\Bias_{Y=1/-1}(\Ycorr)$ & $\Bias_{Y=1/-1}(\Ytrue)$ & $\Error(\widetilde{Y})$ & $\Error(\Ycorr)$ & 
$\Error(\Ytrue)$ & C. I. \eqref{independence_measure}\\
\toprule
Amphet & 0.36 & 0.085 / 0.106 & 0.076 / 0.065 & 0.043 / 0.027  & 0.317 & 0.339 & 0.352 & 0.033 \\ 
\midrule
Amyl & 0.19 & 0.08 / 0.032 & 0.002 / 0.001 & 0.0 / 0.0 & 0.226 & 0.195 & 0.195 & 0.032 \\ 
\midrule
Benzos & 0.41 & 0.074 / 0.132 & 0.064 / 0.1 & 0.041 / 0.034 & 0.351 & 0.369 & 0.39 & 0.036 \\ 
\midrule
Cannabis & 0.67 & 0.092 / 0.052 & 0.091 / 0.073 & 0.041 / 0.077 & 0.214 & 0.227 & 0.255 & 0.032 \\ 
\midrule
Coke & 0.36 & 0.075 / 0.107 & 0.054 / 0.068 & 0.04 / 0.024 & 0.331 & 0.347 & 0.358 & 0.032 \\ 
\midrule
Crack & 0.1 & 0.075 / 0.025 & 0.0 / 0.0 & 0.0 / 0.0 & 0.129 & 0.101 & 0.101 & 0.039 \\ 
\midrule
Ecstasy & 0.4 & 0.095 / 0.117 & 0.109 / 0.084 & 0.064 / 0.049  & 0.294 & 0.313 & 0.331 & 0.032 \\ 
 \midrule
Heroin & 0.11 & 0.086 / 0.022 & 0.002 / 0.0 & 0.0 / 0.0 & 0.137 & 0.112 & 0.112 & 0.042 \\ 
\midrule
Ketamine & 0.19 & 0.067 / 0.043 & 0.0 / 0.0 & 0.0 / 0.0 & 0.236 & 0.185 & 0.185 & 0.035 \\ 
\midrule
Legalh & 0.4 & 0.098 / 0.062 & 0.119 / 0.047 &  0.071 / 0.044  & 0.261 & 0.281 & 0.289 & 0.031 \\ 
\midrule
LSD & 0.29 & 0.076 / 0.082 & 0.095 / 0.059 & 0.061 / 0.032 & 0.246 & 0.264 & 0.279 & 0.032 \\ 
\midrule
Meth & 0.22 & 0.07 / 0.063 & 0.015 / 0.009 & 0.003 / 0.002 & 0.229 & 0.223 & 0.223 & 0.038 \\ 
\midrule
Mushroom & 0.37 & 0.084 / 0.106 & 0.094 / 0.075 & 0.071 / 0.041 & 0.279 & 0.297 & 0.316 & 0.031 \\ 
\midrule
Nicotine & 0.67 & 0.081 / 0.077 & 0.041 / 0.047 & 0.014 / 0.026 & 0.317 & 0.329 & 0.332 & 0.03 \\ 
\midrule
VSA & 0.12 & 0.074 / 0.037 & 0.0 / 0.0 &  0.0 / 0.0    & 0.148 & 0.12 & 0.12 & 0.043 \\





\bottomrule \bottomrule
\end{tabular}
\end{table}
\end{landscape}

\clearpage
\subsection{Plots for the Experiment of Section~\ref{subsec_exp_real_data} on the Adult Data Set and 
some Statistics of the COMPAS and Adult Data Sets}\label{supp_mat_setion_stats_real_data}

Figure~\ref{figure_experiments_real_data_APPENDIX} provides the plots for the experiment 
of Section~\ref{subsec_exp_real_data} on the Adult data set. 
Table~\ref{table_statistics} provides several statistics of the COMPAS and Adult data sets
(before splitting them into a training and a test set).

\vspace{8mm}

\renewcommand{\wire}{8cm}
\begin{figure*}[h]
 \centering
\includegraphics[width=\wire]{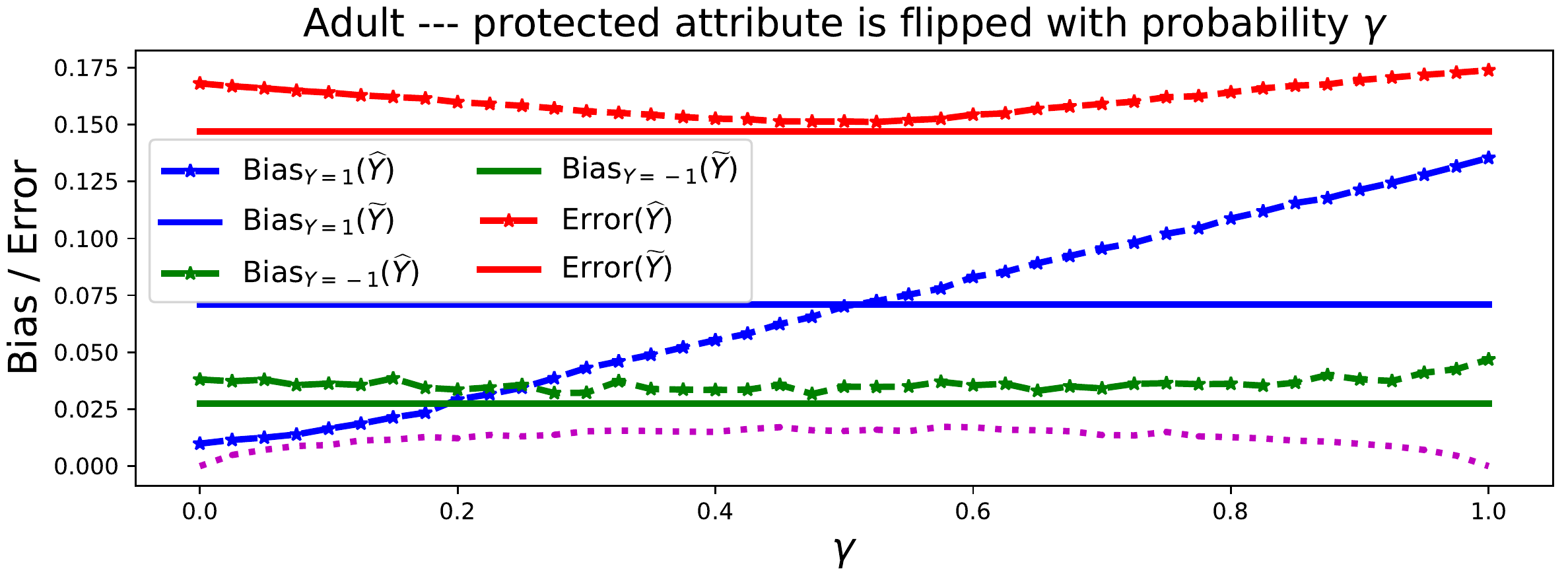}
\hspace{5mm}
\includegraphics[width=\wire]{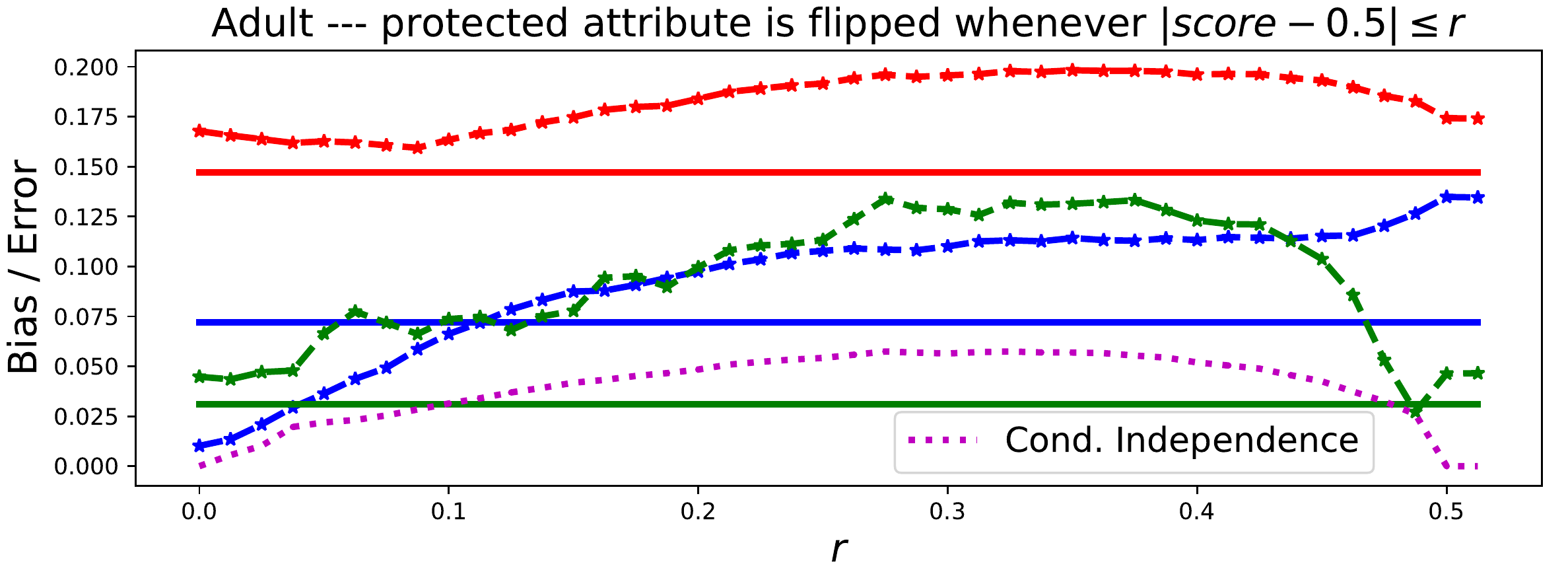}

\vspace{3mm}
\includegraphics[width=\wire]{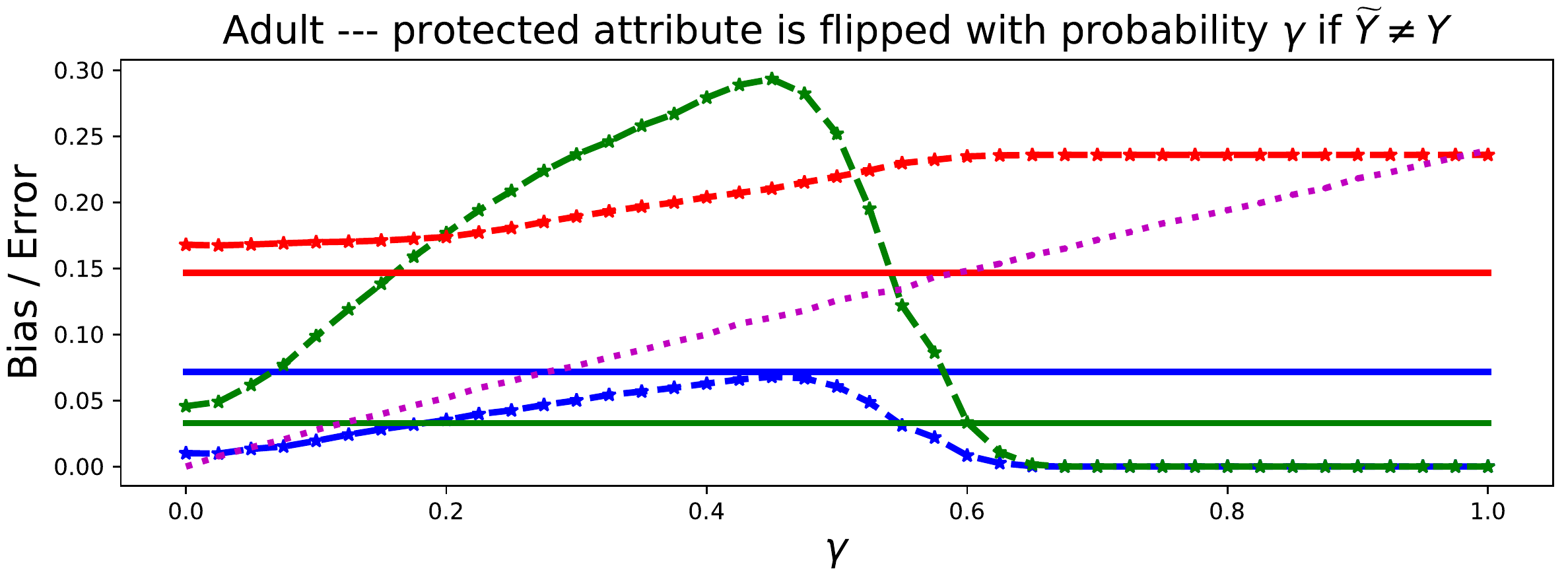}
\hspace{5mm}
\includegraphics[width=\wire]{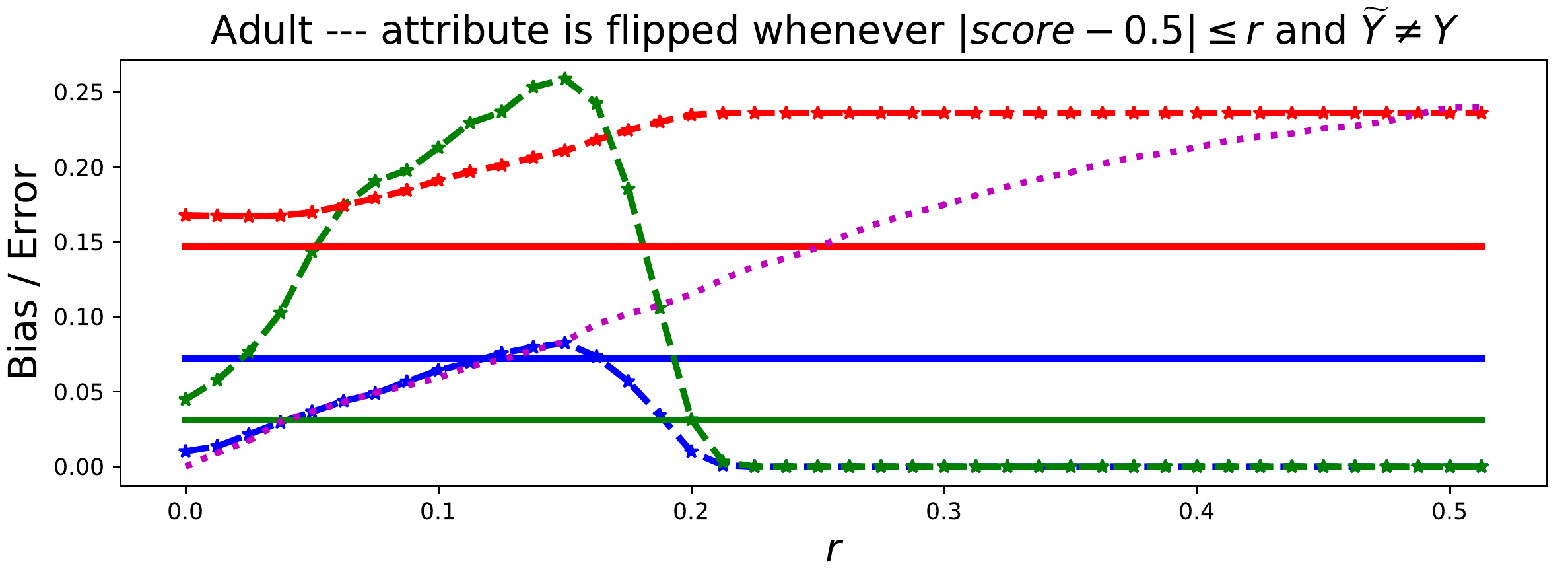}

\caption{Adult data set. 
$\Bias_{Y=+1 / -1}(\widehat{Y})$ (dashed blue / dashed green) 
and $\Error(\widehat{Y})$ (dashed red) as a function of the perturbation level 
in four perturbation scenarios.
The solid lines show the bias (blue and green) and the error (red) of 
$\widetilde{Y}$. 
The magenta line shows an estimate of \eqref{independence_measure} and how heavily Assumptions~\ref{assu_bias}\,\ref{assu_bias_a} is violated.}\label{figure_experiments_real_data_APPENDIX}
\end{figure*}

\vspace{8mm}

\vspace{2mm}
\begin{table}[h!]
  \caption{Statistics of the real data sets used in Section~\ref{subsec_exp_real_data}.}\label{table_statistics}
  \centering
\renewcommand{\arraystretch}{1.5}
\begin{tabular}{ccc}
\toprule
& ~~~~COMPAS~~~~ & ~~~Adult~~~~ \\
\midrule
$\#$ records & 6150 & 9768 \\
\midrule
$\frac{\#\, (Y=1 \,\wedge\, A=0)}{\#~\text{records}}$ & 0.157 & 0.470 \\[4pt]
$\frac{\#\, (Y=1 \,\wedge\, A=1)}{\#~\text{records}}$ & 0.309 & 0.294 \\[4pt]
$\frac{\#\, (Y=-1 \,\wedge\, A=0)}{\#~\text{records}}$ & 0.242 & 0.201 \\[4pt]
$\frac{\#\, (Y=-1 \,\wedge\, A=1)}{\#~\text{records}}$  & 0.292 & 0.036 \\[2pt]
\midrule
$\frac{\#\, (\widetilde{Y}=1)}{\#~\text{records}}$ & 0.394 & 0.795\\[4pt]
$\frac{\#\, (\widetilde{Y}\neq Y)}{\#~\text{records}}$  & 0.344 & 0.147 \\[2pt]
\midrule
$\frac{\#\, (\widetilde{Y}=1 \,\wedge\, Y=1 \,\wedge\, A=0)}{\#\, (Y=1 \,\wedge\, A=0)}$  & 0.408 & 0.897\\[4pt]
$\frac{\#\, (\widetilde{Y}=1 \,\wedge\, Y=1 \,\wedge\, A=1)}{\#\, (Y=1 \,\wedge\, A=1)}$  &  0.628 & 0.968 \\[4pt]
$\frac{\#\, (\widetilde{Y}=1 \,\wedge\, Y=-1 \,\wedge\, A=0)}{\#\, (Y=-1 \,\wedge\, A=0)}$  & 0.147 & 0.374 \\[4pt]
~~~$\frac{\#\, (\widetilde{Y}=1 \,\wedge\, Y=-1 \,\wedge\, A=1)}{\#\, (Y=-1 \,\wedge\, A=1)}$~~~  & 0.343  & 0.398\\
\bottomrule
\end{tabular}
\end{table}

\subsection{Repeated Loss Minimization Experiment Outlined in Section \ref{subsec_exp_repeated_loss_minimization}}\label{supp_mat_rep_loss_min}

As another application of our results, 
 we compare the equalized odds 
 postprocessing 
 method to the method of \citet{Hashimoto2018}, discussed in Section~\ref{section_related_work}, in the sequential classification setting studied by 
 \citeauthor{Hashimoto2018}. In this setting, at each time step a classifier is trained on a data set that comprises several protected groups. 
 The fraction of a group at time step~$t$ depends on 
 the group's fraction and the classifier's accuracy 
 for 
 the group
 at time step~$t-1$. \citeauthor{Hashimoto2018} show that in such a 
 sequential 
 setting 
 standard empirical risk minimization can lead to disparity amplification with a group having 
 a very small fraction /  
 classification accuracy
 after some time while their 
 proposed 
 method helps to avoid~this~situation. 

 In Figure~\ref{figure_experiment_Hashimoto} we present an experiment that reproduces and extends the experiment shown in Figure~5 in 
 \citet{Hashimoto2018}.\footnote{We used the code provided by \citeauthor{Hashimoto2018} and extended it without  
 changing any parameters.} 
 Figure~\ref{figure_experiment_Hashimoto} shows the classification accuracy (left plot) and the fraction (right plot) of the minority group over time for various 
 classification strategies. In this experiment, there are 
 only two groups that initially have the same size, and by minority group we mean the group that has a smaller fraction on average over time 
 (hence, at some time steps the fraction of the minority group can be greater than one half). The classification strategies that we consider are 
 all based on logistic regression. ERM refers to a ``standard'' logistic regression classifier trained with empirical risk minimization and  
 DRO to a logistic regression classifier trained with distributionally robust optimization (the method proposed by \citeauthor{Hashimoto2018}; see their paper for details).   
  EO refers to the ERM strategy with equalized odds postprocessing. We consider EO using the true protected attribute and when 
 the true attribute $A$ is perturbed and replaced by $\As$, which is obtained by flipping $A$ to its complementary value with probabilities $\gamma_0:=\Psymb[\As\neq A| A=0]$ and 
 $\gamma_1:=\Psymb[\As\neq A| A=1]$, respectively, 
 independently for each data point. We can see from the plots that EO achieves the same goal as DRO, namely avoiding disparity amplification, even when the protected attribute is
 highly perturbed (orange and magenta curves with $\gamma_0=\gamma_1=0.45$ and $\gamma_0=0.1$ / $\gamma_1=0.8$, respectively). DRO achieves a slightly higher accuracy, 
 at least in this experiment, and other than EO, 
 it does not require knowledge about the protected attribute at all. 
 However, the underlying optimization problem for DRO is non-convex, and as a result 
 DRO  does not come with theoretical per-step  guarantees. 
 Hence, we believe that in situations where one has access to 
 a  perturbed version of the 
 protected attribute and can assume Assumptions~\ref{assu_bias} and~\ref{assu_error} to be satisfied, the equalized odds 
 postprocessing 
 method is a more trustworthy alternative.

\begin{figure*}[t]
 \centering
\includegraphics[height=4cm]{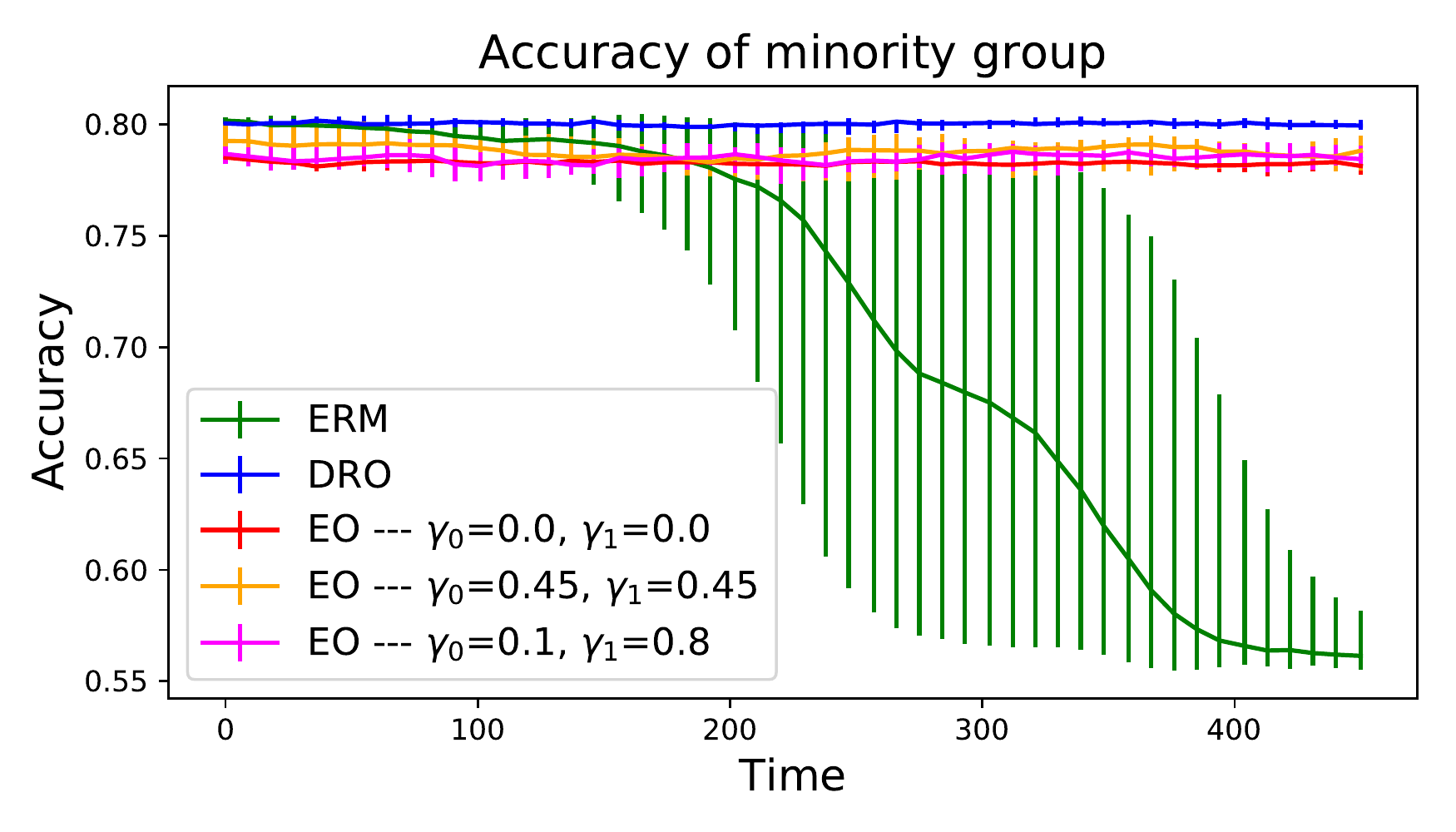}
\hspace{12mm}
\includegraphics[height=4cm]{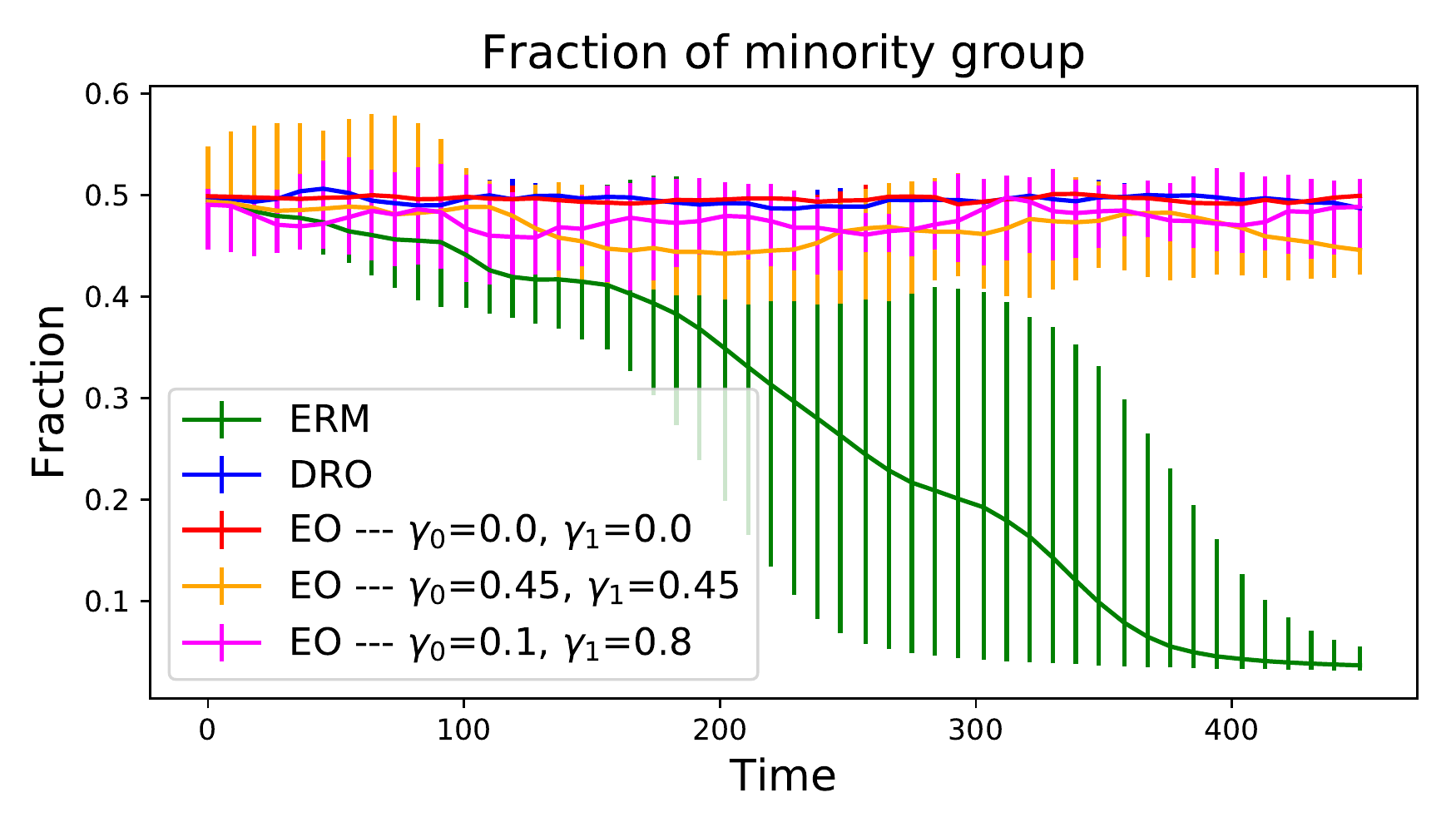}
\caption{Repeated loss minimization experiment of \citet{Hashimoto2018} (Figure~5 in their paper). Not only the method proposed by \citeauthor{Hashimoto2018} (DRO), but also equalized odds postprocessing 
guarantees high user retention, and hence high accuracy, for both groups over time, even when the protected attribute is highly perturbed. The curves and error bars show the accuracy \textbf{(left)} and 
fraction \textbf{(right)} of the minority group over time over 10 replicates of the experiment. 
}\label{figure_experiment_Hashimoto}
\end{figure*}

\end{document}